\documentclass[manuscript]{acmart}%screen
\usepackage[ruled]{algorithm2e} % For algorithms
\usepackage{algorithmic}
\usepackage{soul}

\usepackage{paralist}
\usepackage{subcaption}
\usepackage{tikz}
\usetikzlibrary{arrows} 
\usetikzlibrary[patterns]

\DeclareRobustCommand{\hlw}[1]{{\sethlcolor{white}\hl{#1}}}

\acmJournal{TEAC}
\acmVolume{9}
\acmNumber{4}
\acmArticle{39}
\acmYear{2010}
\acmMonth{3}

\newcount\Comments  % 0 suppresses notes to selves in text
\newcommand{\citeAY}[1]{\citeANP{#1} [\citeyear{#1}]}

\newcommand{\egend}{\hfill $\blacksquare$}
\newcommand{\defend}{\hfill $\blacklozenge$}

\newcommand{\R}{\mathbb{R}}
\newcommand{\N}{\mathbb{N}}
\newcommand{\Q}{\mathbb{Q}}
\newcommand{\cal}[1]{\mathcal{#1}}
\newcommand{\PoD}{\mathit{PoD}}
\newcommand{\PoDL}{\mathit{PoDL}}
\newcommand{\tup}[1]{\langle #1 \rangle}
\newcommand{\ASTC}{\textsc{AssignTC}\xspace}
\newcommand{\BCM}{\textsc{BCMatching}\xspace}
\newcommand{\MCF}{\textsc{MinCostFlow}\xspace}
\newcommand{\poly}{\mathit{poly}}
\DeclareMathOperator{\argmin}{\mathrm{argmin}}
\DeclareMathOperator{\OPT}{\mathit{OPT}}
\DeclareMathOperator{\loc}{\mathit{loc}}
\DeclareMathOperator{\Dist}{\mathit{Dist}}
\DeclareMathOperator{\Ethn}{\mathit{Type}}
\DeclareMathOperator{\Proj}{\mathit{Project}}
\DeclareMathOperator{\Price}{\mathit{Price}}
\DeclareMathOperator{\UB}{\mathit{UB}}
\DeclareMathOperator{\LB}{\mathit{LB}}
\newcommand{\eps}{\varepsilon}

\SetKwInput{KwInit}{Initialize}

\setcopyright{acmcopyright}
\acmDOI{0000001.0000001}

\begin{document}
% Title portion
\title{The Price of Quota-based Diversity in Assignment Problems}
 \titlenote{This work is an extension of the paper ``Diversity Constraints in Public Housing Allocation", published in the proceedings of the 17th International Conference on Autonomous Agents and Multiagent Systems (AAMAS) 2018, pp. 973--981.}
% \subtitle{This is a subtitle}
% \subtitlenote{Subtitle note}

\author{Nawal Benabbou}
\orcid{0000-0002-4589-4162}
\affiliation{%
 \institution{Sorbonne Universit{\' e}, CNRS, Laboratoire d'Informatique de Paris 6}
 \streetaddress{LIP6 F-75005, Paris}
 %\postcode{117417}
 \country{France}
 }\authornote{This work was done while the authors were at the National University of Singapore, supported by the National Research Foundation Fellowship.}
\email{nawal.benabbou@lip6.fr}
\author{Mithun Chakraborty}
\affiliation{%
  \institution{University of Michigan, Ann Arbor}
  \streetaddress{2260 Hayward Street, Ann Arbor}
  \postcode{MI 48109}
  \country{USA}}
  \authornotemark[2]
\email{dcsmc@umich.edu}
\author{Xuan-Vinh Ho}
\affiliation{%
 \institution{Micron Technology}
 \streetaddress{No. 1 Woodlands Industrial, Park D, Street 1}
 \postcode{Singapore 738799}
 \country{Republic of Singapore}}
\authornotemark[2]
\email{hxvinh.hcmus@gmail.com}
\author{Jakub Sliwinski}
\affiliation{%
  \institution{ETH Z{\"u}rich}
  \streetaddress{R{\"a}mistrasse 101, 8092}
  \city{Zurich}
  \country{Switzerland}
}\authornotemark[2]
\email{sljakub@ethz.ch}
\author{Yair Zick}
\affiliation{%
  \institution{University of Massachusetts, Amherst}
  \streetaddress{140 Governors Drive, Amherst}
  \postcode{MA 01002}
  \country{USA}}
\authornotemark[2]
\email{yairzick@cs.umass.edu}

\begin{abstract}
In this paper, we introduce and analyze an extension to the matching problem on a weighted bipartite graph (i.e. the assignment problem): Assignment with Type Constraints. Here, the two parts of the graph are each partitioned into subsets, called types and blocks respectively; we seek a matching with the largest sum of weights under the constraint that there is a pre-specified cap on the number of vertices matched in every type-block pair. Our primary motivation stems from the large-scale public housing program run by the state of Singapore, accounting for over 70\% of its residential real estate. 
To promote ethnic diversity within its housing projects, Singapore imposes ethnicity quotas: the population is divided into ethnicity-based groups and each new housing development into blocks of flats such that each group must not own more than a certain percentage of flats in a block. 
However, other domains use similar hard capacity constraints to maintain diversity: these include matching prospective students to schools or medical residents to hospitals. Limiting agents' choices for ensuring diversity in this manner naturally entails some welfare loss. One of our goals is to study the tradeoff between diversity and (utilitarian) social welfare in such settings. We first show that, while the classic assignment program is polynomial-time computable, adding diversity constraints makes the problem computationally intractable; however, we identify a $\tfrac{1}{2}$-approximation algorithm, as well as reasonable assumptions on the structure of utilities (or weights) which permit poly-time algorithms. Next, we provide two upper bounds on the {\em price of diversity} --- a measure of the loss in welfare incurred by imposing diversity constraints --- as functions of natural problem parameters. We conclude the paper with simulations based on publicly available data from two diversity-constrained allocation problems --- Singapore Public Housing and Chicago School Choice --- which shed light on how the constrained maximization as well as lottery-based variants perform in practice. 
\end{abstract}




\maketitle

%\input{samplebody-journals}

%\section{Introduction}

\section{Introduction}\label{sec:intro}

Consider a mechanism that allocates a set of goods to agents; agents have utilities over items, and we are interested in finding a {\em socially optimal allocation}. This setting (known in the literature as the {\em assignment problem}) is often used to model real-world problems such as allocating public housing, assigning slots in schools, or courses to students. But it is also often the case in these contexts that one wishes to maintain a {\em diverse} allocation: it would be undesirable (from the mechanism designer's perspective) to have certain apartment blocks that predominantly consist of a specific ethnic group, or to have a public school serving students from a specific district. In both cases, agents have different {\em types}, and goods are partitioned into {\em blocks}; our goal is to ensure that each block of goods is allocated to a diverse population of agent types. A diverse allocation of goods is desirable for many reasons (especially in the case of government-funded public goods). First and foremost, it avoids the inadvertent creation of segregated communities; secondly, by ensuring equal access to a public resource, one avoids the risk of discriminatory funding: for example, systematically underfunding schools that serve certain segments of the population, or investing in parks and public facilities in neighborhoods dominated by certain ethnic groups. 

In this work, we study {\em quota-based} mechanisms for maintaining diversity; the initial motivation for this work stems from Singapore's public housing system.

The state of Singapore operates a unique national public housing program, offering a variety of flats for sale at subsidized rates to Singapore citizens and permanent residents. The construction of public housing projects as well as the sale of the flats in these projects on a large-scale public market is centrally managed by a government body called the Housing and Development Board (HDB),\footnote{\url{http://www.hdb.gov.sg}} a statutory board of the Ministry of National Development.\footnote{\url{https://www.mnd.gov.sg/}} As per the latest reports available at the time of writing this paper, an estimated $82\%$ of the resident population of Singapore live in HDB flats \cite{HDB17keystats} that constitute approximately $73\%$ of all apartments in the country \cite{Sing2017}. Since its inception in 1960, HDB has been providing a public good --- affordable apartments in a small country with little real estate --- but by 1989, the system began to exhibit an unforeseen side-effect: the emergence of de facto ethnic enclaves. Mr. S. Dhanabalan, then Minister for National Development, voiced the following concerns as he introduced the Ethnic Integration Policy (EIP) in parliament on 16 February 1989 \cite{Parl1989}:
\begin{quote}
[P]roportionately more Chinese applied for flats in Ang Mo Kio/Hougang Zone and proportionately more Malays applied for flats in the Bedok/Tampines Zone. $[\ldots]$ Malays bought more than half ($55\%$) of the flats in the Bedok/Tampines Zone. In Bedok new town alone, if present trends continue, the proportion of Malays will reach $30\%$ by 1991, and will exceed $40\%$ in 10 years' time. $[\ldots]$\\
%Similar trends are emerging in other estates, largely through open-market resale of HDB flats. $\ldots$ [A] high proportion of buyers in Bukit Merah, Redhill and Henderson were Chinese. In Bedok, Eunos, Teban Gardens and Taman Jurong, more than half the buyers were Malays. In parts of Yishun and in Kampong Java, there was a high proportion of Indian purchasers. 
There are clear signs that racial congregations are re-emerging. Although the problem has not reached crisis proportions, the experience in other multi-racial societies such as the United States shows that while racial groupings start slowly, once a critical point is passed, racial groupings accelerate suddenly.\footnote{We note that this aligns with long-known models of segregation: in his seminal paper, \citeAY{schelling1971dynamic} shows how agents of two types, who are allowed to distribute themselves over an area based on their preferences for the composition of their immediate neigborhoods, lead to the emergence of segregated enclaves, even when each individual prefers a minority of neighbors of a different type to having all neighbors of the same type as herself.}
\end{quote}
The EIP was officially implemented on 1 March 1989; it imposes quotas on the number of units occupied by each of the three ethnic groups: Chinese, Malay and Indian/Others. In 1989, when the percentages of the three ethnic groups (Chinese, Malay, and Indian/Others) in the population were $76.0\%$, $15.1\%$, and $8.9\%$, the corresponding quotas (a cap on the percentage of flats in every block that can be occupied by families of each ethnic group) were set at $87\%$, $25\%$, and $13\%$ respectively; since 5 March 2010, the quota for Indian/Others has been revised to $15\%$ \cite{HDB10PR,deng2013publichousing}.\footnote{In addition to these block capacities, the EIP also imposes neighborhood capacities, where each neighborhood comprises several blocks. Thus, an estate is partitioned into neighborhoods and a neighborhood into blocks with the neighborhood capacity being naturally smaller than the block capacity for each ethnic group: $84\%$, $22\%$, and $12\%$ (increased from $10\%$ in 2010) for Chinese, Malay, and Indian/Others respectively. Moreover, the EIP applies to both new and resale flats, e.g. a Chinese occupant of an HDB flat is free to resell it to another Chinese buyer but will not be able to resell to a Malay buyer if the Malay ethnic quota for that block is already filled. In this paper, we do not address these complications for the most part.}

Ethnic quotas add another layer of complexity to what is, at its foundation, a straightforward allocation problem. HDB uses a lottery mechanism to allocate new developments: all applicants who apply for a particular development pick their flats in random order (see Section~\ref{sec:hdb} for further details). %however, these ethnicity constraints introduce some peculiarities. For example, consider 
Consider an applicant $i$ of Chinese ethnicity applying for an estate with $100$ flats per block, up to $87$ of which may be assigned to ethnically Chinese applicants, and at most $25$ of which can be assigned to ethnically Malay applicants. Assume that $i$ is $90$th in line to select an apartment; will she get a chance to pick a flat in a block she prefers? 
If at least $87$ Chinese applicants were allowed to choose a flat before $i$ and all of them picked flats in this block, the Chinese ethnic quota for the block will have been filled and applicant $i$ will no longer be eligible for the block, even if it still has vacant flats. 
On the other hand, suppose that $i$ is $105$th in line to select an apartment. If $40$ of the applicants picked before $i$ are Malay, then $15$ of them will be rejected from $i$'s preferred block and at least $75$ of the flats will be available for non-Malay applicants. Hence, $i$ will have a spot even if all other $64$ applicants before her also received flats in the same block.\footnote{While this example is stylized, the effects it describes are quite real: one often hears stories of young couples who arrive at the HDB office to select a flat, only to be notified that their ethnic quota has just been filled.}. 

As the example above shows, diversity constraints interact with the allocation mechanism in peculiar ways to affect the overall welfare of the allocation. This issue is not restricted to Singapore public housing; the following are a few examples where upper bounds similar to those in the above housing allocation problem are applied (see \cite{fragiadakis2017improving} and references therein for more detailed expositions). To circumvent a shortage of doctors in rural areas due to medical graduates' preference for urban residency programs, the Japanese government places a ``regional cap'' on the total number of residents matched within each of its $47$ prefectures \cite{kamada2015efficient} -- here, the population of residency applicants is not partitioned, but hospitals within a prefecture can be thought of as forming a block of items. Many school districts in the U.S.A. take active measures for the integration of students from families with differing \emph{soci-economic statuses} (SES) \cite{USedu2017}, one of which is to allot a fraction of the vacant spots in schools via lotteries with percentage caps for all SES groups, as is done in the city of Chicago, Illinois (see Section~\ref{sec:chicago} for further details). The United States Military Academy assigns newly graduated cadets to positions in the army branches, taking cadets' preferences into account but under ``artificial caps''\footnote{These caps are called artificial since they are calculated by the Academy in such a way that any feasible assignment ex post satisfies the maximum and minimum quotas for each branch that are based on actual staffing needs.} on the number of assignments per branch \cite{fragiadakis2017improving}. 

%Indeed, the allocation of public housing is an economic problem similar to the classic assignment problem: a central planner (HDB) wishes to allocate goods (apartments) to agents (residents) in a manner satisfying certain economic criteria. 
%Diversity, on the other hand, is a social goal external to the underlying economic domain; imposing it may result in reduced social welfare. But, 
The imposition of diversity constraints as above can naturally lead to a reduction in the total achievable utility/economic value by the assignment, but we must bear in mind that diversity is a social desideratum external to any such economic consideration. For purposes such as policy making and the proper functioning of diversity-inducing measures included in automated decision-making systems, it is imperative to deepen our understanding of the impact that these measures have on the underlying assignment mechanism. In this paper, we investigate this impact from both computational and economic angles.

\subsection{Our Contributions}\label{sec:contrib}
We study the interplay between diversity and utility in assignment problems; we set up a benchmark where a central planner (e.g. HDB) has access to the correct utilities (or, in general, weights) of all agents (e.g. applicant households) for all items (e.g. flats); agents are partitioned into \emph{types} with respect to a single attribute (e.g. ethnic groups) and goods are also similarly divided into disjoint \emph{blocks} (e.g. blocks of flats as defined by HDB); a limited number of goods in each block can be allocated to agents of each type. We call these upper limits \emph{type-block capacities}. 

These restrictions result in several interesting outcomes. 
While the unconstrained optimal assignment problem is well-known to be polynomial-time solvable \cite{kuhn1955hungarian}, we show that imposing type-block constraints makes it computationally intractable (Section~\ref{sec:complexity}). However, we also show that, in general, a polynomial-time $\tfrac{1}{2}$-approximation algorithm (Section \ref{sec:approx}) exists, and identify utility models for which one can find the optimal assignment with type-block constraints in polynomial time (Section~\ref{sec:polytime}). 
In Section~\ref{sec:diversity}, we study the potential welfare loss from imposing type-block constraints, which we term the {\em price of diversity} as in \citeAY{ahmed2017diverse}, and we show that it can be bounded by natural problem parameters.
Finally, we analyze the empirical price of diversity as well as the welfare loss induced by the lottery mechanism on simulated instances generated from publicly available, real-world data pertaining to public housing in Singapore and school choice in Chicago, IL, USA (Section~\ref{sec:experiments}).

\subsection{Related work}\label{sec:related}
%One can think of public housing allocation as a
The problem we study is an extension to the bipartite matching problem \cite{lovasz2009matching} where each edge joins an agent to an item and is weighted with the utility the agent will receive if she is allocated that item. %an edge from an agent $i$ to an apartment $j$ is weighted with the utility the agent will receive if she is allocated the apartment. 
There is a rich body of literature on weighted bipartite matching problems (also known as assignment problems~\cite{munkres1957algo}), and polynomial-time algorithms for the unconstrained version have long been known (e.g. \cite{kuhn1955hungarian}). 
Several generalizations and/or constrained versions have been studied, e.g. recent work by \citeAY{lian2017conference} who allow each agent (resp. item) to be matched to multiple items (resp. agents) but within upper and lower capacities. Some previously studied variants correspond to (polynomial-time) special cases of our problem. For example, the assignment problem with subset constraints studied by \citeAY{bauer2005subsetconst} can be thought of as a special case of our problem, with a single block or a single type; if all agents of each type have identical utilities for all apartments in each block, and each type-block capacity is smaller than both the corresponding type and block sizes, then our problem reduces to a special case of the polynomial-time solvable \emph{capacitated $b$-matching} on a bipartite graph \cite{ahn2014near}.

%Singapore's public housing is our primary motivating domain but 
In addition to our main motivating problem of HDB housing allocation and the other documented examples \cite{fragiadakis2017improving} noted in the introduction, 
type-block constraints can naturally arise in many other settings related to assignment/allocation problems with no monetary transfers \cite{hylland1979efficient,zhou1990conjecture}. For example, consider the course allocation problem analyzed by \citeAY{budish2012multi}; one might require that each course has students from different departments and impose maximal quotas to ensure this. Other examples include allocating subsidized on-campus housing to students \cite{abdulkadirouglu1998random}, appointing teachers at public schools in different regions as done by some non-profit organizations \cite{featherstone2015rank}, or assigning first year business school students to overseas programs \cite{featherstone2015rank}. Our results apply to the seminal work on public school allocation \cite{abdulkadiroglu2003school,abdulkadirouglu2009strategy,pathak2013school} and matching medical interns or residents to hospitals \cite{roth1984evolution} that does not concern itself with diversity/distributional constraints. This line of work mainly explores the interaction between individual selfish behavior and allocative efficiency (e.g. Pareto-optimality) of matching mechanisms, under either ordinal preferences or cardinal utilities, one-sided or two-sided (see, e.g. \cite{bogomolnaia2001new,bhalgat2011social,bade2014random,AnsDas13} and references therein); we, on the other hand, focus on the impact of type-block constraints on welfare loss, when agents' utilities are known to a central planner.

Another relevant strand of literature is that on the fair allocation of indivisible goods (see, e.g. \cite{procaccia2014fair,caragiannis2016unreasonable,kurokawa2016can,barman2017finding,barman2017approx} and references therein): fairness is usually quantified in terms of the utilities or preferences of agents for allocated items (e.g. proportionality, envy-freeness and the maximin share guarantee) but our contribution deals with a different notion of fairness: the proportionate representation of groups in the realized allocation, with no regard to agents' utilities. 

Some recent work has formally addressed diversity issues in computational social choice. Unlike our paper, \citeAY{ahmed2017diverse} treat ``diversity as an objective, not a constraint'' in a $b$-matching context (e.g. matching papers to reviewers with diverse interests): they minimize a supermodular objective function to encourage the matching of each item to agents of different types. Our diversity concept comes close to that of \citeAY{bredereck2018multiwinner} who also achieve diversity by imposing hard constraints on the maximization of a (submodular) objective that measures the quality of the solution; however, they work in a committee (subset) selection setting with variously structured agent labels while we solve a matching problem with both agents and items split into disjoint subsets. \citeAY{lang2016multi} focus on multi-attribute \emph{proportional representation} in committee selection where they essentially define diversity in terms of the divergence between the realized distribution of attribute values in the outcome and some target distribution, but admit no notion of solution quality in addition to diversity.

In a recent paper, \citeAY{raffles17immorlica} study the efficiency of lottery mechanisms such as the ones used by HDB to allocate apartments; however, their work does not account for block ethnicity constraints; as we show both theoretically and empirically, these type-block constraints can have a significant effect on allocative efficiency.

\subsection{The Singapore Public Housing Allocation System}\label{sec:hdb}
%The Singapore public housing system, managed by the HDB, provides low-cost apartments to Singapore citizens and permanent residents. Public housing is a dominant force in Singapore: as of 2017, approximately $73.3\%$ of apartments in Singapore are HDB flats \cite{Sing2017}. 
A few facts about HDB public housing, a dominant force in Singapore, are in order. New HDB flats are purchased directly from the government, which offers them at a heavily subsidized rate. New apartments are typically released at quarterly sales launches; these normally consist of plans for several estates at various locations around Singapore, an estate consisting of four or five blocks (each apartment block has approximately $100$ apartments) sharing some communal facilities (e.g. a playground, a food court, a few shops etc.). Estates take between 3 to 5 years to complete, during which HDB publicly advertises calls to ballot for an apartment in the new estate. A household (say, a newly married couple looking for a new house) would normally ballot for a few estates (balloting is cheap: only S\$10 per application \cite{HDB15cost}). HDB allocates apartments using a lottery: all applicants to a certain estate choose their flat in some random order; they are only allowed to select an apartment in a block such that their ethnic quota is not reached. 

%\mithun{I added notes about actual HDB lottery here.}
The lottery mechanism actually employed by HDB has further necessary complications:
HDB has elaborate eligibility criteria\footnote{\url{http://www.hdb.gov.sg/cs/infoweb/residential/buying-a-flat/new/hdb-flat}} as well as privilege and priority schemes\footnote{\url{http://www.hdb.gov.sg/cs/infoweb/residential/buying-a-flat/new/eligibility/priority-schemes}} that take into account sales launch types, flat types, and relevant attributes of the applicants, e.g. first-timers and low-income families usually have improved chances of being balloted for a flat;
%first-time applicants and low-income families usually receive priority numbers in the lottery scheme; 
moreover, the same estate may have several balloting rounds in order to ensure that all apartments are allocated by the time of completion. However, the focus of this work is on the welfare effects of using ethnic quotas rather than the intricacies of the HDB lottery mechanism. Hence, we use a simplified version of the HDB lottery mechanism where applicants are selected one by one uniformly at random from the remaining pool and assigned the available flat which they value the most, respecting ethnic quotas (see Section~\ref{sec:experiments}). 

We must mention the existing literature on the documentation of Singapore's residential desegregation policies \cite{chua1991race,deng2013publichousing,kim2013singapore} and the empirical evaluation of their impact on various socioeconomic factors \cite{sim2003public,wong2014estimating}; to the best of our knowledge, ours is the first formal approach towards this problem.

% {\color{orange} Recently, \citeAY{ahmed2017diverse} addressed diversity issues in matching, but they study a bipartite weighted $b$-matching problem where only one part of the graph is clustered and there is an upper and a lower bound on the number of nodes that each node can be matched with, and optimize a different diversity objective: a quadratic function of weights that encourages the representation of multiple clusters in the matching. Their motivating example is the matching of papers to (preferably diverse) sets of reviewers. On the other hand, we impose diversity in the form of type-block constraints on a classical assignment problem (each node is matched to at most one node) with the utilitarian welfare objective and, to the best of our knowledge, are the first to study the effects of such constraints on hardness and economic efficiency.}

\subsection{Public School Choice in Chicago, U.S.A.}\label{sec:chicago}
Many school districts across the U.S.A. employ a variety of strategies for promoting student diversity \cite{USedu2017,kahlenberg2016school}, e.g. controlled choice systems wherein parents are allowed to apply for options beyond their neighborhood schools, thereby counteracting underlying residential segregation. Following restrictions placed on the explicit use of race in defining diversity goals in school choice by the U.S. Supreme Court in 2007, it has been common to use some indicator of the \emph{socio-economic status} (SES) of a family in integration efforts. The system in Chicago, IL, is a notable example.

Chicago Public Schools (CPS) is one of the largest school districts in the U.S.A.\footnote{\url{http://www.cps.edu/About_CPS/At-a-glance/Pages/Stats_and_facts.aspx}}, overseeing more than 600 schools of various types: neighborhood schools, selective schools, magnet schools, and charter schools.\footnote{\url{http://cpstiers.opencityapps.org/about.html}} The application and selection processes for these schools \cite{schools2017chicago} may involve a number of computerized lotteries with no diversity component, e.g. sibling lottery, proximity lottery, school staff preference lottery; however, a significant number of entry-level seats in magnet and selective enrollment schools are filled by lotteries based on a tier system. We briefly describe its operation as follows. A composite SES score is computed for each of the census tracts that Chicago is divided into, based on six factors (median family income, adult education level, home-ownership rate, single-parent family rate, rate of English-speaking, and neighborhood school performance), and each tract is placed in one of four tiers based on its score. The maximum and minimum scores defining a tier are set in such a way that (roughly) a quarter of school-aged children end up in each tier, with Tier 1 having the lowest scores. The tier of a child is determined by the residential address furnished by the parents. Of the seats in each school earmarked for a \emph{citywide SES lottery} or \emph{general lottery}, an equal number is allocated to each tier. There is an upper limit on the number of schools that a child can apply to, and each applicant is entered into a lottery for each school they apply to, for their own tier (thus, there is a lottery per school per tier); an applicant, who comes up in the lottery and accepts the offer from the school under consideration, is removed from all lotteries. If the size of the applicant pool from a tier to a school falls short of the number of its allocated seats for that tier at any stage, ``the unfilled seats will be divided evenly and redistributed across the remaining tier(s) as the process continues" \cite{schools2017chicago}.

For an empirical study of the impact of Chicago's diversity-promoting measures on integration and student outcomes, the interested reader is referred to \cite{quick2016chicago} and citations therein.
% Thus, unlike in the HDB lottery, the quotas here are, in general, heterogeneous and dynamic. For our simulations, we will assume that they are static but determined \emph{after} the total applicant pool to every school is observed.
%If I accept a spot at one school, I am removed from the queue at every other school also, so that the applicant pool to each school keeps changing during the process.

\section{Preliminaries}\label{sec:prelim}
We first describe a formal model for the allocation problem with diversity quotas. 
Throughout the paper, given $s \in \N$, we denote the set $\{1, 2,\ldots, s\}$ by $[s]$.

\begin{definition}[\ASTC]\label{def:AwTC}
	An instance of the Assignment with Type Constraints (\ASTC) problem is given by: 
	\begin{enumerate}[(i)]
		\item \label{1st} a set $N$ of $n$ agents partitioned into $k$ types $N_1,\dots, N_k$,
		\item a set $M$ of $m$ items/goods partitioned into $l$ blocks $M_1,\dots,M_l$,
		\item a utility $u(i,j) \in \R_+$ for each agent $i \in N$ and each item $j \in M$,
		\item \label{last} a capacity $\lambda_{pq}\! \in\! \N$ for all $(p,q)\! \in [k] \! \times \! [l]$, indicating the upper bound on the number of agents of type $N_p$ allowed in the block $M_q$. \defend
	\end{enumerate} 
\end{definition}
Without loss of generality, we assume that the inequality $\lambda_{pq} \le |M_q|$ holds for all type-block pairs $(p,q) \in [k] \times [l]$, since it is not possible to assign more than $|M_q|$ agents of type $N_p$ to block $M_q$ by definition. 
%In the Singapore public housing allocation problem, potential occupant households are the agents and apartments are the items; types correspond to ethnic groups (Chinese, Malay, Indian/Others) and blocks to actual apartment blocks in a sales launch. 
In general, agents types could be based on any criterion such as gender, profession, or geographical location. 
We consider the idealized scenario where we have a central planner who has access to the utilities of each agent for all items, and determines an assignment that maximizes social welfare under type-block constraints. 

A few words about the type-block capacities are in order. Note that our analysis is agnostic to how these capacities are determined and just treats the vector $\{\lambda_{pq}\}_{p\in[k],q\in[l]}$ as a problem input.\footnote{The Singapore EIP percentage caps consider various factors such as ``[t]he racial composition of the population[,] $[\ldots]$ the rate at which new households are being formed in each one of the racial groups and the present composition of applications'' \cite{Parl1989}, but these aspects of the problem are beyond the scope of the present work.} Moreover, neither do we assume inequalities of the form $\lambda_{pq} \le |N_p|$ nor is there any positive lower bound on the number of assignments for any type-block pair: this is in keeping with the actual HDB housing problem where $\lambda_{pq}$'s are fixed by policy (as percentages of block size) even before observing the applicant pool so that capacities larger than the size of an ethnic group are possible. Adding lower bounds \emph{a priori} may render the problem infeasible if there not enough applicants of a certain type.\footnote{\citeAY{fragiadakis2017improving} show that, for some assignment problems with actual floor and ceiling constraints for each type-block pair, where the agent population is known beforehand and there is a guarantee that no agent remains unassigned, it is possible to reformulate the problem constraints in terms of ``artificial caps'' (modifying block sizes as well as type-block ceilings) and no floors: our analysis applies to these problems in this modified form.}

An assignment of items to agents can be represented by a $(0,1)$-matrix $X \! = \! (x_{ij})_{n \times m}$ where $x_{ij}\! = \! 1$ if and only if item $j$ is assigned to agent $i$; a feasible solution is an assignment in which each item is allocated to at most one agent, and each agent receives at most one item, respecting the type-block capacities defined in \eqref{last}. 
We define the objective value (or total utility) as the utilitarian social welfare, i.e. the sum of the utilities of all agents in an assignment $u(X) \triangleq \sum_{i \in N}\sum_{j \in M} x_{ij} u(i,j)$.
Clearly, this optimization problem can be formulated as the following integer linear program:
\begin{align}
\max  			& & \sum_{i \in N}\sum_{j \in M} x_{ij} u(i,j) 		&						  & \label{eq:matchingdiversityILP}\\
\mathit{s.t.} & & \sum_{i \in N_p} \sum_{j \in M_q} x_{ij} &\le \lambda_{pq} 	  & \forall p \in [k], \forall q \in [l] \label{eq:capcityconstraint}\\
& &  \sum_{j \in M} x_{ij} 								&\le 1 						& \forall i \in N\label{eq:uniqueapt}\\
& & \sum_{i \in N}  x_{ij}  								&\le 1 						& \forall j \in M  \label{eq:uniquetenant}\\
& &x_{ij}  														&\in\{0,1\} 				& \forall i \in N, \forall j \in M\label{eq:integerallocation}
\end{align}
where constraints~(\ref{eq:uniqueapt}-\ref{eq:integerallocation}) jointly ensure that $X$ is a matching of items to agents, and inequalities~\eqref{eq:capcityconstraint} embody our type-block constraints.

Finally, an instance of the decision version of \ASTC consists of parameters~\eqref{1st} to \eqref{last} in Definition~\ref{def:AwTC}, as well as a positive value $U$: it is a `yes'-instance iff there exists a feasible assignment, satisfying constraints~(\ref{eq:capcityconstraint}-\ref{eq:integerallocation}), whose objective value is at least $U$.

\section{The Complexity of the Assignment Problem with Type Constraints}\label{sec:complexity}
The hardness and approximation results in this section are based on a deep connection between \ASTC and the known NP-complete problem, Bounded Color Matching problem \cite{Garey1979}, defined as follows:
\begin{definition}[\BCM]\label{def:BCM}
	An instance of the Bounded Color Matching (\BCM) problem is given by 
	\begin{inparaenum}[(i)]
		\item a bipartite graph $G=(A \cup B,E)$, where the set of edges $E$ is partitioned into $r$ subsets $E_1, \ldots, E_r$ representing the $r$ different edge colors,
		\item a capacity $w_t \in \N$ for each color $t \in [r]$,
		\item a profit $\pi_e \in \Q_+$ for each edge $e \in E$, and
		\item a positive integer $P$ called the threshold.
	\end{inparaenum} 
	It is a `yes'-instance iff there exists a matching (i.e. a collection of pairwise non-adjacent edges) $E' \subseteq E$ such that the sum of the profits of all edges in the matching is at least $P$,  and there are at most $w_t$ edges of color $t$ in it, i.e. $\sum_{e \in E'} \pi_e \ge P$ and $|E' \cap E_t| \le w_t$ for all $ t \in [r]$. \defend
\end{definition}
The following lemma, together with its proof, establishes that \ASTC is, in fact, a special case of \BCM.
\begin{lemma}\label{lem:sred}
	There exists an S-reduction \cite{crescenzi1991note,crescenzi1997short} from the \ASTC problem to the \BCM problem.
\end{lemma}
\begin{proof}
	Given an instance of the \ASTC problem, we construct a graph as follows. We define a node corresponding to each agent in $N$ and also each item in $M$; we draw an edge from agent-node $i$ to item-node $j$ and make the associated profit equal to the utility $u(i,j)$ for all $i\in N$, $j \in M$. This gives us a
	complete bipartite graph with bipartition $(N,M)$.
	We also give all edges joining agents of one type to items in one block the same color, one unique color for each agent-item pair. Thus, there are $kl$ colors indexed lexicographically by pairs $(p,q) \in [k] \times [l]$. We set the capacity for color $(p,q)$ at $\lambda_{pq}$. This produces, in $O(mn)$ time, an instance of \BCM (Definition~\ref{def:BCM}). The size of this instance is obviously polynomial in that of the original. By construction, there is a one-to-one correspondence between the sets of feasible solutions of the original and reduced instances with each corresponding pair having the same objective value (sum of edge-profits/utilities). Hence, the optimal values of the instances are also equal.
\end{proof}
This result will be useful in Section~\ref{sec:approx} for obtaining approximation guarantees. But it does not settle the question of the hardness of \ASTC since a special case of an NP-complete problem may be tractable. We could consider \ASTC and \BCM to be identical if there existed a trivial reduction from \BCM to \ASTC in the following sense. For a \BCM instance, define an agent for each node in one part of the graph and an item for each node in the other; make the utility of an agent-item pair equal to the profit of the edge joining the corresponding pair of nodes; define types of agents and blocks of items such that each type-block pair corresponds to a unique edge-color. Unfortunately, the last part may not be possible depending on how the edges are colored, as illustrated by the following example. 
\begin{example}\label{ex:BCM_notASTC} 
	Consider the instance of the \BCM problem defined on the graph in Figure~\ref{fig:BCM_notASTC}. There are three colors: $(a_1,b_1)$ and $(a_1,b_2)$ are blue, $(a_2,b_1)$ is red, and $(a_2,b_2)$ is gray. The color-capacities, the edge-profits, and the threshold are arbitrary, hence omitted. 
	\begin{figure}[!h]
		\centering
		\footnotesize
		\begin{tikzpicture}[scale=0.7]
		
		\node[draw,circle,minimum height=0.95cm, thick] (A) at (0,0){$a_1$};
		\node[draw,circle,minimum height=0.95cm, thick] (APrim) at (0,-2){$a_2$};
		
		\node[draw,circle,minimum height=0.95cm, thick] (B) at (3,0){$b_1$};
		\node[draw,circle,minimum height=0.95cm, thick] (BPrim) at (3,-2){$b_2$};
		
		\draw[ultra thick,blue] (A) -- (B) node[midway,above,black]{};
		\draw[ultra thick,blue] (A) -- (BPrim) node[pos=0.2,above,black]{};
		\draw[ultra thick,red] (APrim) -- (B) node[pos=0.1,above,black]{};
		\draw[ultra thick,gray] (APrim) -- (BPrim) node[midway,above,black]{};
		
		\end{tikzpicture}
		
		\caption{\textnormal{An instance of \BCM that does not trivially reduce to \ASTC.}
			\label{fig:BCM_notASTC}}
	\end{figure}

Let the sets of agents and items be $N=\{a_1,a_2\}$ and $M=\{b_1,b_2\}$ respectively. Let us now try to define types and blocks that are consistent with edge-colors.
Recall that, to achieve the desired reduction, no two edges joining $a_1$ to items in different blocks can have the same color. Since edges $(a_1,b_1)$ and $(a_1,b_2)$ have the same color, $b_1$ and $b_2$ must be in the same block. This also implies that all edges joining $b_1$ and $b_2$ to $a_2$ must be of the same color, regardless of whether $a_1$ and $a_2$ belong to the same type. However, $(a_2,b_1)$ and $(a_2,b_2)$ are of different colors -- a contradiction. \egend
\end{example}
However, we will now show that there does exist a non-trivial polynomial-time reduction in the desired direction where edges of the graph map to agents (the colors corresponding to types) and items are defined in a more complicated way based on the nodes (the two parts of the bipartition corresponding to blocks)!
\begin{theorem} \label{theoNP}
The \ASTC problem is NP-complete. 
\end{theorem}
We will prove this by describing a polynomial-time reduction from \BCM to the decision problem we introduced in Section~\ref{sec:prelim}.
\begin{proof}
That the problem is in NP is immediate: given an assignment, one can verify in poly-time that it satisfies the problem constraints and compute total social welfare.  
Given an instance $\tup{G;\vec w;\vec \pi;P}$ of \BCM, we construct an instance of the \ASTC problem as follows (see Example~\ref{exampleReduction} for an illustration). Each edge $e \in E$ is an agent, whose type is its color. 
Items in our construction are partitioned into two blocks: $M_1$ and $M_2$. The items in block $M_1$ correspond to the vertices in $B$: there is one item $j_{b}$ for each node $b\in B$. For every $a \in A$, we add $\deg(a)-1$ items $j_a^1,\dots,j_a^{\deg(a) - 1}$ to $M_2$, for a total of $|E| - |A|$ items. Thus, there is a total of $m  = |B| + |E| - |A|$ items. 
Block $M_1$ accepts at most $w_p$ agents of type $N_p$, whereas block $M_2$ has unlimited type-block capacity; in other words, $\lambda_{p1} = w_p$ and $\lambda_{p2} = \min \{|N_p|,|M_2|\}$ for all $p \in [k]$. Given $e = (a,b)$, we define the utility function of agent $e$ as follows: 
$$u(e,j) = 
\begin{cases}
\pi_e & \mbox{if } j = j_b,\\
\Phi & \mbox{if } j = j_a^s \mbox{ for some } s \in [\deg(a) - 1],\\
0 & \mbox{otherwise.}
\end{cases}
$$
Here, $\Phi$ is an arbitrarily large constant, e.g. $\Phi = 1+ \sum_{e \in E} \pi_e$. Finally, let $U = P + \Phi(|E| - |A|)$; that is, our derived \ASTC instance is a ``yes'' instance iff there is some assignment of items to agents such that the social welfare exceeds $U$.

We begin by showing that if the original $\BCM$ instance is a `yes' instance, then so is our constructed $\ASTC$ instance. Let $E' \subseteq E$ be a valid matching whose value is at least $P$; let us construct an assignment $X$ of items to agents via $E'$ as follows.
Observe some node $a \in A$; if $(a,b) \in E'$ then we assign the item $j_b\in M_1$ to the agent $(a,b)$; the remaining $\deg(a) - 1$ agents of the form $(a,b')$, with $b'\in B$, are arbitrarily assigned to the items $j_a^1,\dots,j_a^{\deg(a) - 1} \in M_2$. If $E'$ contains no edges adjacent to $a$, then we arbitrarily choose $\deg(a) - 1$ edges adjacent to $a$ and assign the corresponding agents to the items $j_a^1,\dots,j_a^{\deg(a) - 1}$. 
We now show that this indeed results in a valid assignment satisfying the type-block constraints.
 
First, by construction, every agent $(a,b)$ is assigned at most one item. Moreover, since $E'$ is a matching, every item $j_b\in M_1$ is assigned to at most one agent of the form $(a,b)$; hence, every item in $M_2$ is assigned to at most one agent. 

Let $E_p'=E_p \cap E'$ be the edges of color $p$ in $E'$. Since the matching $E'$ satisfies the capacity constraints of the $\BCM$ instance, we have $|E_p'| \le w_p$ for all $p \in [k]$; in particular, the number of items in $M_1$ assigned to agents of type $p$ is no more than $w_p = \lambda_{p1}$. Thus, the type-block constraints for $M_1$ are satisfied. On the other hand, the type-block constraints for $M_2$ are trivially satisfied. We conclude that our constructed assignment is indeed valid, and satisfies the type-block constraints.

Finally, we want to show that total social welfare exceeds $U$ the prescribed bound. Let us fix a node $a \in A$. By our construction, if the edge $e=(a,b)$ is in the matching $E'$, then agent $e$ is assigned the item $j_b$ for a utility of $\pi_e$. Thus the total welfare of agents in $E'$ equals $\sum_{e \in E'} \pi_e$, which is at least $P$ by choice of $E'$.
In addition, for every $a \in A$, there are exactly $\deg(a) - 1$ agents assigned to items in $M_2$ for a total utility of $\Phi(\deg(a) - 1)$. 
Summing over all $a \in A$, we have that the total utility derived by agents in $E \setminus E'$ is 
\begin{align*}
\sum_{a \in A}\Phi(\deg(a)-1) &= \Phi\left(\sum_{a \in A}\deg(a) - \sum_{a \in A}1\right)= \Phi(|E| - |A|).
\end{align*}
Putting it all together, we have that the total utility obtained by our assignment is at least $P+ \Phi(|E| - |A|) =U$. 

Next, we assume that our constructed $\ASTC$ instance is a `yes' instance, and show that the original $\BCM$ instance must also be a `yes' instance.  Let $X$ be a constrained assignment whose social welfare is at least $U=P+ \Phi(|E| - |A|)$. Let $E'$ be the set of edges corresponding to agents $(a,b)$ assigned to items in $M_1$; we show that $E'$ is a valid matching whose value is at least $P$. First, for any $b \in B$, $X$ must assign the item $j_b$ to at most one agent $e \in E'$. Next, since $\Phi$ is greater than the total utility obtainable from assigning all items in $M_1$, it must be the case that $X$ assigns all items $j_{a}^1,\dots,j_a^{\deg(a) - 1}$ to $\deg(a) - 1$ agents of the form $(a,b)$, with $b \in B$, for every node $a \in A$; thus, there can be one edge in $E'$ that is incident on $a$ for every $a \in A$. Next,
since $X$ satisfies the type-block constraints, we know that for every $p \in [k]$, there are at most $\lambda_{p1}=w_p$ agents from $E_p$ that are assigned items in $M_1$; thus, $E'$ satisfies the capacity constraints. 
Finally, the utility extracted from the agents assigned to items in $M_2$ is exactly $\Phi(|E| - |A|)$; the total utility of the matching $X$ is at least $U = P + \Phi(|E| - |A|)$, thus $E'$ has a total profit of at least $P$ in the original $\BCM$ instance, and we are done.
\end{proof}

\begin{example}\label{exampleReduction} In Figure \ref{fig:exampleReduction}, the graph 
$G=(A\cup B, E_1 \cup E_2)$, with $A=\{a_1,a_2\}$, $B=\{b_1,b_2,b_3\}$, $E_1=\{(a_1,b_1), (a_2,b_2)\}$ and $E_2=\{(a_1,b_2), (a_2,b_1), (a_2,b_3)\}$, is an instance of the $\BCM$ problem; edge labels are profits. 
The associated instance of the \ASTC problem is defined by $N=N_1 \cup N_2$ and $M=M_1 \cup M_2$, where $N_1=\{(a_1,b_1),(a_2,b_2)\}$, $N_2=\{(a_1,b_2),(a_2,b_1),(a_2,b_3)\}$, $M_1=\{j_{b_1},j_{b_2},j_{b_3}\}$ and $M_2=\{j^1_{a_1},j^1_{a_2},j^2_{a_2}\}$;
the utility of an agent for an item is equal to $0$ if there is no edge between them, to $\Phi$ if the edge is dashed, and to the edge label otherwise.
\begin{figure}[!h]
\centering
\footnotesize
\begin{tikzpicture}[scale=0.7]

\node[draw,circle,minimum height=0.95cm, thick] (A) at (0,0){$a_1$};
\node[draw,circle,minimum height=0.95cm, thick] (APrim) at (0,-2){$a_2$};

\node[draw,circle,minimum height=0.95cm, thick] (B) at (3,1){$b_1$};
\node[draw,circle,minimum height=0.95cm, thick] (BPrim) at (3,-1){$b_2$};
\node[draw,circle,minimum height=0.95cm, thick] (BDPrim) at (3,-3){$b_3$};

\draw[ultra  thick,blue] (A) -- (B) node[midway,above,black]{$2$};
\draw[ultra  thick,red] (A) -- (BPrim) node[pos=0.2,above,black]{$6$};
\draw[ultra thick,red] (APrim) -- (B) node[pos=0.1,above,black]{$3$};
\draw[ultra thick,blue] (APrim) -- (BPrim) node[midway,above,black]{$1$};
\draw[ultra thick,red] (APrim) -- (BDPrim) node[midway,above,black]{$4$};

\node[draw,circle,minimum height=0.95cm, thick] (JB) at (5.5,1){$j_{b_1}$};
\node[draw,circle,minimum height=0.95cm, thick] (JBPrim) at (5.5,-1){$j_{b_2}$};
\node[draw,circle,minimum height=0.95cm, thick] (JBDPrim) at (5.5,-3){$j_{b_3}$};

\node[draw,circle,minimum height=0.95cm, thick,fill=blue!30] (IAB) at (8,3){$(a_1,b_1)$};
\node[draw,circle,minimum height=0.95cm, thick,fill=red!30] (IABPrim) at (8,1){$(a_1,b_2)$};

\node[draw,circle,minimum height=0.95cm, thick,fill=red!30] (IAPrimB) at (8,-1){$(a_2,b_1)$};
\node[draw,circle,minimum height=0.95cm, thick,fill=blue!30] (IAPrimBPrim) at (8,-3){$(a_2,b_2)$};
\node[draw,circle,minimum height=0.95cm, thick,fill=red!30] (IAPrimBDPrim) at (8,-5){$(a_2,b_3)$};

\node[draw,circle,minimum height=0.95cm, thick] (JA1) at (10.5,1){$j_{a_1}^1$};
\node[draw,circle,minimum height=0.95cm, thick] (JAPrim1) at (10.5,-1){$j_{a_2}^1$};
\node[draw,circle,minimum height=0.95cm, thick] (JAPrim2) at (10.5,-3){$j_{a_2}^2$};

\draw[thick] (JB) -- (IAB)  node[midway,above,black]{$2$};
\draw[thick] (JB) -- (IAPrimB) node[pos=0.3,above,black]{$3$};
\draw[thick] (JBPrim) -- (IABPrim) node[pos=0.1,above,black]{$6$};
\draw[thick] (JBPrim) -- (IAPrimBPrim) node[midway,above,black]{$1$};
\draw[thick] (JBDPrim) -- (IAPrimBDPrim) node[midway,above,black]{$4$};

\draw[thick,dashed] (IAB) -- (JA1);
\draw[thick,dashed] (IABPrim) -- (JA1);

\draw[thick,dashed] (IAPrimB) -- (JAPrim1);
\draw[thick,dashed] (IAPrimBPrim) -- (JAPrim1);
\draw[thick,dashed] (IAPrimBDPrim) -- (JAPrim1);

\draw[thick,dashed] (IAPrimB) -- (JAPrim2);
\draw[thick,dashed] (IAPrimBPrim) -- (JAPrim2);
\draw[thick,dashed] (IAPrimBDPrim) -- (JAPrim2);

\end{tikzpicture}

\caption{\textnormal{A reduction from \BCM to \ASTC.}
\label{fig:exampleReduction}}
\end{figure}
\egend
\end{example}

\subsection{Polynomial-Time Constant-Factor Approximation}\label{sec:approx}
Having established that the \ASTC problem is computationally intractable in general, we next present an efficient constant-factor approximation algorithm, based on a known polynomial-time approximation algorithm for the \BCM problem. 
\begin{theorem} \label{GCMAPX}
There exists a polynomial-time $\tfrac12$-approximation algorithm for the \ASTC problem.
\end{theorem}
\begin{proof}
	Note that S-reduction is an approximation-preserving reduction \cite{orponen87approximation,crescenzi1997short}. Thus, given an instance of the \ASTC problem, we transform it into the corresponding \BCM instance in accordance with Lemma~\ref{lem:sred}, and 
	then apply the polynomial-time $\tfrac{1}{2}$-approximation algorithm introduced by \citeAY{stamoulis2014approximation} for \BCM on general weighted graphs.
\end{proof}
Theorem~\ref{GCMAPX} does not prove that $\tfrac12$ is the best approximation-ratio possible for the \ASTC problem. It is left for future work to investigate whether a better polynomial-time approximation algorithm exists.

\subsection{Uniformity Breeds Simplicity: Polynomial-Time Special Cases}\label{sec:polytime} 
Our results thus far make no assumptions on agent-item utilities;
as we now show, the \ASTC problem admits a polynomial-time algorithm under some assumptions on the utility model.

\begin{definition}[Type-uniformity and Block-uniformity]\label{def:typeU_blockU}
A utility model $u$ is called \emph{type-uniform} if all agents of the same type have the same utility for each item, i.e. for all $p \in [k]$ and for all $j\in M$, there exists $U_{pj} \in \R_+$ such that $u(i,j) = U_{pj}$ for all $ i \in N_p$. A utility model $u$ is called \emph{block-uniform} if all items in the same block offer the same utility to every agent; that is, for all $q \in [l]$ and for all $i \in N$, there exists $U_{iq} \in \R_+$ such that $u(i,j) = U_{iq}$ for all $j \in M_q$. \defend
\end{definition}
In the context of the HDB allocation problem, type uniformity implies that Singaporeans of the same ethnicity share the same preferences over apartments (perhaps due to cultural or socioeconomic factors). Cases that deal with uniform goods satisfy the block-uniformity assumption: e.g. students applying for spots in public schools or job applicants applying for multiple (identical) positions; in the HDB domain, block-uniformity captures purely location-based preferences, i.e. a tenant does not care which apartment she gets as long as it is in a specific block close to her workplace, family, or favorite public space. %items have some intrinsic (e.g. monetary) value, and are partitioned into blocks based on these values
\begin{theorem} \label{theoPoly1} 
	The \ASTC problem can be solved in $\poly(n,m)$ time under either a type-uniform or a block-uniform utility model.
\end{theorem}
We prove the result for a type-uniform utility model; the result for block-uniform utilities can be similarly derived. 
We propose a polynomial time algorithm based on the {\em Minimum-Cost Flow} problem which is known to be solvable in polynomial time. Recall that a flow network is a directed graph $G=(V,E)$ with a source node $s \in V$ and a sink node $t \in V$, where each arc $(a,b)\in E$ has a cost $\gamma(a,b) \in \R$ and a capacity $\psi(a,b) > 0$ representing the maximum amount that can flow on the arc; for convenience, we set $\gamma(a,b) = 0$ and $\psi(a,b) = 0$ for all $a,b \in V$ such that $(a,b) \not \in E$. Let us denote by $\Gamma$ and $\Psi$ the matrices of costs and capacities respectively defined by $\Gamma = (\gamma (a, b))_{|V|\times|V|}$ and $\Psi = (\psi (a, b))_{|V|\times|V|}$. 
A flow in the network is a function $f:V \times V \rightarrow \R_+$ satisfying: 
\begin{enumerate}[(i)]
	\item $f(a,b) \le \psi(a,b)$ for all $a,b \in V$ (capacity constraints),
	\item $f(a,b) = -f(b,a)$ for all $a,b \in V$ (skew symmetry), and 
	\item $\sum_{b \in V} f(a,b) = 0$ for all $a \in V \backslash \{s,t\}$ (flow conservation).
\end{enumerate}
The value $v(f)$ of a flow $f$ is defined by $v(f)=\sum_{a \in V} f(s,a) = \sum_{a \in V} f(a,t)$ and its cost is given by 
$\gamma(f)=\sum_{(a,b)\in E} f(a,b) \gamma(a,b)$.
The optimization problem can be formulated as follows. Given a value $F$, find a flow $f$ that minimizes the cost $\gamma(f)$ subject to $v(f) = F$. This optimization problem that takes as input the graph $G =(V,E)$, the matrices $\Gamma$ and $\Psi$, and the value $F$, will be denoted by \MCF hereafter; given an instance $\tup{G;\Gamma;\Psi;F}$ of the \MCF problem, we let $\gamma(G,\Gamma,\Psi,F)$ be the cost of the optimal flow for that instance. 

Given an instance $\cal I$ of \ASTC, we construct a flow network $G_{\cal I}(V,E)$ and matrices $\Gamma_{\cal I}$ and $\Psi_{\cal I}$ as follows  (see Figure \ref{exampleReductionPoly1} for an illustration). 
The node set $V$ is partitioned into layers: $V=\{s\}\cup A\cup B\cup C \cup \{t\}$. $A$ is the {\em agent type layer}: there is one node $a_p \in A$ for all agent types $N_p, p \in [k]$. $B$ is the {\em type-block layer}: it has a node $b_{pq} \in B$ for every type-block pair $(p,q) \in [k] \times [l]$. Finally, $C$ is the {\em item layer}: there is one node $c_j \in C$ for all items $j \in M$. The arcs in $E$ are as follows: for every $a_p$ in $A$, there is an arc from $s$ to $a_p$ whose capacity $\psi(s,a_p)$ is $|N_p|$. Fixing $p \in [k]$, there is an arc from $a_p \in A$ to every $b_{pq} \in B$, where the capacity of $(a_p,b_{pq})$ is the quota for type $N_p$ in block $M_q$ (i.e.  $\psi(a_p,b_{pq})=\lambda_{pq}$). Finally, given $q \in [l]$, there is an arc from $b_{pq}$ to $c_j$ iff $j \in M_q$; in that case, we have $\psi(b_{pq},c_j) = 1$. 
The costs associated with arcs from $B$ to $C$ (i.e. arcs of the form $(b_{pq},c_j)$ where $j \in M_q$) are $-U_{pj}$; recall that $U_{pj}$ is the utility that every agent of type $N_p$ assigns to item $j$. All other arc costs are set to $0$. We begin by proving a few technical lemmas on the above network.

Given a positive integer $F$, there exists an optimal flow that is integer-valued since $\tup{G_{\cal I};\Gamma_{\cal I};\Psi_{\cal I};F}$ is integer-valued as well. 
Let $f^*$ be an integer-valued optimal flow, taken over all possible values of $F$; that is:
\begin{align}
f^* \in  \underset{F \in [n]}{\argmin} \, \gamma(G_{\cal I},\Gamma_{\cal I},\Psi_{\cal I},F) \label{prop:fstar-poly}
\end{align}
Finding the flow $f^*$ involves solving $n$ instances of \MCF by definition; thus, one can find $f^*$ in polynomial time. 
Given $f^*$ as defined in (\ref{prop:fstar-poly}), let $X^* = (x_{ij}^*)_{n\times m}$ be defined as follows: for every item $j \in M_q$, if $f^*(b_{pq},c_j) = 1$ for some $p \in [k]$, then we choose an arbitrary unassigned agent $i \in N_p$ and set $x_{ij}^* = 1$. 
\begin{lemma}\label{lem:Xstar-assignment}
$X^*$ is a feasible solution of the \ASTC instance $\cal I$.
\end{lemma}
\begin{proof}
	First, we assign at most one item to every agent by construction; next, let us show that each item $j\in M_q$ is assigned to at most one agent. 
	Since $f^*$ is a flow, we have $\sum_{p=1}^k f^*(b_{pq},c_j) = f^*(c_j,t)$ due to flow conservation; note that the capacity of the arc $(c_j,t)$ is $1$, thus at most one arc $(b_{pq},c_j)$ has $f^*(b_{pq},c_j) = 1$. Finally, since item $j$ is assigned to an agent in $N_p$ iff $f^*(b_{pq},c_j) = 1$, we conclude that item $j$ is assigned to at most one of the agents in $N$.
	
	Next, let us prove that assignment $X^*$ satisfies the type-block constraints; in other words, we need to show that:
	\begin{align}
	\sum_{i \in N_p}\sum_{j \in M_q} x_{ij}^* \le \lambda_{pq}, \, \forall p \in [k], \forall q \in [l]\label{eq:capacity-poly}
	\end{align}
	Since $f^*$ is a flow, we have $f^*(a_p,b_{pq}) = \sum_{j\in M_q} f^*(b_{pq},c_j)$  for every type-block pair $(p,q) \in [k] \times [l]$ due to flow conservation; moreover, we have $f^*(a_p,b_{pq}) \le  \psi(b_{pq},c_j)  = \lambda_{pq}$ by construction.
	As a consequence, we necessarily have $\sum_{j\in M_q} f^*(b_{pq},c_j) \le \lambda_{pq}$ for all $p \in [k]$. 
	Since an item $j \in M_q$ is matched with some agent $i\in N_p$ if and only if we have $f^*(b_{pq},c_j) = 1$, we conclude that \eqref{eq:capacity-poly} indeed holds.
\end{proof}
Now, let us establish a relation between the cost of $f^*$ and the utility of the feasible assignment $X^*$.
\begin{lemma}\label{lem:Xstar-fstar-same-util}
The cost of the flow $f^*$ satisfies $\gamma(f^*) = -u(X^*)$.
\end{lemma}
\begin{proof}
By construction, the cost of $f^*$ can only be induced by arcs from nodes in $B$ to nodes in $C$, where the cost of all arcs of the form $(b_{pq},c_j)$, with $j \in M_q$, is equal to $-U_{pj}$ (the negative of the uniform utility derived from item $j$ by members of $N_p$). In other words, the cost of $f^*$ can be written as follows:
\begin{align}
\gamma(f^*)=-\sum_{p =1}^k\sum_{q = 1}^l \sum_{j \in M_q} f^*(b_{pq},c_j)U_{pj} \notag
\end{align}
As previously argued, we have that $f^*(b_{pq},c_j) \in \{0,1\}$ for all arcs $(b_{pq},c_j)$; moreover, $f^*(b_{pq},c_j) = 1$ iff item $j$ is assigned to some agent in $N_p$. Therefore, we obtain: 
\begin{align}
\gamma(f^*)=
-\sum_{p =1}^k\sum_{i \in N_p}\sum_{j \in M} x_{ij}^* U_{pj} = -\sum_{i \in N}\sum_{j \in M}x_{ij}^* u(i,j)= - u(X^*)\notag
\end{align}
where the second equality holds since all agents in $N_p$ have the same utility by assumption.
\end{proof}
Finally, we show that for every feasible solution to the $\ASTC$ instance $\cal I$, there exists a flow with a matching cost.

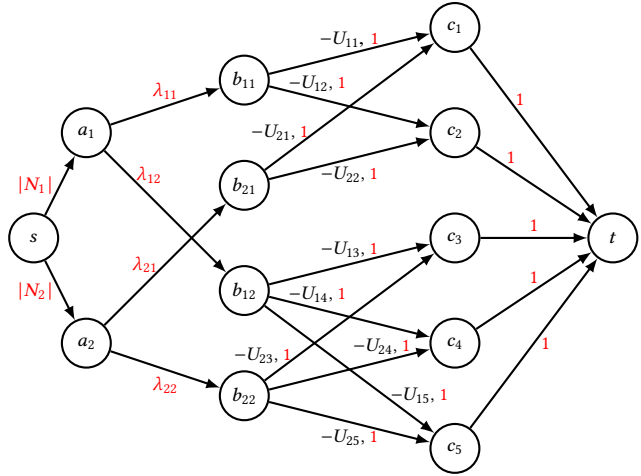
\begin{figure}[t]
\centering
	\footnotesize
	\begin{tikzpicture}[scale=0.7]
	
	\node[draw,circle,minimum height=0.65cm, thick] (S) at (-1,0){$s$};
	
	\node[draw,circle,minimum height=0.65cm, thick] (B1) at (0,2){$a_1$};
	\node[draw,circle,minimum height=0.65cm, thick] (B2) at (0,-2){$a_2$};

	\draw[thick,-latex] (S) -- (B1) node[midway,left]{\textcolor{red}{$|N_1|$}};
	\draw[thick,-latex] (S) -- (B2) node[midway,left]{\textcolor{red}{$|N_2|$}};

	\node[draw,circle,minimum height=0.65cm, thick] (A11) at (3,3){$b_{11}$};
	\node[draw,circle,minimum height=0.65cm, thick] (A21) at (3,1){$b_{21}$};
	\node[draw,circle,minimum height=0.65cm, thick] (A12) at (3,-1){$b_{12}$};
	\node[draw,circle,minimum height=0.65cm, thick] (A22) at (3,-3){$b_{22}$};
	
	\draw[thick,-latex] (B1) -- (A11) node[midway,above]{\textcolor{red}{$\lambda_{11}$}};
	\draw[thick,-latex] (B2) -- (A21) node[midway,left]{\textcolor{red}{$\lambda_{21}$}};
	\draw[thick,-latex] (B1) -- (A12) node[pos=0.2,right]{\textcolor{red}{$\lambda_{12}$}};
	\draw[thick,-latex] (B2) -- (A22) node[midway,below]{\textcolor{red}{$\lambda_{22}$}};

	\node[draw,circle,minimum height=0.65cm, thick] (A1) at (7,4){$c_{1}$};
	\node[draw,circle,minimum height=0.65cm, thick] (A2) at (7,2){$c_{2}$};
	\node[draw,circle,minimum height=0.65cm, thick] (A3) at (7,0){$c_{3}$};
	\node[draw,circle,minimum height=0.65cm, thick] (A4) at (7,-2){$c_{4}$};
	\node[draw,circle,minimum height=0.65cm, thick] (A5) at (7,-4){$c_{5}$};
	
	\draw[thick,-latex] (A11) -- (A1) node[midway,above]{$-U_{11},$ \textcolor{red}{1}};
	\draw[thick,-latex] (A11) -- (A2) node[pos=0.3,above]{$-U_{12}$, \textcolor{red}{1}};
	\draw[thick,-latex] (A21) -- (A1) node[pos=0.3,left]{$-U_{21}$, \textcolor{red}{1}};
	\draw[thick,-latex] (A21) -- (A2) node[midway,below]{$-U_{22}$, \textcolor{red}{1}};
	
	\draw[thick,-latex] (A12) -- (A3) node[midway,above]{$-U_{13}$, \textcolor{red}{1}};
	\draw[thick,-latex] (A12) -- (A4) node[pos=0.3,above]{$-U_{14}$, \textcolor{red}{1}};
	\draw[thick,-latex] (A12) -- (A5) node[pos=0.7,right]{$-U_{15}$, \textcolor{red}{1}};
	\draw[thick,-latex] (A22) -- (A3) node[pos=0.2,left]{$-U_{23}$, \textcolor{red}{1}};
	\draw[thick,-latex] (A22) -- (A4) node[pos=0.7,above]{$-U_{24}$, \textcolor{red}{1}};
	\draw[thick,-latex] (A22) -- (A5) node[midway,below]{$-U_{25}$, \textcolor{red}{1}};
	
	\node[draw,circle,minimum height=0.65cm, thick] (T) at (10,0){$t$};
	
	\draw[thick,-latex] (A1) -- (T) node[pos=0.3,right]{\textcolor{red}{$1$}};
	\draw[thick,-latex] (A2) -- (T) node[pos=0.3,above]{\textcolor{red}{$1$}};
	\draw[thick,-latex] (A3) -- (T) node[midway,above]{\textcolor{red}{$1$}};
	\draw[thick,-latex] (A4) -- (T) node[midway,above]{\textcolor{red}{$1$}};
	\draw[thick,-latex] (A5) -- (T) node[midway,right]{\textcolor{red}{$1$}};
	
	\end{tikzpicture}
	
	\caption{\textnormal{Network flow constructed for the proof of Theorem~\ref{theoPoly1}; in this case, we have $2$ types and $2$ blocks: $M_1 = \{1,2\}$ and $M_2=\{3,4,5\}$. Arc capacities are given in red. All arcs have a cost of $0$, except those between $b_{pq} \in B$ and $c_j \in C$ whose cost equals $-U_{pj}$.}\label{exampleReductionPoly1}}
\end{figure}

\begin{lemma}\label{lem:forall-X-exists-f}
Let $X$ be a feasible assignment for the $\ASTC$ instance $\cal I$; there exists some feasible flow $f$ such that $\gamma(f) = -u(X)$. Moreover, we have $v(f) = |\{i \in N: \sum_{j \in M}x_{ij} = 1\}|$.
\end{lemma}
\begin{proof}
	Given a feasible assignment $X=(x_{ij})_{n\times m}$, we define $f: V \times V \rightarrow \R_+$ as follows:
	$$
	\begin{cases}
	f(s,a_p) = \sum_{i \in N_p}\sum_{j \in M} x_{ij} & \forall a_p \in A\\
	f(a_p,b_{pq})= \sum_{i \in N_p}\sum_{j \in M_q} x_{ij} & \forall (a_p,b_{pq})\in E\\
	f(b_{pq}, c_j) = \sum_{i \in N_p} x_{ij} & \forall (b_{pq},c_j) \in E\\ 
	f(c_j, t) = \sum_{i \in N} x_{ij} & \forall c_j \in C\\
	f(a,b) = - f(b,a) &  \forall (a,b) \in E\\
	f(a,b)  =  0 & \forall (a,b)\notin E
	\end{cases}
	$$
	The function $f$ is indeed a flow:  
	$f$ trivially satisfies the skew symmetry condition by construction; next, we show that $f$ satisfies flow conservation.  
	For all $a_p \in A$, the incoming flow to node $a_p$ from node $s$ is $f(s,a_p) =\sum_{i \in N_p} \sum_{j \in M} x_{ij}$, and the outgoing flow to every $b_{pq}$ is 
	$\sum_{q=1}^l f(a_p,b_{pq}) = \sum_{i \in N_p} \sum_{j \in M}x_{ij}$ 
	since $M$ is partitioned into $M_1,\dots,M_l$; hence flow is conserved. 
For a node $b_{pq} \in B$, the incoming flow equals 
	$f(a_p,b_{pq})=\sum_{i \in N_p}\sum_{j \in M_q} x_{ij}$ and an amount of $f(b_{pq},c_j)=\sum_{i \in N_p}x_{ij}$ flows to every node $c_j$ such that $j \in M_q$, thus flow is conserved. For a node $c_j \in C$ such that $j \in M_q$, its incoming flow equals $f(b_{pq},c_j)=\sum_{i \in N_p}x_{ij}$ from every $b_{pq}$, for a total flow of $\sum_{p =1}^k \sum_{i \in N_p}x_{ij}$, which equals its outgoing flow to $t$. To conclude, $f$ satisfies flow conservation. 
	
	Now let us prove that $f$ satisfies the capacity constraints (i.e. $f(a,b) \le \psi(a,b)$ for all arcs $(a,b)\in E$). 
	For all $(s,a_{p})\in E$, we have $f(s,a_{p}) = \sum_{i \in N_p}\sum_{j \in M} x_{ij} \le |N_p| = \psi(s,a_{p})$ since every agent $i \in N_p$ is matched with at most one item. For all $(a_{p},b_{pq})\in E$, we have $f(a_{p},b_{pq})= \sum_{i \in N_p}\sum_{j \in M_l} x_{ij} \le \lambda_{pq} = \psi(a_{p},b_{pq})$ since $X$ satisfies the type-block constraints. For all arcs $(b_{pq},c_j)\in E$, we have $f(b_{pq}, c_j) = \sum_{i \in N_p} x_{ij} \le 1 = \psi(b_{pq}, c_j)$ since item $j$ is matched with at most one of the agents in $N_p$. 
	For all $(c_j,t) \in E$, we have $f(c_j, t) = \sum_{i \in N} x_{ij} \le 1 = \psi(c_j, t)$ since item $j$ is matched with at most one of the agents in $N$.
	Hence, $f$ satisfies the capacity constraints and is a valid flow. 
	Note that we have:
	\begin{align*}
	v(f) = \! \sum_{a \in V} f(s,a) = \! \sum_{p=1}^k f(s,a_{p}) 
	= \!  \sum_{p=1}^k \sum_{i \in N_p}\sum_{j \in M} x_{ij}= \! \sum_{i \in N} \sum_{j \in M} x_{ij}
	\end{align*} 
	Then, since $X$ is a feasible assignment of the $\ASTC$ instance $\cal I$, we conclude that we have $v(f) = |\{i \in N: \sum_{j\in M}x_{ij} = 1\}|$. 
	We just need to prove that we have $\gamma(f) = -u(X)$, and we are done. 
	By definition of the flow network, only arcs of the form $(b_{pq},c_j)$ contribute to the  cost $\gamma(f)$ and we have $\gamma(b_{pq},c_j) = -U_{pj}$; 
 therefore, $\gamma(f) = - \sum_{(b_{pq},c_j) \in E} f(b_{pq},c_j) U_{pj}$. Since $f(b_{pq},c_j)\!= \! \sum_{i \in N_p} x_{ij}$ (by definition of $f$) and $u(i,j)\! =\! U_{pj}$ for all agents $i\! \in\! N_p$ (by hypothesis), we finally obtain $\gamma(f) \! = \! - \sum_{j \in M} \sum_{p=1}^k  \sum_{i \in N_p} x_{ij} u(i,j)\! =\! - \sum_{j \in M} \sum_{i \in N} x_{ij} u(i,j) = -u(X)$.
\end{proof}
We are now ready to prove Theorem~\ref{theoPoly1}.

\begin{proof}[Proof of Theorem~\ref{theoPoly1}] 
We begin by observing the flow $f^*$ as defined in \eqref{prop:fstar-poly}, and the assignment $X^*$ derived from it. First, according to Lemma~\ref{lem:Xstar-assignment}, $X^*$ is a feasible assignment of the $\ASTC$ instance $\cal I$. Moreover, we have $u(X^*) = -\gamma(f^*)$ according to Lemma~\ref{lem:Xstar-fstar-same-util}. Finally, for any feasible assignment $X$  of the $\ASTC$ instance $\cal I$, there exists a flow $f$ such that $\gamma(f) = -u(X)$; furthermore, since $v(f) = |\{i \in N: \sum_{j \in M} x_{ij} = 1\}|\in [n]$, flow $f$ is a feasible solution of the $\MCF$ instance $\tup{G_{\cal I};\Gamma_{\cal I};\Psi_{\cal I};F}$ for some $F\in [n]$. Therefore, we have:
\begin{align*}
u(X) = -\gamma(f) \le - \gamma(G_{\cal I},\Gamma_{\cal I},\Psi_{\cal I},v(f)) \le -\gamma(f^*) = u(X^*)
\end{align*}
Thus, $X^*$ is an optimal solution of the $\ASTC$ instance $\cal I$; since $X^*$ can be computed in poly-time (Proposition~\ref{prop:fstar-poly}), we are done.
\end{proof}

\section{The Price of Diversity}\label{sec:diversity}
We now turn to the allocative efficiency of the constrained assignment. 
As before, an instance of the \ASTC problem is given by a set of $n$ agents $N$ partitioned into types $N_1,\dots, N_k$, a set of $m$ items $M$ partitioned into $M_1,\dots, M_l$, a list of capacity values $(\lambda_{pq})_{k\times l}$, and agent utilities for items given by $u = (u(i,j))_{n\times m}$. We denote the set of all assignments $X$ of items to agents satisfying only the matching constraints~(\ref{eq:uniqueapt}-\ref{eq:integerallocation}) of Section~\ref{sec:prelim} by $\cal X$, and that of all assignments additionally satisfying the type-block constraints~\eqref{eq:capcityconstraint} by $\cal X_C$; the corresponding optimal social welfares for any given utility matrix $(u(i,j))_{n \times m}$ are: 
\begin{align*}
\OPT(u) \triangleq \max_{X \in \cal X} u(X); \,
\OPT_C(u) \triangleq \max_{X \in \cal X_C} u(X).
\end{align*} 
Clearly, $\OPT_C(u) \le \OPT(u)$ since $\cal X_C \subseteq \cal X$; we define the following natural measure of this welfare loss that lies in $[1,\infty]$:

\begin{definition}\label{def:pod}
For any instance of the \ASTC problem, we define the \emph{Price of Diversity} as follows, along the lines of \citeAY{ahmed2017diverse} and \citeAY{bredereck2018multiwinner}: \defend
$$\PoD(u) \triangleq \frac{\OPT(u)}{\OPT_C(u)}.$$ 
\end{definition}
The main result of this section is to establish an upper bound on $\PoD(u)$ that is independent of the utility model. 
Denote the ratio of a type-block capacity to the size of the corresponding block by:
 $$\alpha_{pq} \triangleq \frac{\lambda_{pq}}{|M_q|}.$$ 
\begin{theorem}\label{thm_generalpod} For any instance of \ASTC, we have:
$$\PoD(u) \le \frac{1}{\min_{(p,q) \in [k] \times [l]} \alpha_{pq}}$$
and the above upper bound is tight.
\end{theorem}

In general, the bound in Theorem~\ref{thm_generalpod} can grow linearly in $m$; in the following family of problem instances where the capacities $\lambda_{pq}$ are fixed constants, the $\PoD$ can be indefinitely large. %(e.g. if the capacities $\lambda_{pq}$ are fixed constants). 
\begin{example}\label{ex:podtight}
Consider any instance of the \ASTC problem with $l=k$ and $|N_p|=|M_p|=\mu$ $\forall p \in [k]$ so that $n=m=k\mu$; the utilities are:
\begin{align*}
u(i,j) = \begin{cases} 1 & \mbox{if } i \in N_{p} \mbox{ and } j \in M_{p} \quad \forall p \in [k],\\ 0 &\mbox{otherwise.}\end{cases}
\end{align*}
Evidently, any complete matching of items in $M_p$ to agents in $N_p$ $\forall p \in [k]$ is an optimal solution for the unconstrained version of the problem, hence $\OPT(u)=k\mu$. But if the capacities are $\lambda_{pq}=1$ $\forall (p,q) \in [k] \times [k]$ then only one agent per group can receive an item for which she has non-zero utility, hence $\OPT_C(u)=k$. Thus, $\PoD=k\mu/k=\mu$. \egend
\end{example}
However, type-block capacities are determined by a central planner in our model; a natural way of setting them is to fix the proportional capacities or quotas $\alpha_{pq}$ in advance, and then compute $\lambda_{pq} = \alpha_{pq} \times |M_q|$ when block sizes become available: by committing to a fixed minimum type-block quota $\alpha^*$ (i.e. $\alpha_{pq} \ge \alpha^*$ for all $(p,q) \in [k]\times [l]$), the planner can ensure a $\PoD(u)$ of at most $1/\alpha^*$, regardless of the problem size and utility function.
Higher values of $\alpha^*$ reduce the upper bound on $\PoD(u)$ but also increase the capacity of a block for every ethnicity, potentially affecting the diversity objective adversely: it thus functions as a tunable tradeoff parameter between ethnic integration and worst-case welfare loss.
In fact, in the Singapore allocation problem, the Ethnic Integration Policy fixes a universal percentage cap for each of the three ethnicities in all blocks; these percentages are set slightly higher than the actual respective population proportions: the current block quotas $\alpha_{pq}$ are $0.87$ for Chinese, $0.25$ for Malays and $0.15$ for Indian/Others \cite{deng2013publichousing}. Hence, for the  Singapore housing system, we have $\min_{(p,q) \in [k] \times [l]} \alpha_{pq}=0.15$ which is achieved for Indian/Others and any block, so that from Theorem~\ref{thm_generalpod},  
\[ \PoD(u) \leq \frac{1}{0.15} \approx 6.67. \]
This bound makes no assumptions on agent utilities; in other words, it holds under {\em any utility model}.\footnote{In practice, the \emph{effective} value of each fractional capacity $\lambda_{pq}/|M_q|$ might be smaller than the corresponding pre-specified fraction $\alpha_{pq}$. Since each $\lambda_{pq}$ must be an integer, we need to set $\lambda_{pq} = \lfloor \alpha_{pq} \times |M_q| \rfloor$ $\forall (p,q) \in [k] \times [l]$ to respect all capacity constraints. Hence, for a given instance of \ASTC, the effective upper bound on $\PoD(u)$ is given by $1/\min_{p,q}\frac{\lfloor \alpha_{pq} \times |M_q| \rfloor}{|M_q|}$, which depends on the $|M_q|$-values and may be higher than $1/\min_{p,q}\alpha_{pq}$. For example, if we have a uniform block size of $10$, the actual numerical capacities for Chinese, Malay, and Indian/Others based on EIP quotas become $8$, $2$, and $1$ respectively, so that the effective $\PoD$-bound is $10$. However, the effective bound is still independent of utility values as well as the population size of agents of any type; moreover, larger block sizes reduce the discrepancy between the effective bound and the theoretical bound $1/\min_{p,q}\alpha_{pq}$ provided by Theorem~\ref{thm_generalpod}. For example, for the values of $\alpha_{pq}$ and $|M_q|$ used in our experiments in Section~\ref{sec:experiments} (see Figure~\ref{figMap}), the minimum effective capacity of any block for any type is $14.42\%$, hence the effective upper bound on $\PoD(u)$ is $6.93$. Similar considerations apply to Theorem~\ref{thm_parampod}.}

The proof relies on the following lemma. Given an assignment $X\in \cal X$, let $u_p(X)$ denote the total utility of agents in $N_p$ under $X$:
\begin{align} 
u_p(X) \triangleq \sum_{i \in N_p} \sum_{j \in M} x_{ij} u(i,j) = \sum_{q \in [l]} \sum_{i \in N_p} \sum_{j \in M_q} x_{ij} u(i,j). \label{def:u_p} 
\end{align}
\begin{lemma}\label{lemma1}
For any instance of \ASTC and any optimal unconstrained assignment $X^* \in \cal X$, we have:
$$
\PoD(u) \leq \frac{u(X^*)}{\sum_{p \in [k]} u_p(X^*) \min_{q\in [l]} \alpha_{pq}}.
$$
\end{lemma}

\begin{proof}
Based on the optimal assignment $X^*$, we can construct an assignment $X \in \cal X_C$ satisfying the type-block constraints, by carefully `revoking' the smallest-utility items in $M_q$ from agents in $N_p$ for every $(p,q)$-pair that violates the corresponding type-block constraint. %By revoking from agents with the smallest utilities, we ensure that at least $\alpha_{pq}$ proportion of the utility remains under $X$ for $(p,q)$. 
In other words, let $n_{pq}$ denote the number of items in $M_q$ assigned to agents in $N_p$ under $X^*$. If $n_{pq} \le \lambda_{pq}$, we leave that type-block pair untouched, so that $\sum_{i\in N_p} \sum_{j \in M_q} x_{ij}u(i,j) = \sum_{i\in N_p} \sum_{j \in M_q} x^*_{ij}u(i,j)$. If $n_{pq} > \lambda_{pq}$, we order these $n_{pq}$ agents according to their utilities for the items they are assigned and retain only the top $\lambda_{pq}$ agents in that order (breaking ties lexicographically), setting $x_{ij}=0$ for the remaining agents. This change increases the average utility of assignments for this type-block pair:
\begin{align*} 
&\frac{\sum_{i\in N_p} \sum_{j \in M_q} x_{ij}u(i,j)}{\lambda_{pq}} \ge \frac{\sum_{i\in N_p} \sum_{j \in M_q} x^*_{ij}u(i,j)}{n_{pq}}.
\end{align*}
(To see why this is true, consider a sequence $z_1 \ge z_2 \ge \ldots$, where $z_i \ge 0$ and $z_i \ge z_{i+1}$ $\forall i =1,2,\ldots$, and two positive integers $\nu > \mu \ge 1$. Clearly, $(\nu-\mu)\sum_{i=1}^{\mu} z_i \ge (\nu-\mu)\mu z_{\mu} \ge \mu\sum_{i=\mu+1}^{\nu} z_i$ since $z_i \ge z_\mu$ $\forall i=1,2,\ldots,\mu$ and $z_i \le z_\mu$ $\forall i=\mu+1,\mu+2,\ldots,\nu$. Rearranging and simplifying, we get $\frac{1}{\mu}\sum_{i=1}^{\mu} z_i \ge \frac{1}{\nu}\sum_{i=1}^{\nu} z_i$.)

Further, since $n_{pq} \le |M_q|$, the above inequality implies that
\begin{align*} 
\frac{\sum_{i\in N_p} \sum_{j \in M_q} x_{ij}u(i,j)}{\lambda_{pq}} &\ge \frac{\sum_{i\in N_p} \sum_{j \in M_q} x^*_{ij}u(i,j)}{|M_q|}\\
\Longrightarrow \quad \sum_{i\in N_p} \sum_{j \in M_q} x_{ij}u(i,j) &\ge \frac{\lambda_{pq}}{|M_q|} \sum_{i\in N_p} \sum_{j \in M_q} x^*_{ij}u(i,j)\\
&= \alpha_{pq} \sum_{i\in N_p} \sum_{j \in M_q} x^*_{ij}u(i,j), \quad \text{since $\alpha_{pq}=\frac{\lambda_{pq}}{|M_q|}$}.
\end{align*}
Thus, for every $p \in [k]$ and every $q \in [l]$, we have
\[\sum_{i\in N_p} \sum_{j \in M_q} x_{ij}u(i,j) \ge \left( \min_{q \in [l]} \alpha_{pq} \right) \sum_{i\in N_p} \sum_{j \in M_q} x^*_{ij}u(i,j), \quad \text{since $\min_{q \in [l]} \alpha_{pq} \le \alpha_{pq} \le 1$.} \]
Summing over blocks, we obtain from Equation~\eqref{def:u_p}:
$$u_p(X) \geq u_p(X^*) \min_{q \in [l]} \alpha_{pq}, \, \forall p \in [k].$$
By definition, $u(X^*) = \OPT(u)$. Moreover, since $X \in \cal X_C$, we have $u(X) \le\OPT_C(u)$. Hence, by Definition~\ref{def:pod},
$$\PoD(u) \le \frac{u(X^*)}{u(X)}  \le \frac{u(X^*)}{\sum_{p \in [k]} u_p(X^*) \min_{q\in [l]} \alpha_{pq}}.$$
\end{proof}

We can now complete the proof of the theorem.
\begin{proof}[Proof of Theorem~\ref{thm_generalpod}] 
Since we have $\min_{(p,q) \in [k] \times [l]} \alpha_{pq} \le \min_{q \in [l]} \alpha_{p'q}$ for all $p' \in [k]$, Lemma \ref{lemma1} implies that:
\begin{align*}
\PoD(u) \le \frac{u(X^*)}{\sum\limits_{p\in[k]} u_p(X^*) \min\limits_{(p,q) \in [k] \times [l]} \alpha_{pq}} = \frac{1}{\min\limits_{(p,q)\in [k]\times [l]} \alpha_{pq}}.
\end{align*}
Depending on the utility matrix $u$, this upper bound can be tight whenever $|N_{p_0}| \ge |M_{q_0}|$ for some type-block pair $(p_0,q_0)$ in the set $\argmin_{(p,q)\in [k]\times [l]} \alpha_{pq}$. We identify an agent utility matrix for which the bound holds with equality:
\begin{align*}
u(i,j) = \begin{cases} 1 & \mbox{if } i \in N_{p_0} \mbox{ and } j \in M_{q_0},\\ 0 &\mbox{otherwise.}\end{cases}
\end{align*}
The optimal assignment without type-block constraints fully allocates the items in block $M_{q_0}$ to agents in $N_{p_0}$ for a total utility of $|M_{q_0}|$; furthermore, we know that any optimal constrained assignment allocates exactly $\lambda_{p_0 q_0}$ items in $M_{q_0}$ to agents in $N_{p_0}$ for a total utility of $\lambda_{p_0 q_0}$. Since $\lambda_{p_0 q_0}=\alpha_{p_0 q_0} \times |M_{q_0}|$, we have:
$$\PoD(u) = \frac{|M_{q_0}|}{\alpha_{p_0 q_0} \times |M_{q_0}|}= \frac{1}{\alpha_{p_0 q_0}} = \frac{1}{\min_{(p,q)\in [k]\times [l]} \alpha_{pq}}.$$
\end{proof}

\subsection{The Impact of Disparity among Types}

Theorem~\ref{thm_generalpod} offers a worst-case tight bound on the price of diversity, making no assumptions on agent utilities. However, its proof suggests that this upper bound is attained when social welfare is solely extracted from a single agent type and a single block. 
Intuitively, we can obtain a better bound on the price of diversity if a less `disparate' optimal assignment exists.
To formalize this notion, we introduce a new parameter:
\begin{definition}
For an optimal unconstrained assignment $X^*\! \in\! \cal X$, denote by $\beta_p(X^*)$ the ratio of the average utility of agents in $N_p$ to the average utility of all agents under $X^*$. 
The \emph{inter-type disparity parameter} $\beta(X^*)$ is defined as: \defend
$$\beta(X^*) \triangleq \min_{p \in [k]} \beta_p(X^*) = \min_{p \in [k]} \frac{u_p(X^*)/|N_p|}{u(X^*)/n}.$$
\end{definition}
Notice that $\beta(X^*) \in (0,1]$ can be computed in polynomial time and is fully independent of the type-block capacities. The closer $\beta(X^*)$ is to $1$, the lower the disparity between average agents of different types under $X^*$.
\begin{theorem}\label{thm_parampod}
For any \ASTC instance and any unconstrained optimal assignment $X^* \in \cal X$, we have:
\[ \PoD(u) \leq \frac{1/\beta(X^*)}{\sum_{p \in [k]} \nu_p \min_{q \in [l]} \alpha_{pq}},\]
where $\nu_p = \frac{|N_p|}{n}$ is the proportion of type $p$ in the agent population, for every $p \in [k]$.
%$$
%\PoD(u) \leq \frac{1}{\beta(X^*) \displaystyle \sum\limits_{p \in [k]}\frac{|N_p|}{n} \min\limits_{q \in [l]} \alpha_{pq}}.
%$$
\end{theorem}
\begin{proof}
By definition of $\beta(X^*)$, for every $p \in [k]$, we have: 
\begin{align*}
u_p(X^*) \ge \beta(X^*) \frac{|N_p|}{n} u(X^*) = \beta(X^*) \nu_p u(X^*).
\end{align*}
Substituting this in Lemma \ref{lemma1}, we obtain the desired bound.
\end{proof}

Let us now apply the result to the Singapore public housing domain; we use the ethnic proportions reported in the 2010 census report \cite{Sing2010} to obtain $|N_1|/n= 0.741$ (Chinese), $|N_2|/n=0.134$ (Malay), and $|N_3|/n= 0.125$ (Indian/Others). Using the same block quotas $\alpha_{pq}$ as before, we have:
$$
\PoD(u) \leq \frac{1/\beta(X^*)}{0.87 \times 0.741 + 0.25 \times 0.134 + 0.15 \times 0.125} \approx \frac{1.43}{\beta(X^*)}.
$$
In other words, if inter-type disparity is low, i.e. $\beta(X^*)$ is close to $1$, the $\PoD$ may be significantly lower than the bound provided by Theorem~\ref{thm_generalpod}.\footnote{If there are multiple optimal assignments, they may have different values of the inter-type disparity parameter; if that is the case, we should choose the largest of these values for computing the upper bound provided by Theorem~\ref{thm_parampod} since the inequality holds for \emph{any} unconstrained optimal assignment.}%e.g. suppose there are two types $N_1=\{a_1,a_2\}$ and $N_1=\{b_1\}$ and two items $M=\{1,2\}$ with utilities: $u(a_1,1)=0.75$, $u(a_1,2)=0.25$; $u(i,1)=0.25$ and $u(i,2)=0.75$ $\forall i \in \{a_2,b_1\}$. There are two optimal assignments: $X^*_1$ that assigns $1$ to $a_1$ and $2$ to $a_2$; $X^*_2$ that assigns $1$ to $a_1$ and $2$ to $b_1$. Here, $\beta$

The two upper bounds provided by Theorems~\ref{thm_generalpod} and \ref{thm_parampod} are incomparable due to the dependence of the latter on the parameter $\beta(X^*)$. We will elaborate on this point with the help of two examples (\ref{ex:bound1} and \ref{ex:bound2}) in each of which every agent has a positive utility for every item (unlike Example~\ref{ex:podtight}). 

In each example, we have only one block of items $M=[m]$ and two types $N_1=\{a_1,a_2,\ldots,a_m\}$ and $N_2=\{b_1,b_2,\ldots,b_m\}$ for an arbitrary positive integer $m > 1$. Hence the proportion of each type in the population is $\tfrac12$. Let the proportional capacities of the single block for the two types be $\alpha_1$ and $\alpha_2$ respectively, such that $\alpha_1+\alpha_2 \ge 1$ and  $\alpha_1 m$ and $\alpha_2 m$ are both integers. Thus, for each of these examples, Theorem~\ref{thm_generalpod} puts the upper bound on the price of diversity at $\frac{1}{\min\{\alpha_1,\alpha_2\}}$. In the first example, the bound based on the inter-type disparity parameter turns out to be useless/uninformative.
\begin{example}\label{ex:bound1}
	For some $\eps \ll 1-\frac{1}{m}$, let the utilities be
	\begin{align*}
	u(a_r,j) &=\begin{cases}
	1-\eps & \text{if $j = r$,}\\
	\frac{\eps}{m-1} & \text{otherwise}
	\end{cases} && \forall r \in [m];\\
	u(b_r,j) &= \frac{1}{m} && \forall j \in M, \forall r \in [m].
	\end{align*}
	Evidently, the unique unconstrained optimal assignment is to match item $j$ with agent $a_j$ for every $j \in [m]$. Hence, per-agent average utilities of $N_1$ and $N_2$ are $1-\eps$ and $0$ respectively, making the inter-type disparity parameter zero. Thus, Theorem~\ref{thm_parampod} does not place any finite upper bound on $\PoD(u)$.
	
	However, a little thought reveals that, in a constrained optimal allocation, $\alpha_1 m$ items are allocated to $N_1$ such that each of these items $j$ is assigned to the agent $a_j$, and the remaining $(1-\alpha_1)m$ items are allocated to $N_2$ and arbitrarily assigned to one agent each. Since $(1-\alpha_1)m \le \alpha_2 m$, all type-block capacities are satisfied. Thus,
	\[\PoD(u) = \frac{m(1-\eps)}{\alpha_1 m(1-\eps)+(1-\alpha_1)m\cdot \frac{1}{m}} = \frac{1}{\alpha_1 + \psi \cdot (1-\alpha_1)},\]
	where $\psi=\frac{1}{m(1-\eps)} \in (0,1)$.
	Since $\alpha_1 < 1$, the denominator exceeds $\alpha_1$, so that $\PoD(u) < \frac{1}{\alpha_1}$. If $\alpha_1 > \alpha_2$, then $\alpha_1 + \psi \cdot (1-\alpha_1) = \psi + (1-\psi)\alpha_1 > \psi + (1-\psi)\alpha_2 = \alpha_2 + \psi \cdot (1-\alpha_2) > \alpha_2$; hence $\PoD(u) < \frac{1}{\alpha_2}$. In any case, $\PoD(u) \le \max\{\frac{1}{\alpha_1},\frac{1}{\alpha_2}\} = \frac{1}{\min\{\alpha_1,\alpha_2\}}$, i.e. the realized price of diversity respects the Theorem~\ref{thm_generalpod} bound. \egend
\end{example}
In the next example, the parametrized bound is more informative than the other.
\begin{example}\label{ex:bound2}
	We will further assume that $m$ is even. Let $u(a_r,j)=u(b_r,j)=\frac1m$ for every $r \in [m]$, $j \in [m]$. There is an unconstrained optimal assignment which also achieves $\beta(X^*)=1$: assign $\frac{m}{2}$ items arbitrarily to $\frac{m}{2}$ agents of each type, giving $\OPT(u)=1$. Thus, the Theorem~\ref{thm_parampod} bound $\frac{1}{(\alpha_1+\alpha_2)/2} \le \frac{1}{\min\{\alpha_1,\alpha_2\}}$, the Theorem~\ref{thm_generalpod} bound. Since $\alpha_1+\alpha_2 \ge 1$, there are at least two assignments, respecting capacity constraints, whose social welfare is $\OPT(u)=1$: either $\alpha_1 m$ to $N_1$ and $(1-\alpha_1)m$ to $N_2$, or $(1-\alpha_2)m$ to $N_1$ and $\alpha_2 m$ to $N_2$. In either case, the constrained optimum is $1$; hence $\PoD(u)$ has its ideal value of $1$. \egend
\end{example}	

Finally, combining Theorems~\ref{thm_generalpod} and \ref{thm_parampod}, we obtain the following upper bound on the price of diversity of any instance of \ASTC:
\begin{align} 
\PoD(u) \le \min\left\{\frac{1}{\min_{(p,q) \in [k] \times [l]} \alpha_{pq}}, \frac{1/\beta(X^*)}{\sum_{p \in [k]} \nu_p \min_{q \in [l]} \alpha_{pq}}\right\}. \label{comb_bound}
\end{align}
Thus, if we plot the $\PoD(u)$ against the disparity parameter $\beta(X^*)$, the point corresponding to any \ASTC instance with block quotas and ethnic proportions as in Singapore must lie in the shaded region of Figure~\ref{fig:PoDbeta}.

\begin{figure}[!h]
\centering
\begin{tikzpicture}[scale=0.45] \scriptsize
\coordinate (A) at (0,0);
\coordinate (B) at (11,0);
\coordinate (C) at (0,8);

\draw[thick,-latex] (A) -- (B);
\draw (11,-0.1) node[thick,right]{$\beta(X^*)$};
\draw[thick,-latex] (A) -- (C);
\draw (-0.1,8) node[thick,left]{$\PoD(u) $};

\draw (0,0) node[below,left]{$0$};
\draw (10,0) node{$|$};
\draw (10,-0.1) node[below]{$1$};
\draw (0,6.67) node{$-$};
\draw (0,6.67) node[left]{$6.67$};
\draw (0,1) node{$-$};
\draw (0,1) node[left]{$1$};

\draw  (0.3*10,1.43/0.3) node[right]{$1.43/\beta(X^*)$};

\draw[dashed,thick] (0,6.67) -- (10,6.67);
\draw (0,1) -- (10,1);

\draw[thick]  (0.2*10,1.43/0.2) -- (0.22*10,1.43/0.22)-- (0.25*10,1.43/0.25) -- (0.27*10,1.43/0.27)-- (0.3*10,1.43/0.3) -- (0.32*10,1.43/0.32)--  (0.35*10,1.43/0.35) -- (0.37*10,1.43/0.37) -- (0.4*10,1.43/0.4) -- (0.45*10,1.43/0.45) -- (0.5*10,1.43/0.5)-- (0.55*10,1.43/0.55) -- (0.6*10,1.43/0.6) -- (0.65*10,1.43/0.65) -- (0.7*10,1.43/0.7) -- (0.75*10,1.43/0.75) -- (0.8*10,1.43/0.8) -- (0.85*10,1.43/0.85) -- (0.9*10,1.43/0.9)-- (0.95*10,1.43/0.95) -- (1*10,1.43/1);

\fill[color=blue!80!black, opacity=0.3]  (1.43/6.67*10,6.67) -- (0.25*10,1.43/0.25) -- (0.27*10,1.43/0.27)-- (0.3*10,1.43/0.3) -- (0.32*10,1.43/0.32)-- (0.35*10,1.43/0.35) -- (0.4*10,1.43/0.4) -- (0.45*10,1.43/0.45) -- (0.5*10,1.43/0.5)-- (0.55*10,1.43/0.55) -- (0.6*10,1.43/0.6) -- (0.65*10,1.43/0.65) -- (0.7*10,1.43/0.7) -- (0.75*10,1.43/0.75) -- (0.8*10,1.43/0.8) -- (0.85*10,1.43/0.85) -- (0.9*10,1.43/0.9)-- (0.95*10,1.43/0.95) -- (1*10,1.43/1) -- (10,1) -- (0,1)-- (0,6.67) -- cycle;
\end{tikzpicture}
\caption{$\PoD$ vs disparity parameter for the HDB problem for ethnic proportions $|N_1|/n= 0.741$ (Chinese), $|N_2|/n=0.134$ (Malay), and $|N_3|/n= 0.125$ (Indian/Others), and corresponding quotas $\alpha_{1q} = 0.87, \alpha_{2q} = 0.25$ and $\alpha_{3q} = 0.15$ for every block $M_q$.\label{fig:PoDbeta}}
\end{figure}
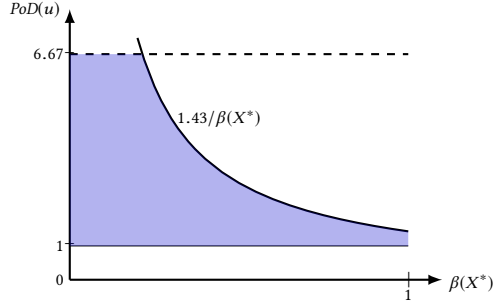

\section{Experimental Analysis}\label{sec:experiments}
In this section, we simulate instances of the \ASTC problem using recent, publicly available Singaporean demographic and housing allocation statistics and Chicago public school admission data. We compare the welfare of three assignment mechanisms: the optimal unconstrained mechanism, the optimal constrained mechanism, and the lottery-based mechanism (see Section \ref{sec:lottery} below). Both the unconstrained and constrained social welfare maximizations are solved using the Gurobi Optimizer.\footnote{Each iteration of the constrained optimization problem took about $30$ minutes while that of the unconstrained optimization and the lottery mechanism took less than $1$ minute of running time on a standard laptop (Intel i7-7600U Processor, 2.8Ghz, 8GB RAM, 256 GB SSD).} We refer the reader to \url{https://git.io/fNhhm} for full implementation details. %\url{https://github.com/DataDrivenStrategicCollaborationGroup/AssignmentProblemWithDiversityConstraints}

\subsection{The Lottery Mechanism}\label{sec:lottery}
Sections~\ref{sec:prelim} and \ref{sec:complexity} study an {\em optimal} mechanism for assigning goods to agents under diversity constraints. To the best of our knowledge, this mechanism is not used for allocating goods in practice; rather, lotteries are used to allocate items in both real-world instances that inspire this work. The mechanisms randomly order agents, and let each agent pick their favorite item in turn, {\em while respecting predetermined quotas}. 
In this section, we formulate a simple one-shot lottery-based mechanism that captures the aspect of the problems described in Sections~\ref{sec:hdb} and \ref{sec:chicago} that we are most interested in: the impact of type-block constraints (as defined by inequalities~\eqref{eq:capcityconstraint}) on assignment by lottery. Algorithm~\ref{lottery} is not the actual mechanism used in Singapore public housing or Chicago school choice (see the discussion in the respective sections). It is easy to see that the algorithm takes $\mathrm{poly}(mn)$ time to run. 
%In this section, we analyze the quality of the assignments returned by the simplified HDB lottery mechanism, which is implemented as shown in Algorithm~\ref{lottery} and takes $\mathrm{poly}(mn)$ time to run. Note that this algorithm can also be used for the unconstrained assignment problem since any instance of this problem can be seen as an instance of $\ASTC$ with only one agent type, one block of items and an unlimited type-block capacity.
\begin{algorithm}[h]
	\DontPrintSemicolon
	\caption{Lottery Mechanism for Assignment with Type-Block Constraints}\label{lottery}
	\KwIn{Agents $N$ grouped into types $N_1,N_2,\ldots,N_k$; items $M$ split into blocks $M_1,M_2,\ldots,M_l$; type-block capacities $\lambda_{pq}$ $\forall (p,q) \in [k] \times [l]$; utility matrix $(u(i,j))_{n \times m}$.}
	\KwInit{Allocation matrix $X = (x_{ij})_{n\times m} \leftarrow (0)_{n \times m}$; remaining agents $N_{\mathrm{rem}} \leftarrow N$; unassigned items $M_{\mathrm{rem}} \leftarrow M$.}
	\For{$t \in \{1,2,\ldots, n\}$}{
		Draw agent uniformly at random: $i_t \sim \mathbb{U} (N_{\mathrm{rem}})$.\\
		Find type of $i_t$: $p_t \leftarrow p \in [k]$ s.t. $i_t \in N_p$.\\
		Find blocks that have not hit capacity for type $p_t$: $Q_t \leftarrow \{ q \in [l] : \sum_{i \in N_{p_t}} \sum_{j \in M_q} x_{i j} < \lambda_{p_t q}\}$.\\
		Find available items: $M_t  \leftarrow M_{\mathrm{rem}} \cap \left( \cup_{q \in Q_t} M_q\right)$.\\
		\If{$M_t \neq \emptyset$}{
			Assign to $i_t$ available item for which she has highest utility, breaking ties lexicographically:\\
			$j_t \leftarrow \arg \max_{j \in M_t} u(i_t,j)$.\\
			$x_{i_t j_t} \leftarrow 1$.\\
			$M_{\mathrm{rem}} \leftarrow M_{\mathrm{rem}} \backslash \{j_t\}$.
		}
		$N_{\mathrm{rem}} \leftarrow N_{\mathrm{rem}} \backslash \{i_t\}$.
	}
	\Return $X$.
\end{algorithm}

We know that this lottery mechanism with quotas cannot produce better welfare than the optimal constrained mechanism (\ASTC); one of the objectives of our experiments is to find out how much worse the lottery performs for various utility models. We define the \textit{price of} (one instance of) \textit{the diverse lottery} as follows. %the \emph{relative loss} of (one instance of) the lottery mechanism as the ratio of $\OPT(u)$ to the  total utility of the assignment produced by a run of Algorithm~\ref{lottery}.
\begin{definition}\label{def:podl}
	Let $a$ denote an arbitary run of Algorithm~\ref{lottery} and $X_a$ the unique matching of items to agents induced by the run $a$. Then, the \emph{price of the diverse lottery instance $a$} under utilities $u = (u(i,j))_{n\times m}$ is given by: \defend
	$$\PoDL(u,a)  \triangleq \frac{\OPT(u)}{u(X_a)}.$$
\end{definition}

  %We have suppressed the dependence of this ratio on the type-block capacities $(\lambda_{pq})_{k\times l}$.  
  We will estimate and report the expected value of $\PoDL(u,a)$ (the expectation being over all possible permutations of agents induced by the uniform random sampling without replacement in Algorithm~\ref{lottery})  alongside the realized value of $\PoD(u)$ (the corresponding performance measure for the optimal constrained mechanism) for the same set of parameter values (agent-item utilities and type-block capacities): see Sections~\ref{sec:HDBexpts} and~\ref{sec:chicago_expts}. Note that the realized $\PoD(u)$ is a \emph{lower bound} on $\PoDL(u,a)$ for any run $a$ of Algorithm~\ref{lottery} for the same problem instance; but the upper bounds from Theorems~\ref{thm_generalpod} and~\ref{thm_parampod} do not apply to $\PoDL(u,a)$ at all.

\subsection{The Singapore Public Housing Allocation Problem}\label{sec:HDBexpts}

\paragraph{Data Collection}
In order to create realistic instances of the \ASTC problem within the Singaporean context, we collected data on the location and number of flats of recent HDB housing development projects advertised over the second and third quarters of 2017.\footnote{\url{http://www.hdb.gov.sg/cs/infoweb/residential/buying-a-flat/new/bto-sbf}} Each of these developments corresponds to a block in our setup, for a total of $m = 1350$ flats partitioned into $l = 9$ blocks (a detailed map is given in Figure \ref{figMap}).
Moreover, each flat in any of these blocks belongs to one of several pre-specified categories, viz. 2-room flexi, 3-room, 4-room, and 5-room; our data set includes lower and upper bounds, $\LB(t,q)$ and $\UB(t,q)$ respectively, on the monthly cost (loan) for a flat of category $t$ in block $M_q$ for every $t$ and $q$.
 We consider two applicant pools whose ethnic composition follows the 2010 Singapore census report \cite{Sing2010}: there are $n = m = 1350$ applicants in the first pool with $|N_1| = 1000$ ($\approx74.1\%$ Chinese), $|N_2| = 180$ ($\approx 13.4\%$ Malay), and $|N_3| = 170$ ($\approx 12.5\%$ Indian/Others); the second pool has $n=3000$ applicants with $|N_1| = 2223$, $|N_2| = 402$, and $|N_3| = 375$. From the 2010 Singapore census report, we also collected the average salary $S(p)$ of each ethnicity group $p \in [k]$, given in Singapore dollars: $S(1) = 7,326$, $S(2)=4,575$ and $S(3) = 7,664$.\footnote{We found no public data on applicant pools for public housing allocation in Singapore. We wanted to test the performance of the constrained optimization approach (mainly) for representative values of $n$ in the interesting domain $n \ge m$. We chose $n = m$ (for which it is possible to achieve a perfect matching) and $n = \lceil 2m \times 10^{-3}\rceil \times 10^3$; we observed a surprisingly small difference in the realized $\PoD(u)$ for these two values of $n$ and did not repeat our expensive experiments for higher values of $n$. Moreover, in both our motivating real-world problems, it is unlikely that the number of agents is orders of magnitude higher than the number of items. Similar reasoning applies to our choices of $n$ in Section~\ref{sec:chicago_expts}.}
 From publicly available data\footnote{\url{https://data.gov.sg/dataset/master-plan-2014-planning-area-boundary-web}} on Singapore's Master Plan 2014,\footnote{\url{https://www.ura.gov.sg/Corporate/Planning/Master-Plan/}} we collected the locations of the geographic centers of the $55$ planning areas that Singapore is divided into; we also obtained the population sizes of the three ethnicity groups under consideration in each planning area from the General Household Survey 2015 data available from the Department of Statistics, Singapore.\footnote{\url{https://www.singstat.gov.sg/publications/ghs/ghs2015content} > Statistical Tables > Basic Demographic Characteristics}
 Finally, we use a uniform block capacity using the latest HDB block quotas~\cite{deng2013publichousing}: for every block $M_q$, we have $\alpha_{1q} = 0.87, \alpha_{2q} = 0.25$ and $\alpha_{3q} = 0.15$. 

\begin{figure}[!t]
\centering
\begin{subfigure}{.6\columnwidth}
\centering
\includegraphics[scale=0.1]{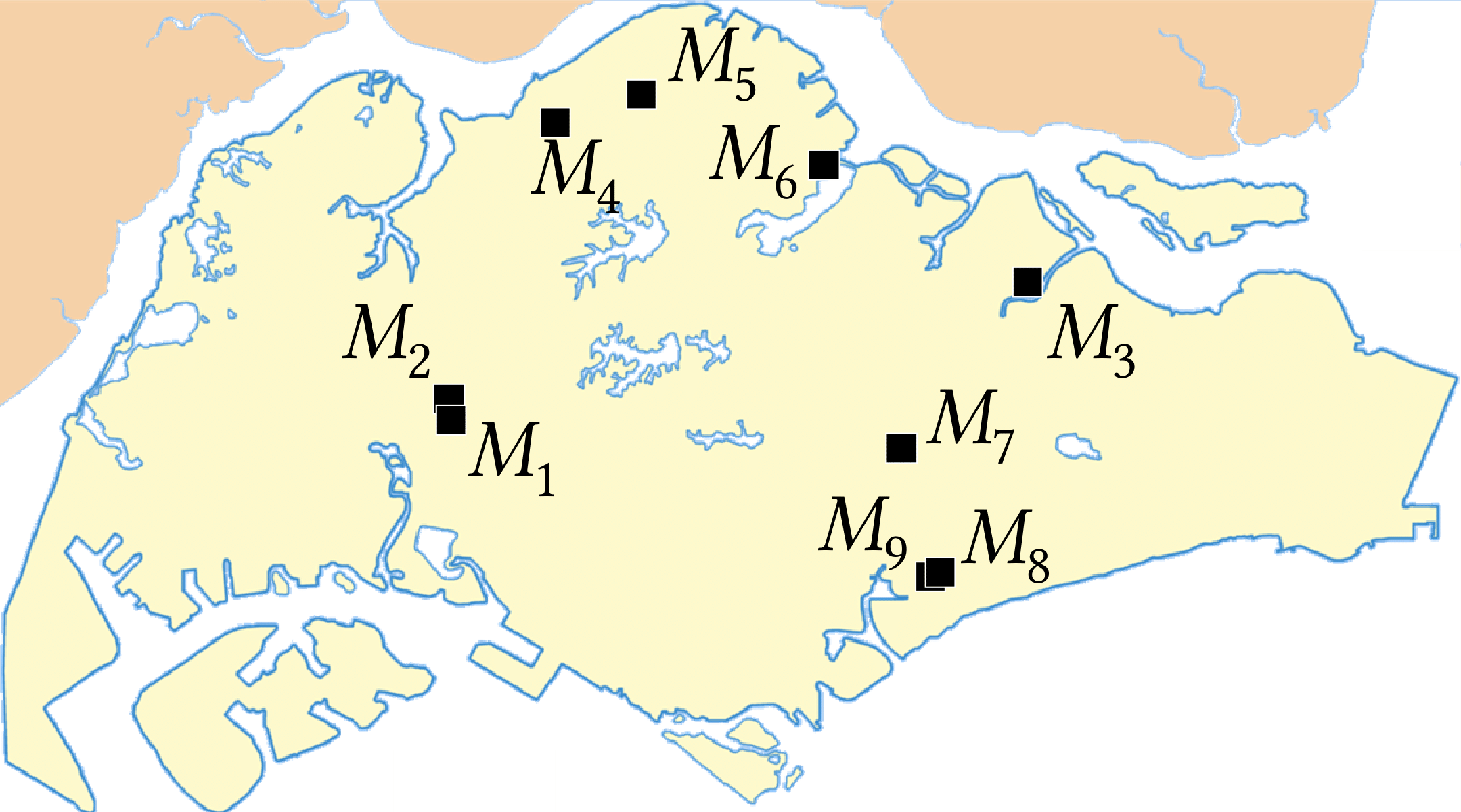}%scale=0.179
\end{subfigure}%
\begin{subfigure}{.4\columnwidth}

\begin{itemize} 
\item $M_1$: Sky Vista (128 flats)
\item $M_2$: West Scape (162 flats)
\item $M_3$: Rivervale Shores (156 flats)
\item $M_4$: Marsiling Grove (249 flats)
\item $M_5$: Woodlands Spring (108 flats)
\item $M_6$: Forest Spring (94 flats)
\item $M_7$: Woodleigh Hillside (104 flats)
\item $M_8$: Dakota Breeze (190 flats)
\item $M_9$: Pine Vista (159 flats)

\end{itemize}

\end{subfigure}%
\caption{Block locations and number of flats.\label{figMap}}
\end{figure}

\paragraph{Utility Models}

All parameters used to generate \ASTC instances in our simulations are based on real data, except for agent utilities over apartments. Conducting large-scale surveys that elicit user preferences over apartments is beyond the scope of this work; thus, we base our agent utility models on simulated utilities. We examine four utility models, each characterized by a parameter whose value does not come from the data: {\em distance-based} ($\Dist(\sigma^2)$), {\em type-based} ($\Ethn(\sigma^2)$), {\em project approval-based} ($\Proj(\rho)$), and {\em price-based} ($\Price(\sigma^2)$). 
\begin{itemize}
\item In the distance-based utility model, each agent $i \in N$ has a preferred geographic location $\vec a_i \in \R^2$ (chosen uniformly at random within the physical landmass of Singapore)\footnote{More precisely, we considered the largest rectangle that fits within the main island of Singapore, constructed a two-dimensional Cartesian coordinate system with one degree of latitude/longitude (each roughly equal to 111 km at the location of Singapore), and then picked x- and y-coordinates independently and uniformly within this rectangle.} that she would like to live as close as possible to (say, the location of her parents' apartment, workplace, or preferred school). For every block $M_q$, we generate the utility that agent $i$ derives from apartment $j \in M_q$ by first drawing a sample from the normal distribution $\cal N(1/ d(\vec a_i,\loc(M_q)),\sigma^2)$, where $\loc(M_q) \in \R^2$ is the geographical location of block $M_q$ and $d(\cdot,\cdot)$ represents Euclidean distance, and then renormalizing to make the sum of utilities of each agent for all apartments in $M$ equal to one. 
\item In the type-based utility model, we assume that all agents of the same type (i.e. ethnic group) have the same preferred location (i.e.  $\forall p \in [k], \forall i,i' \in N_p, \vec a_i = \vec a_{i'}$). The rest of the model description follows the above distance-based model.
\item In the project approval-based utility model, we construct, for each type, a categorical distribution over the $55$ planning areas of Singapore, the probability of each area being proportional to the fraction of the sub-population of that type living in that area; for each agent $i$, we sample a preferred planning area from the above distribution corresponding to $i$'s type; if a project $M_q$ is within a radius $\rho$ of the geographic center of agent $i$'s preferred planning area, then $i$ approves of the project, i.e. $u(i,j) = 1$ $\forall j \in M_q$, else $i$ disapproves of the project, i.e. $u(i,j) = 0$ $\forall j \in M_q$.\footnote{This is a  specific instance of the \emph{approval model} (or \emph{dichotomous preferences} over items) where an agent either wants or does not want an item but does not distinguish among the items she wants. This also corresponds to an unweighted bipartite matching setting (between agents and items) where there is an edge between an agent and item if and only if the agent approves/wants the item. In such a situation, the utilities of an agent are usually not normalized over items (see, e.g. \cite{bogomolnaia2004random}).}
\item In the price-based utility model, each agent $i \in N_p$ has a salary $s_i$ that is generated according to the normal distribution $\cal N(S(p),\sigma^2)$. Each flat $j \in M_q$ of category $t$ has a monthly cost $p_j$ that is chosen uniformly in $[LB(t,q), UB(t,q)]$. We assume that  agent $i$ is willing to pay one-third of her monthly salary on mortgage installments.\footnote{The choice of the one-third fraction is inspired by the ``3-3-5 rule" for deciding whether one can afford a flat given one's income (\url{https://www.areyouready.sg/YourInfoHub/Pages/News-How-to-use-the-3-3-5-rule-to-consider-if-you-can-afford-your-new-home.aspx}), endorsed by the Central Provident Fund Board of Singapore (\url{https://www.cpf.gov.sg/members}).}
The utility that agent $i$ derives from flat $j$ is then defined by $$u(i,j)= \dfrac{1/(p_j - \frac{s_i}{3})^2}{\widehat{U}_i},$$ where $\widehat{U}_i \triangleq \sum_{z \in M} 1/(p_z - \frac{s_i}{3})^2$ is the normalization factor. The rationale for the utility formula is that a much higher cost relative to the budget makes the flat unaffordable while a much lower cost indicates unsatisfactory quality, making the agent unhappy in both scenarios. 
\end{itemize}

\paragraph{Evaluation}
For each of our treatments (Figures~\ref{figTests1}-\ref{figTests3}); we report the price of diversity $\PoD(u)$ as per Definition~\ref{def:pod} (hatched bar); the theoretical upper bound on $\PoD(u)$ as per Theorem~\ref{thm_parampod} (dark gray bar); and the price of the diverse lottery $\PoDL(u,\cdot)$ as per Definition~\ref{def:podl}, averaged over $100$ agent permutations (light gray bar).\footnote{The error bars in Figures~\ref{figTests1}-\ref{figTests3} and~\ref{figTests4} represent one (estimated) standard error of the mean on either side; see e.g. \citet{clymo2019many} for recommendations on how many significant digits of the mean to report based on its standard error, which we have tried to follow.}

First, we want to compare the distance-based utility model $\Dist(\sigma^2)$ and the type-based model $\Ethn(\sigma^2)$ in order to estimate the welfare loss due to imposing ethnicity constraints. To do so, we vary both $\sigma^2$ in $\{1,5,10\}$ and $n$ in $\{1350, 3000\}$; the results reported in Figures~\ref{figTests1} are on average performance over $100$ randomly generated instances. Our first observation is that, in all our experiments, the $\Dist(\sigma^2)$ exhibits virtually no reduction in welfare due to the imposition of type-block constraints (see the hatched bars in the charts on the left). This is because utilities in $\Dist(\sigma^2)$ are independent of ethnicities, resulting in a very low value for the inter-type disparity parameter (see the dark gray bars) --- in fact, for any utility model where utilities are independent of ethnicities, the value of the disparity parameter should intuitively be close to $1$ with a high probability. 
For utilities generated based on the $\Ethn(\sigma^2)$ model, the disparity parameter is somewhat higher (utilities do strongly depend on ethnicities), resulting in a higher $\PoD(u)$. 
Despite making no attempt to optimize social welfare under type-block constraints, the HDB lottery mechanism does surprisingly well when the number of agents equals the number of apartments (Figure~\ref{figTests1} (a)), extracting at least $84\%$ of the optimal unconstrained welfare under the $\Dist(\sigma^2)$ utility model, and at least $79\%$ of the social welfare under the $\Ethn(\sigma^2)$ model. However, the welfare loss induced by the lottery mechanism is negatively impacted by the number of agents (Figure~\ref{figTests1} (b)); for instance, it only extracts $65\%$ of the optimal unconstrained welfare under $\Dist(1)$ with $n=3000$ and, in fact, the lottery-induced welfare loss for this treatment even exceeds the theoretical upper bound on the price of diversity.

\begin{figure}[h]
	\centering
	\begin{subfigure}{.5\columnwidth}%Dist 1350 agents
		\begin{tikzpicture}[scale=1] \normalsize
		\coordinate (A) at (-1/2,0);
		\coordinate (B) at (12.7/2,0);
		\coordinate (C) at (-1/2,2.6);
		
		\draw[thick,-latex] (A) -- (B);
		\draw (12.5/2,-0.02) node[below]{$\sigma^2$};
		\draw[thick,-latex] (A) -- (C);
		
		\draw (-1/2,1) node{$\_$};
		\draw (-1.5/2,1) node{$1$};
		\draw (-1.5/2,2.4) node{\rotatebox{90}{welfare ratio}};
		
		\draw (1.1/2,0) node{$|$};
		\draw (1.1/2,-0.1) node[below]{$1$};
		\draw (5.6/2,0) node{$|$};
		\draw (5.6/2,-0.1) node[below]{$5$};
		\draw (10/2,0) node{$|$};
		\draw (10/2,-0.1) node[below]{$10$};
		
		%random var 1
		\draw (-1/2,0) -- (-1/2,1.00) -- (0.4/2,1.00) -- (0.4/2,0) --cycle;
		\filldraw[pattern=north east lines] (-1/2,0) -- (-1/2,1.00) -- (0.4/2,1.00) -- (0.4/2,0) --cycle;
		\draw (-0.3/2,0) node[above]{\hlw{$1.00$}};
		\draw (-0.3/2,1.00-0.0002) -- (-0.3/2,1.00+0.0002);
		\draw (-0.8/2, 1.00-0.0002) -- (0.2/2, 1.00-0.0002);
		\draw (-0.8/2, 1.00+0.0002) -- (0.2/2, 1.00+0.0002);
		
		\draw (0.4/2,0) rectangle (1.8/2,1.1807711);
		\filldraw[color=gray!80!black, opacity=0.3]  (0.4/2,0) rectangle (1.8/2,1.1807711);
		\draw (1.1/2,1/4) node[above]{\hlw{$1.18$}};
		\draw (1.1/2,1.1807711-0.011465986657428689) -- (1.1/2,1.1807711+0.011465986657428689);
		\draw (0.6/2, 1.1807711-0.011465986657428689) -- (1.6/2, 1.1807711-0.011465986657428689);
		\draw (0.6/2, 1.1807711+0.011465986657428689) -- (1.6/2, 1.1807711+0.011465986657428689);
		
		\draw (1.8/2,0) -- (1.8/2,1.4877853) -- (3.2/2,1.4877853) -- (3.2/2,0)  --cycle;
		\filldraw[color=gray!80!black, opacity=0.9] (1.8/2,0) -- (1.8/2,1.4877853) -- (3.2/2,1.4877853) -- (3.2/2,0)  --cycle;
		\draw (2.5/2,2/4) node[above]{\hlw{$1.49$}};
		\draw (2.5/2,1.4877853-0.036494393754673574) -- (2.5/2,1.4877853+0.036494393754673574);
		\draw (2/2, 1.4877853-0.036494393754673574) -- (3/2, 1.4877853-0.036494393754673574);
		\draw (2/2, 1.4877853+0.036494393754673574) -- (3/2, 1.4877853+0.036494393754673574);

		%random var 5
		\draw (3.5/2,0) -- (3.5/2,1.00) -- (4.9/2,1.00) -- (4.9/2,0) --cycle;
		\filldraw[pattern=north east lines] (3.5/2,0) -- (3.5/2,1.00) -- (4.9/2,1.00) -- (4.9/2,0) --cycle;
		\draw (4.2/2,0) node[above]{\hlw{$1.00$}};
		\draw (4.2/2,1.00-0.00009) -- (4.2/2,1.00+0.00009);
		\draw (3.7/2, 1.00-0.00009) -- (4.7/2, 1.0-0.00009);
		\draw (3.7/2, 1.00+0.00009) -- (4.7/2, 1.00+0.00009);
		
		\draw (4.9/2,0) rectangle  (6.3/2,1.125433);
		\filldraw[color=gray!80!black, opacity=0.3] (4.9/2,0) rectangle  (6.3/2,1.125433);
		%\draw[pattern=horizontal lines, pattern color=black] (4.9/2,0) -- (4.9/2,1.125433) -- (6.3/2,1.125433) -- (6.3/2,0)--cycle;
		\draw (5.6/2,1/4) node[above]{\hlw{$1.13$}};
		\draw (5.6/2,1.125433-0.007479383910016952) -- (5.6/2,1.125433+0.007479383910016952);
		\draw (5.1/2, 1.125433-0.007479383910016952) -- (6.1/2, 1.125433-0.007479383910016952);
		\draw (5.1/2, 1.125433+0.007479383910016952) -- (6.1/2, 1.125433+0.007479383910016952);
		
		\draw (6.3/2,0) -- (6.3/2,1.4750199) -- (7.7/2,1.4750199) -- (7.7/2,0)  --cycle;
		\filldraw[color=gray!80!black, opacity=0.9] (6.3/2,0) -- (6.3/2,1.4750199) -- (7.7/2,1.4750199) -- (7.7/2,0)  --cycle;
		\draw (7/2,2/4) node[above]{\hlw{$1.48$}};
		\draw (7/2,1.4750199-0.027128398973331067) -- (7/2,1.4750199+0.027128398973331067);
		\draw (6.5/2, 1.4750199-0.027128398973331067) -- (7.5/2, 1.4750199-0.027128398973331067);
		\draw (6.5/2, 1.4750199+0.027128398973331067) -- (7.5/2, 1.4750199+0.027128398973331067);

		%random var 10 - constraints
		\draw (8/2,0) -- (8/2,1.00) -- (9.4/2,1.00) -- (9.4/2,0) --cycle;
		\filldraw[pattern=north east lines] (8/2,0) -- (8/2,1.00) -- (9.4/2,1.00) -- (9.4/2,0) --cycle;
		\draw (8.7/2,0) node[above]{\hlw{$1.00$}};
		\draw (8.7/2,1.00-0.000059) -- (8.7/2,1.00+0.000059);
		\draw (8.2/2, 1.00-0.000059) -- (9.2/2, 1.00-0.0000599);
		\draw (8.2/2, 1.000+0.000059) -- (9.2/2, 1.00+0.000059);
		
		%random var 10 - lottery
		\draw (9.4/2,0) -- (9.4/2,1.0925231) -- (10.8/2,1.0925231) -- (10.8/2,0)--cycle;
		\filldraw[color=gray!80!black, opacity=0.3] (9.4/2,0) -- (9.4/2,1.0925231) -- (10.8/2,1.0925231) -- (10.8/2,0)--cycle;
		\draw (10.1/2,1/4) node[above]{\hlw{$1.09$}};
		\draw (10.1/2,1.0925231-0.006144323297759245) -- (10.1/2,1.0925231+0.006144323297759245);
		\draw (9.6/2, 1.0925231-0.006144323297759245) -- (10.6/2, 1.0925231-0.006144323297759245);
		\draw (9.6/2, 1.0925231+0.006144323297759245) -- (10.6/2, 1.0925231+0.006144323297759245);
		
		%random var 10 - bound
		\draw (10.8/2,0) -- (10.8/2,1.4693056) -- (12.2/2,1.4693056) -- (12.2/2,0) --cycle;
		\filldraw[color=gray!80!black, opacity=0.9] (10.8/2,0) -- (10.8/2,1.4693056) -- (12.2/2,1.4693056) -- (12.2/2,0) --cycle;
		\draw (11.5/2,2/4) node[above]{\hlw{$1.47$}};
		\draw (11.5/2,1.4693056-0.017873566077744946) -- (11.5/2,1.4693056+0.017873566077744946);
		\draw (11/2, 1.4693056-0.017873566077744946) -- (12/2, 1.4693056-0.017873566077744946);
		\draw (11/2, 1.4693056+0.017873566077744946) -- (12/2, 1.4693056+0.017873566077744946);
		
		\end{tikzpicture}
		%  \caption{$\Dist(\sigma^2)$}
		%\label{fig:1-sub1}
	\end{subfigure}%
		\begin{subfigure}{.5\columnwidth}%Dist 3000 agents
		\begin{tikzpicture}[scale=1] \normalsize
		\coordinate (A) at (-1/2,0);
		\coordinate (B) at (12.7/2,0);
		\coordinate (C) at (-1/2,2.6);
		
		\draw[thick,-latex] (A) -- (B);
		\draw (12.5/2,-0.02) node[below]{$\sigma^2$};
		\draw[thick,-latex] (A) -- (C);
		
		\draw (-1/2,1) node{$\_$};
		\draw (-1.5/2,1) node{$1$};
		\draw (-1.5/2,2.4) node{\rotatebox{90}{welfare ratio}};
		
		\draw (1.1/2,0) node{$|$};
		\draw (1.1/2,-0.1) node[below]{$1$};
		\draw (5.6/2,0) node{$|$};
		\draw (5.6/2,-0.1) node[below]{$5$};
		\draw (10/2,0) node{$|$};
		\draw (10/2,-0.1) node[below]{$10$};
		
		%random var 1
		\draw (-1/2,0) -- (-1/2,1.00) -- (0.4/2,1.00) -- (0.4/2,0) --cycle;
		\filldraw[pattern=north east lines] (-1/2,0) -- (-1/2,1.00) -- (0.4/2,1.00) -- (0.4/2,0) --cycle;
		\draw (-0.3/2,0) node[above]{\hlw{$1.00$}};
		\draw (-0.3/2, 1.00-0.0000234243827763512) -- (-0.3/2,1.00+0.0000234243827763512);
		\draw (-0.8/2, 1.00-0.0000234243827763512) -- (0.2/2, 1.00-0.0000234243827763512);
		\draw (-0.8/2, 1.00+0.000023424382776351) -- (0.2/2, 1.00+0.000023424382776351);
		
		\draw (0.4/2,0) -- (0.4/2,1.5485175) -- (1.8/2,1.5485175) -- (1.8/2,0)--cycle;
		\filldraw[color=gray!80!black, opacity=0.3] (0.4/2,0) -- (0.4/2,1.5485175) -- (1.8/2,1.5485175) -- (1.8/2,0)--cycle;
		\draw (1.1/2,1/4) node[above]{\hlw{$1.55$}};
		\draw (1.1/2, 1.5485175-0.001363332721) -- (1.1/2, 1.5485175+0.001363332721);
		\draw (0.6/2, 1.5485175-0.001363332721) -- (1.6/2, 1.5485175-0.001363332721);
		\draw (0.6/2, 1.5485175+0.001363332721) -- (1.6/2, 1.5485175+0.001363332721);
		
		\draw  (1.8/2,0) -- (1.8/2,1.5139613) -- (3.2/2,1.5139613) -- (3.2/2,0)  --cycle;
		\filldraw[color=gray!80!black, opacity=0.9] (1.8/2,0) -- (1.8/2,1.5139613) -- (3.2/2,1.5139613) -- (3.2/2,0)  --cycle;
		\draw (2.5/2,2/4) node[above]{\hlw{$1.51$}};
		\draw (2.5/2,1.5139613-0.006296925045) -- (2.5/2,1.5139613+0.006296925045);
		\draw (2/2, 1.5139613-0.006296925045) -- (3/2, 1.5139613-0.006296925045);
		\draw (2/2, 1.5139613+0.006296925045) -- (3/2, 1.5139613+0.006296925045);

		%random var 5
		\draw (3.5/2,0) -- (3.5/2,1.00) -- (4.9/2,1.00) -- (4.9/2,0) --cycle;
		\filldraw[pattern=north east lines] (3.5/2,0) -- (3.5/2,1.00) -- (4.9/2,1.00) -- (4.9/2,0) --cycle;
		\draw (4.2/2,0) node[above]{\hlw{$1.00$}};
		\draw (4.2/2,1.00-0.00000867640051942029) -- (4.2/2,1.00+0.00000867640051942029);
		\draw (3.7/2, 1.00-0.00000867640051942029) -- (4.7/2, 1.00-0.00000867640051942029);
		\draw (3.7/2, 1.00+0.00000867640051942029) -- (4.7/2, 1.00+0.00000867640051942029);
		
		\draw (4.9/2,0) -- (4.9/2,1.3832498) -- (6.3/2,1.3832498) -- (6.3/2,0)--cycle;
		\filldraw[color=gray!80!black, opacity=0.3] (4.9/2,0) -- (4.9/2,1.3832498) -- (6.3/2,1.3832498) -- (6.3/2,0)--cycle;
		\draw (5.6/2,1/4) node[above]{\hlw{$1.38$}};
		\draw (5.6/2,1.3832498-0.00115900778791413) -- (5.6/2,1.3832498+0.00115900778791413);
		\draw (5.1/2, 1.3832498-0.00115900778791413) -- (6.1/2, 1.3832498-0.00115900778791413);
		\draw (5.1/2, 1.3832498+0.00115900778791413) -- (6.1/2, 1.3832498+0.00115900778791413);
		
		\draw (6.3/2,0) -- (6.3/2,1.5173264) -- (7.7/2,1.5173264) -- (7.7/2,0)  --cycle;
		\filldraw[color=gray!80!black, opacity=0.9] (6.3/2,0) -- (6.3/2,1.5173264) -- (7.7/2,1.5173264) -- (7.7/2,0)  --cycle;
		\draw (7/2,2/4) node[above]{\hlw{$1.52$}};
		\draw (7/2,1.5173264-0.00806715244709548) -- (7/2,1.5173264+0.00806715244709548);
		\draw (6.5/2, 1.5173264-0.00806715244709548) -- (7.5/2, 1.5173264-0.00806715244709548);
		\draw (6.5/2, 1.5173264+0.00806715244709548) -- (7.5/2, 1.5173264+0.00806715244709548);

		%random var 10 - constraints
		\draw (8/2,0) -- (8/2,1.00) -- (9.4/2,1.00) -- (9.4/2,0) --cycle;
		\filldraw[pattern=north east lines] (8/2,0) -- (8/2,1.00) -- (9.4/2,1.00) -- (9.4/2,0) --cycle;
		\draw (8.7/2,0) node[above]{\hlw{$1.00$}};
		\draw (8.7/2,1.00-0.0000115071590806874) -- (8.7/2,1.00+0.0000115071590806874);
		\draw (8.2/2, 1.00-0.0000115071590806874) -- (9.2/2, 1.0-0.0000115071590806874);
		\draw (8.2/2, 1.00+0.0000115071590806874) -- (9.2/2, 1.0+0.0000115071590806874);
		
		%random var 10 - lottery
		\draw (9.4/2,0) -- (9.4/2,1.2788694) -- (10.8/2,1.2788694) -- (10.8/2,0)--cycle;
		\filldraw[color=gray!80!black, opacity=0.3] (9.4/2,0) -- (9.4/2,1.2788694) -- (10.8/2,1.2788694) -- (10.8/2,0)--cycle;
		\draw (10.1/2,1/4) node[above]{\hlw{$1.28$}};
		\draw (10.1/2,1.2788694-0.000838425464258841) -- (10.1/2,1.2788694+0.000838425464258841);
		\draw (9.6/2, 1.2788694-0.000838425464258841) -- (10.6/2, 1.2788694-0.000838425464258841);
		\draw (9.6/2, 1.2788694+0.000838425464258841) -- (10.6/2, 1.2788694+0.000838425464258841);
		
		%random var 10 - bound
		\draw (10.8/2,0) -- (10.8/2,1.5084184) -- (12.2/2,1.5084184) -- (12.2/2,0) --cycle;
		\filldraw[color=gray!80!black, opacity=0.9] (10.8/2,0) -- (10.8/2,1.5084184) -- (12.2/2,1.5084184) -- (12.2/2,0) --cycle;
		\draw (11.5/2,2/4) node[above]{\hlw{$1.51$}};
		\draw (11.5/2,1.5084184-0.00582180556153554) -- (11.5/2,1.5084184+0.00582180556153554);
		\draw (11/2, 1.5084184-0.00582180556153554) -- (12/2, 1.5084184-0.00582180556153554);
		\draw (11/2, 1.5084184+0.00582180556153554) -- (12/2, 1.5084184+0.00582180556153554);
		
		\end{tikzpicture}
	\end{subfigure}\\
	(a)\\
	\hspace{-8mm}
	\begin{subfigure}{.5\columnwidth}%Ethn 1350 agents
		
		\hfill
		\begin{tikzpicture}[scale=1] \normalsize
		
		\coordinate (A) at (-1/2,0);
		\coordinate (B) at (12.7/2,0);
		\coordinate (C) at (-1/2,2.6);
		
		\draw[thick,-latex] (A) -- (B);
		\draw (12.5/2,-0.02) node[below]{\hlw{$\sigma^2$}};
		\draw[thick,-latex] (A) -- (C);
		
		\draw (-1/2,1) node{$\_$};
		\draw (-1.5/2,1) node{$1$};
		\draw (-1.5/2,2.4) node{\rotatebox{90}{welfare ratio}};
		
		\draw (1.1/2,0) node{$|$};
		\draw (1.1/2,-0.1) node[below]{$1$};
		\draw (5.6/2,0) node{$|$};
		\draw (5.6/2,-0.1) node[below]{$5$};
		\draw (10/2,0) node{$|$};
		\draw (10/2,-0.1) node[below]{$10$};

		%ethnicity var 1
		\draw  (-1/2,0) -- (-1/2,1.2020627) -- (0.4/2,1.2020627) -- (0.4/2,0) --cycle;
		\filldraw[pattern=north east lines] (-1/2,0) -- (-1/2,1.2020627) -- (0.4/2,1.2020627) -- (0.4/2,0) --cycle;
		\draw (-0.3/2,0) node[above]{\hlw{$1.20$}};
		\draw (-0.3/2,1.2020627-0.00005) -- (-0.3/2,1.2020627+0.00005);
		\draw (-0.8/2, 1.2020627-0.00005) -- (0.2/2, 1.2020627-0.00005);
		\draw (-0.8/2, 1.2020627+0.00005) -- (0.2/2, 1.2020627+0.00005);
		
		\draw (0.4/2,0) -- (0.4/2,1.259233) -- (1.8/2,1.259233) -- (1.8/2,0)--cycle;
		\filldraw[color=gray!80!black, opacity=0.3] (0.4/2,0) -- (0.4/2,1.259233) -- (1.8/2,1.259233) -- (1.8/2,0)--cycle;
		\draw (1.1/2,1/4) node[above]{\hlw{$1.26$}};
		\draw (1.1/2,1.259233-0.00821508808292104) -- (1.1/2,1.259233+0.00821508808292104);
		\draw (0.6/2, 1.259233-0.00821508808292104) -- (1.6/2, 1.259233-0.00821508808292104);
		\draw (0.6/2, 1.259233+0.00821508808292104) -- (1.6/2, 1.259233+0.00821508808292104);
		
		\draw (1.8/2,0) -- (1.8/2,1.8128535) -- (3.2/2,1.8128535) -- (3.2/2,0)  --cycle;
		\filldraw[color=gray!80!black, opacity=0.9] (1.8/2,0) -- (1.8/2,1.8128535) -- (3.2/2,1.8128535) -- (3.2/2,0)  --cycle;
		\draw (2.5/2,2/4) node[above]{\hlw{$1.81$}};
		\draw (2.5/2,1.8128535-0.08735448555596104) -- (2.5/2,1.8128535+0.08735448555596104);
		\draw (2/2, 1.8128535-0.08735448555596104) -- (3/2, 1.8128535-0.08735448555596104);
		\draw (2/2, 1.8128535+0.08735448555596104) -- (3/2, 1.8128535+0.08735448555596104);

		%ethn var 5 - constraints
		\draw (3.5/2,0) -- (3.5/2,1.1061177) -- (4.9/2,1.1061177) -- (4.9/2,0) --cycle;
		\filldraw[pattern=north east lines] (3.5/2,0) -- (3.5/2,1.1061177) -- (4.9/2,1.1061177) -- (4.9/2,0) --cycle;
		\draw (4.2/2,0) node[above]{\hlw{$1.11$}};
		\draw (4.2/2,1.1061177-0.06252642141974589) -- (4.2/2,1.1061177+0.06252642141974589);
		\draw (3.7/2, 1.1061177-0.06252642141974589) -- (4.7/2, 1.1061177-0.06252642141974589);
		\draw (3.7/2, 1.1061177+0.06252642141974589) -- (4.7/2, 1.1061177+0.06252642141974589);
		
		\draw  (4.9/2,0) -- (4.9/2,1.2008036) -- (6.3/2,1.2008036) -- (6.3/2,0)--cycle;
		\filldraw[color=gray!80!black, opacity=0.3] (4.9/2,0) -- (4.9/2,1.2008036) -- (6.3/2,1.2008036) -- (6.3/2,0)--cycle;
		\draw (5.6/2,1/4) node[above]{\hlw{$1.20$}};
		\draw (5.6/2,1.2008036-0.0700855584689059) -- (5.6/2,1.2008036+0.0700855584689059);
		\draw (5.1/2, 1.2008036-0.0700855584689059) -- (6.1/2, 1.2008036-0.0700855584689059);
		\draw (5.1/2, 1.2008036+0.0700855584689059) -- (6.1/2, 1.2008036+0.0700855584689059);
		
		\draw (6.3/2,0) -- (6.3/2,1.637792) -- (7.7/2,1.637792) -- (7.7/2,0)  --cycle;
		\filldraw[color=gray!80!black, opacity=0.9] (6.3/2,0) -- (6.3/2,1.637792) -- (7.7/2,1.637792) -- (7.7/2,0)  --cycle;
		\draw (7/2,2/4) node[above]{\hlw{$1.64$}};
		\draw (7/2,1.637792-0.09007021432646573) -- (7/2,1.637792+0.09007021432646573);
		\draw (6.5/2, 1.637792-0.09007021432646573) -- (7.5/2, 1.637792-0.09007021432646573);
		\draw (6.5/2, 1.637792+0.09007021432646573) -- (7.5/2, 1.637792+0.09007021432646573);

		%ethn var 10 - constraints
		\draw (8/2,0) -- (8/2,1.0625844) -- (9.4/2,1.0625844) -- (9.4/2,0) --cycle;
		\filldraw[pattern=north east lines] (8/2,0) -- (8/2,1.0625844) -- (9.4/2,1.0625844) -- (9.4/2,0) --cycle;
		\draw (8.7/2,0) node[above]{\hlw{$1.06$}};
		\draw (8.7/2,1.0625844-0.048821856588766845) -- (8.7/2,1.0625844+0.048821856588766845);
		\draw (8.2/2, 1.0625844-0.048821856588766845) -- (9.2/2, 1.0625844-0.048821856588766845);
		\draw (8.2/2, 1.0625844+0.048821856588766845) -- (9.2/2, 1.0625844+0.048821856588766845);
		
		%ethn var 10 - lottery
		\draw (9.4/2,0) -- (9.4/2,1.1554017) -- (10.8/2,1.1554017) -- (10.8/2,0)--cycle;
		\filldraw[color=gray!80!black, opacity=0.3] (9.4/2,0) -- (9.4/2,1.1554017) -- (10.8/2,1.1554017) -- (10.8/2,0)--cycle;
		\draw (10.1/2,1/4) node[above]{\hlw{$1.16$}};
		\draw (10.1/2,1.1554017-0.05412302618886995) -- (10.1/2,1.1554017+0.05412302618886995);
		\draw (9.6/2, 1.1554017-0.05412302618886995) -- (10.6/2, 1.1554017-0.05412302618886995);
		\draw (9.6/2, 1.1554017+0.05412302618886995) -- (10.6/2, 1.1554017+0.05412302618886995);
		
		%ethn var 10 - bound
		\draw (10.8/2,0) -- (10.8/2,1.5638157) -- (12.2/2,1.5638157) -- (12.2/2,0) --cycle;
		\filldraw[color=gray!80!black, opacity=0.9] (10.8/2,0) -- (10.8/2,1.5638157) -- (12.2/2,1.5638157) -- (12.2/2,0) --cycle;
		\draw (11.5/2,2/4) node[above]{\hlw{$1.56$}};
		\draw (11.5/2,1.5638157-0.07724360627477216) -- (11.5/2,1.5638157+0.07724360627477216);
		\draw (11/2, 1.5638157-0.07724360627477216) -- (12/2, 1.5638157-0.07724360627477216);
		\draw (11/2, 1.5638157+0.07724360627477216 ) -- (12/2, 1.5638157+0.07724360627477216);
		\end{tikzpicture}
		%  \caption{$\Ethn(\sigma^2)$}
		%\label{fig:1-sub3}
	\end{subfigure}
	\begin{subfigure}{.5\columnwidth}%Ethn 3000 agents
		
		\hfill
		\begin{tikzpicture}[scale=1] \normalsize
		\coordinate (A) at (-1/2,0);
		\coordinate (B) at (12.7/2,0);
		\coordinate (C) at (-1/2,2.6);
		
		\draw[thick,-latex] (A) -- (B);
		\draw (12.5/2,-0.02) node[below]{$\sigma^2$};
		\draw[thick,-latex] (A) -- (C);
		
		\draw (-1/2,1) node{$\_$};
		\draw (-1.5/2,1) node{$1$};
		\draw (-1.5/2,2.4) node{\rotatebox{90}{welfare ratio}};

		\draw (1.1/2,0) node{$|$};
		\draw (1.1/2,-0.1) node[below]{$1$};
		\draw (5.6/2,0) node{$|$};
		\draw (5.6/2,-0.1) node[below]{$5$};
		\draw (10/2,0) node{$|$};
		\draw (10/2,-0.1) node[below]{$10$};

		%ethnicity var 1
		\draw (-1/2,0) -- (-1/2,1.2984678) -- (0.4/2,1.2984678) -- (0.4/2,0) --cycle;
		\filldraw[pattern=north east lines] (-1/2,0) -- (-1/2,1.2984678) -- (0.4/2,1.2984678) -- (0.4/2,0) --cycle;
		\draw (-0.3/2,0) node[above]{\hlw{$1.30$}};
		\draw (-0.3/2,1.2984678-0.0186077841222681) -- (-0.3/2,1.2984678+0.0186077841222681);
		\draw (-0.8/2, 1.2984678-0.0186077841222681) -- (0.2/2, 1.2984678-0.0186077841222681);
		\draw (-0.8/2, 1.2984678+0.0186077841222681) -- (0.2/2, 1.2984678+0.0186077841222681);
		
		\draw  (0.4/2,0) -- (0.4/2,1.3897196) -- (1.8/2,1.3897196) -- (1.8/2,0)--cycle;
		\filldraw[color=gray!80!black, opacity=0.3] (0.4/2,0) -- (0.4/2,1.3897196) -- (1.8/2,1.3897196) -- (1.8/2,0)--cycle;
		\draw (1.1/2,1/4) node[above]{\hlw{$1.39$}};
		\draw (1.1/2,1.3897196-0.0200686782890189) -- (1.1/2,1.3897196+0.0200686782890189);
		\draw (0.6/2, 1.3897196-0.0200686782890189) -- (1.6/2, 1.3897196-0.0200686782890189);
		\draw (0.6/2, 1.3897196+0.0200686782890189) -- (1.6/2, 1.3897196+0.0200686782890189);
		
		\draw (1.8/2,0) -- (1.8/2,2.5163732) -- (3.2/2,2.5163732) -- (3.2/2,0)  --cycle;
		\filldraw[color=gray!80!black, opacity=0.9] (1.8/2,0) -- (1.8/2,2.5163732) -- (3.2/2,2.5163732) -- (3.2/2,0)  --cycle;
		\draw (2.5/2,2/4) node[above]{\hlw{$2.52$}};
		\draw (2.5/2,2.5163732-0.0538922408271425) -- (2.5/2,2.5163732+0.0538922408271425);
		\draw (2/2, 2.5163732-0.0538922408271425) -- (3/2, 2.51637325-0.0538922408271425);
		\draw (2/2, 2.5163732+0.0538922408271425) -- (3/2, 2.5163732+0.0538922408271425);

		%ethn var 5 - constraints
		\draw (3.5/2,0) -- (3.5/2,1.1601377) -- (4.9/2,1.1601377) -- (4.9/2,0) --cycle;
		\filldraw[pattern=north east lines] (3.5/2,0) -- (3.5/2,1.1601377) -- (4.9/2,1.1601377) -- (4.9/2,0) --cycle;
		\draw (4.2/2,0) node[above]{\hlw{$1.16$}};
		\draw (4.2/2,1.1601377-0.0102516273558165) -- (4.2/2,1.1601377+0.0102516273558165);
		\draw (3.7/2, 1.1601377-0.0102516273558165) -- (4.7/2, 1.1601377-0.0102516273558165);
		\draw (3.7/2, 1.1601377+0.0102516273558165) -- (4.7/2, 1.1601377+0.0102516273558165);
		
		\draw (4.9/2,0) -- (4.9/2,1.33763126) -- (6.3/2,1.3376312) -- (6.3/2,0)--cycle;
		\filldraw[color=gray!80!black, opacity=0.3] (4.9/2,0) -- (4.9/2,1.33763126) -- (6.3/2,1.3376312) -- (6.3/2,0)--cycle;
		\draw (5.6/2,1/4) node[above]{\hlw{$1.34$}};
		\draw (5.6/2,1.3376312-0.0120853279007776) -- (5.6/2,1.3376312+0.0120853279007776);
		\draw (5.1/2, 1.3376312-0.0120853279007776) -- (6.1/2, 1.3376312-0.0120853279007776);
		\draw (5.1/2, 1.3376312+0.0120853279007776) -- (6.1/2, 1.3376312+0.0120853279007776);
		
		\draw  (6.3/2,0) -- (6.3/2,2.2683122) -- (7.7/2,2.2683122) -- (7.7/2,0)  --cycle;
		\filldraw[color=gray!80!black, opacity=0.9] (6.3/2,0) -- (6.3/2,2.2683122) -- (7.7/2,2.2683122) -- (7.7/2,0)  --cycle;
		\draw (7/2,2/4) node[above]{\hlw{$2.27$}};
		\draw (7/2,2.2683122-0.0430439805488617) -- (7/2,2.2683122+0.0430439805488617);
		\draw (6.5/2, 2.2683122-0.0430439805488617) -- (7.5/2, 2.2683122-0.0430439805488617);
		\draw (6.5/2, 2.2683122+0.0430439805488617) -- (7.5/2, 2.2683122+0.0430439805488617);

		%ethn var 10 - constraints
		\draw  (8/2,0) -- (8/2,1.0958045) -- (9.4/2,1.0958045) -- (9.4/2,0) --cycle;
		\filldraw[pattern=north east lines] (8/2,0) -- (8/2,1.0958045) -- (9.4/2,1.0958045) -- (9.4/2,0) --cycle;
		\draw (8.7/2,0) node[above]{\hlw{$1.10$}};
		\draw (8.7/2,1.0958045-0.00973695861674025) -- (8.7/2,1.0958045+0.00973695861674025);
		\draw (8.2/2, 1.0958045-0.00973695861674025) -- (9.2/2, 1.0958045-0.00973695861674025);
		\draw (8.2/2, 1.0958045+0.00973695861674025) -- (9.2/2, 1.0958045+0.00973695861674025);
		
		%ethn var 10 - lottery
		\draw (9.4/2,0) -- (9.4/2,1.2907251) -- (10.8/2,1.2907251) -- (10.8/2,0)--cycle;
		\filldraw[color=gray!80!black, opacity=0.3] (9.4/2,0) -- (9.4/2,1.2907251) -- (10.8/2,1.2907251) -- (10.8/2,0)--cycle;
		\draw (10.1/2,1/4) node[above]{\hlw{$1.29$}};
		\draw (10.1/2,1.2907251-0.0124853575208772) -- (10.1/2,1.2907251+0.0124853575208772);
		\draw (9.6/2, 1.2907251-0.0124853575208772) -- (10.6/2, 1.2907251-0.0124853575208772);
		\draw (9.6/2, 1.2907251+0.0124853575208772) -- (10.6/2, 1.2907251+0.0124853575208772);
		
		%ethn var 10 - bound
		\draw (10.8/2,0) -- (10.8/2,2.12354) -- (12.2/2,2.12354) -- (12.2/2,0) --cycle;
		\filldraw[color=gray!80!black, opacity=0.9] (10.8/2,0) -- (10.8/2,2.12354) -- (12.2/2,2.12354) -- (12.2/2,0) --cycle;
		\draw (11.5/2,2/4) node[above]{\hlw{$2.12$}};
		\draw (11.5/2,2.12354-0.0452269438134879) -- (11.5/2,2.12354+0.0452269438134879);
		\draw (11/2, 2.12354-0.0452269438134879) -- (12/2, 2.12354-0.0452269438134879);
		\draw (11/2, 2.12354+0.0452269438134879) -- (12/2, 2.12354+0.0452269438134879);
		\end{tikzpicture}
		%  \caption{$\Ethn(\sigma^2)$}
		%\label{fig:1-sub3}
	\end{subfigure}\\
	(b)\\
	\caption{Averaged welfare ratios obtained for (a) $\Dist(\sigma^2)$ and (b) $\Ethn(\sigma^2)$ with $n = m = 1350$ (left) and $n = 3000$ (right), $m = 1350$ in our simulated instances of the Singapore public housing allocation problem. Horizontal axis shows values of utility model parameter $\sigma^2$ (variance); vertical axis shows ratio (or upper bound on ratio) of optimal unconstrained welfare to welfare induced by some constrained mechanism --- hatched bar: $\PoD(u)$ as per Definition~\ref{def:pod}, light gray bar: upper bound from Theorem~\ref{thm_parampod}, dark gray bar: $\PoDL(u,\cdot)$ as per Definition~\ref{def:podl}, averaged over $100$ agent permutations. \label{figTests1}}
\end{figure}

Let us now turn to the project approval-based utility model $\Proj(\rho)$. To compute distances in km, we make use of the fact that one degree of latitude or longitude at the location of Singapore corresponds to roughly 111 km; we vary the radius $\rho$ in $\{5,7.5,10\}$ (in km). The results averaged over 100 runs are provided in Figure~\ref{fig:ProjApprov}. In all instances, not only is the price of diversity almost one but the lottery-induced welfare is also nearly as good, achieving at least $87\%$ of the unconstrained optimum for $1350$ agents and practically $100\%$ for $3000$ agents; the disparity parameter is also consistently close to its ideal value of $1$, keeping the upper bound at around $1.45$ regardless of the radius. Thus, this can be considered an example of a utility model for which the lottery mechanism virtually implements a constrained optimal allocation for a wide range of model parameters.

%Figure for project approval-based utilities

\begin{figure}[h]
	\centering
	\hspace{-1.4mm}
	\begin{subfigure}{.5\columnwidth}%1350 agents
	\begin{tikzpicture}[scale=1] \normalsize
	\coordinate (A) at (-1/2,0);
	\coordinate (B) at (12.7/2,0);
	\coordinate (C) at (-1/2,2.6);
	
	\draw[thick,-latex] (A) -- (B);
	\draw (12.5/2,-0.1) node[below]{$\rho$};
	\draw[thick,-latex] (A) -- (C);
	
	\draw (-1/2,1) node{$\_$};
	\draw (-1.5/2,1) node{$1$};
	\draw (-1.5/2,2.4) node{\rotatebox{90}{welfare ratio}};
	
	\draw (1.1/2,0) node{$|$};
	\draw (1.1/2,-0.1) node[below]{$5$};
	\draw (5.6/2,0) node{$|$};
	\draw (5.6/2,-0.1) node[below]{$7.5$};
	\draw (10/2,0) node{$|$};
	\draw (10/2,-0.1) node[below]{$10$};

	%radius 5
	\draw  (-1/2,0) -- (-1/2,1.00) -- (0.4/2,1.00) -- (0.4/2,0) --cycle;
	\filldraw[pattern=north east lines] (-1/2,0) -- (-1/2,1.00) -- (0.4/2,1.00) -- (0.4/2,0) --cycle;
	\draw (-0.3/2,0) node[above]{\hlw{$1.00$}};
	\draw (-0.3/2, 1.00-0.00005282286161185068) -- (-0.3/2,1.00+0.00005282286161185068);
	\draw (-0.8/2, 1.00-0.00005282286161185068) -- (0.2/2, 1.0-0.00005282286161185068);
	\draw (-0.8/2, 1.0+0.00005282286161185068) -- (0.2/2, 1.00+0.00005282286161185068);
	
	\draw  (0.4/2,0) -- (0.4/2,1.0779319) -- (1.8/2,1.0779319) -- (1.8/2,0)--cycle;
	\filldraw[color=gray!80!black, opacity=0.3]  (0.4/2,0) -- (0.4/2,1.0779319) -- (1.8/2,1.0779319) -- (1.8/2,0)--cycle;
	\draw (1.1/2,1/4) node[above]{\hlw{$1.08$}};
	\draw (1.1/2, 1.0779319-0.0008012275854516601) -- (1.1/2, 1.0779319+0.0008012275854516601);
	\draw (0.6/2, 1.0779319-0.0008012275854516601) -- (1.6/2, 1.0779319-0.0008012275854516601);
	\draw (0.6/2, 1.0779319+0.0008012275854516601) -- (1.6/2, 1.0779319+0.0008012275854516601);
	
	\draw (1.8/2,0) -- (1.8/2,1.4538779) -- (3.2/2,1.4538779) -- (3.2/2,0)  --cycle;
	\filldraw[color=gray!80!black, opacity=0.9]  (1.8/2,0) -- (1.8/2,1.4538779) -- (3.2/2,1.4538779) -- (3.2/2,0)  --cycle;
	\draw (2.5/2,2/4) node[above]{\hlw{$1.45$}};
	\draw (2.5/2,1.4538779-0.0017859091582984148) -- (2.5/2,1.4538779+0.0017859091582984148);
	\draw (2/2, 1.4538779-0.0017859091582984148) -- (3/2, 1.4538779-0.0017859091582984148);
	\draw (2/2, 1.4538779+0.0017859091582984148) -- (3/2, 1.4538779+0.0017859091582984148);

	%radius 7.5
	\draw (3.5/2,0) -- (3.5/2,1.00) -- (4.9/2,1.00) -- (4.9/2,0) --cycle;
	\filldraw[pattern=north east lines] (3.5/2,0) -- (3.5/2,1.00) -- (4.9/2,1.00) -- (4.9/2,0) --cycle;
	\draw (4.2/2,0) node[above]{\hlw{$1.00$}};
	\draw (4.2/2,1.00-0.00005257287812715034) -- (4.2/2,1.00+0.00005257287812715034);
	\draw (3.7/2, 1.0-0.00005257287812715034) -- (4.7/2, 1.00-0.00005257287812715034);
	\draw (3.7/2, 1.00+0.00005257287812715034) -- (4.7/2, 1.000+0.00005257287812715034);
	
	\draw  (4.9/2,0) -- (4.9/2,1.1503563) -- (6.3/2,1.1503563) -- (6.3/2,0)--cycle;
	\filldraw[color=gray!80!black, opacity=0.3]  (4.9/2,0) -- (4.9/2,1.1503563) -- (6.3/2,1.1503563) -- (6.3/2,0)--cycle;
	\draw (5.6/2,1/4) node[above]{\hlw{$1.15$}};
	\draw (5.6/2,1.1503563-0.002245484061291389) -- (5.6/2,1.1503563+0.002245484061291389);
	\draw (5.1/2, 1.1503563-0.002245484061291389) -- (6.1/2, 1.1503563-0.002245484061291389);
	\draw (5.1/2, 1.1503563+0.002245484061291389) -- (6.1/2, 1.1503563+0.002245484061291389);
	
	\draw (6.3/2,0) -- (6.3/2,1.4560245) -- (7.7/2,1.4560245) -- (7.7/2,0)  --cycle;
	\filldraw[color=gray!80!black, opacity=0.9] (6.3/2,0) -- (6.3/2,1.4560245) -- (7.7/2,1.4560245) -- (7.7/2,0)  --cycle;
	\draw (7/2,2/4) node[above]{\hlw{$1.46$}};
	\draw (7/2,1.4560245-0.0011169362777883698) -- (7/2,1.4560245+0.0011169362777883698);
	\draw (6.5/2, 1.4560245-0.0011169362777883698) -- (7.5/2,1.4560245-0.0011169362777883698);
	\draw (6.5/2, 1.4560245+0.0011169362777883698) -- (7.5/2, 1.4560245+0.0011169362777883698);

	%radius 10
	\draw (8/2,0) -- (8/2,1.00) -- (9.4/2,1.00) -- (9.4/2,0) --cycle;
	\filldraw[pattern=north east lines] (8/2,0) -- (8/2,1.00) -- (9.4/2,1.00) -- (9.4/2,0) --cycle;
	\draw (8.7/2,0) node[above]{\hlw{$1.00$}};
	\draw (8.7/2,1.00-0.0) -- (8.7/2,1.00+0.0);
	\draw (8.2/2, 1.00-0.0) -- (9.2/2, 1.00-0.0);
	\draw (8.2/2, 1.00+0.0) -- (9.2/2, 1.00+0.0);
	
	\draw  (9.4/2,0) -- (9.4/2,1.0677265) -- (10.8/2,1.0677265) -- (10.8/2,0)--cycle;
	\filldraw[color=gray!80!black, opacity=0.3]  (9.4/2,0) -- (9.4/2,1.0677265) -- (10.8/2,1.0677265) -- (10.8/2,0)--cycle;
	\draw (10.1/2,1/4) node[above]{\hlw{$1.07$}};
	\draw (10.1/2,1.0677265-0.001060258302101601) -- (10.1/2,1.0677265+0.001060258302101601);
	\draw (9.6/2, 1.0677265-0.001060258302101601) -- (10.6/2, 1.0677265-0.001060258302101601);
	\draw (9.6/2, 1.0677265+0.001060258302101601) -- (10.6/2, 1.0677265+0.001060258302101601);
	
	\draw  (10.8/2,0) -- (10.8/2,1.4509482) -- (12.2/2,1.4509482) -- (12.2/2,0) --cycle;
	\filldraw[color=gray!80!black, opacity=0.9] (10.8/2,0) -- (10.8/2,1.4509482) -- (12.2/2,1.4509482) -- (12.2/2,0) --cycle;
	\draw (11.5/2,2/4) node[above]{\hlw{$1.45$}};
	\draw (11.5/2,1.4509482-0.0001920647767167375) -- (11.5/2,1.4509482+0.0001920647767167375);
	\draw (11/2, 1.4509482-0.0001920647767167375) -- (12/2, 1.4509482-0.0001920647767167375);
	\draw (11/2, 1.4509482+0.0001920647767167375) -- (12/2, 1.4509482+0.0001920647767167375);
	
	\end{tikzpicture}
    \end{subfigure}
	\begin{subfigure}{.5\columnwidth}%3000 agents
		
		\hfill
		\begin{tikzpicture}[scale=1] \normalsize
		\coordinate (A) at (-1/2,0);
		\coordinate (B) at (12.7/2,0);
		\coordinate (C) at (-1/2,2.6);
		
		\draw[thick,-latex] (A) -- (B);
		\draw (12.5/2,-0.1) node[below]{$\rho$};
		\draw[thick,-latex] (A) -- (C);
		
		\draw (-1/2,1) node{$\_$};
		\draw (-1.5/2,1) node{$1$};
		\draw (-1.5/2,2.4) node{\rotatebox{90}{welfare ratio}};
		
		\draw (1.1/2,0) node{$|$};
		\draw (1.1/2,-0.1) node[below]{$5$};
		\draw (5.6/2,0) node{$|$};
		\draw (5.6/2,-0.1) node[below]{$7.5$};
		\draw (10/2,0) node{$|$};
		\draw (10/2,-0.1) node[below]{$10$};

		%radius 5
		\draw (-1/2,0) -- (-1/2,1.00) -- (0.4/2,1.0) -- (0.4/2,0) --cycle;
		\filldraw[pattern=north east lines] (-1/2,0) -- (-1/2,1.0) -- (0.4/2,1.00) -- (0.4/2,0) --cycle;
		\draw (-0.3/2,0) node[above]{\hlw{$1.00$}};
		\draw (-0.3/2, 1.0-0.0001519354742898406) -- (-0.3/2,1.000+0.0001519354742898406);
		\draw (-0.8/2, 1.0-0.0001519354742898406) -- (0.2/2, 1.00-0.0001519354742898406);
		\draw (-0.8/2, 1.00+0.0001519354742898406) -- (0.2/2, 1.0+0.0001519354742898406);
		
		\draw  (0.4/2,0) -- (0.4/2,1.051549) -- (1.8/2,1.051549) -- (1.8/2,0)--cycle;
		\filldraw[color=gray!80!black, opacity=0.3] (0.4/2,0) -- (0.4/2,1.051549) -- (1.8/2,1.051549) -- (1.8/2,0)--cycle;
		\draw (1.1/2,1/4) node[above]{\hlw{$1.05$}};
		\draw (1.1/2, 1.051549-0.00036593556883131523) -- (1.1/2, 1.051549+0.00036593556883131523);
		\draw (0.6/2, 1.0515499-0.00036593556883131523) -- (1.6/2, 1.051549-0.00036593556883131523);
		\draw (0.6/2, 1.051549+0.00036593556883131523) -- (1.6/2, 1.051549+0.00036593556883131523);
		
		\draw  (1.8/2,0) -- (1.8/2,1.4515908) -- (3.2/2,1.4515908) -- (3.2/2,0)  --cycle;
		\filldraw[color=gray!80!black, opacity=0.9] (1.8/2,0) -- (1.8/2,1.4515908) -- (3.2/2,1.4515908) -- (3.2/2,0)  --cycle;
		\draw (2.5/2,2/4) node[above]{\hlw{$1.45$}};
		\draw (2.5/2,1.4515908-0.0) -- (2.5/2,1.4515908+0.0);
		\draw (2/2, 1.4515908-0.0) -- (3/2, 1.4515908-0.0);
		\draw (2/2, 1.4515908+0.0) -- (3/2, 1.45159084+0.0);
		%0.0 for 2.1945377249831633E-16
		
		%radius 7.5
		\draw (3.5/2,0) -- (3.5/2,1.0) -- (4.9/2,1.0) -- (4.9/2,0) --cycle;
		\filldraw[pattern=north east lines] (3.5/2,0) -- (3.5/2,1.0) -- (4.9/2,1.0) -- (4.9/2,0) --cycle;
		\draw (4.2/2,0) node[above]{\hlw{$1.00$}};
		\draw (4.2/2,1.0-0.0) -- (4.2/2,1.0+0.0);
		\draw (3.7/2, 1.0-0.0) -- (4.7/2, 1.0-0.0);
		\draw (3.7/2, 1.0+0.0) -- (4.7/2, 1.0+0.0);
		
		\draw (4.9/2,0) -- (4.9/2,1.0) -- (6.3/2,1.0) -- (6.3/2,0)--cycle;
		\filldraw[color=gray!80!black, opacity=0.3] (4.9/2,0) -- (4.9/2,1.0) -- (6.3/2,1.0) -- (6.3/2,0)--cycle;
		\draw (5.6/2,1/4) node[above]{\hlw{$1.00$}};
		\draw (5.6/2,1.0-0.0) -- (5.6/2,1.0+0.0);
		\draw (5.1/2, 1.0-0.0) -- (6.1/2, 1.0-0.0);
		\draw (5.1/2, 1.0+0.0) -- (6.1/2, 1.0+0.0);
		
		\draw (6.3/2,0) -- (6.3/2,1.4515908) -- (7.7/2,1.4515908) -- (7.7/2,0)  --cycle;
		\filldraw[color=gray!80!black, opacity=0.9] (6.3/2,0) -- (6.3/2,1.4515908) -- (7.7/2,1.4515908) -- (7.7/2,0)  --cycle;
		\draw (7/2,2/4) node[above]{\hlw{$1.45$}};
		\draw (7/2,1.4515908-0.0) -- (7/2,1.4515908+0.0);
		\draw (6.5/2, 1.4515908-0.0) -- (7.5/2,1.4515908-0.0);
		\draw (6.5/2, 1.4515908+0.0) -- (7.5/2, 1.4515908+0.0);
		%0.0 for 6.605593752270662E-17

		%radius 10
		\draw (8/2,0) -- (8/2,1.0) -- (9.4/2,1.0) -- (9.4/2,0) --cycle;
		\filldraw[pattern=north east lines] (8/2,0) -- (8/2,1.0) -- (9.4/2,1.0) -- (9.4/2,0) --cycle;
		\draw (8.7/2,0) node[above]{\hlw{$1.00$}};
		\draw (8.7/2,1.0-0.0) -- (8.7/2,1.0+0.0);
		\draw (8.2/2, 1.0-0.0) -- (9.2/2, 1.0-0.0);
		\draw (8.2/2, 1.0+0.0) -- (9.2/2, 1.0+0.0);
		
		\draw (9.4/2,0) -- (9.4/2,1.0) -- (10.8/2,1.0) -- (10.8/2,0)--cycle;
		\filldraw[color=gray!80!black, opacity=0.3] (9.4/2,0) -- (9.4/2,1.0) -- (10.8/2,1.0) -- (10.8/2,0)--cycle;
		\draw (10.1/2,1/4) node[above]{\hlw{$1.00$}};
		\draw (10.1/2,1.0-0.0) -- (10.1/2,1.0+0.0);
		\draw (9.6/2, 1.0-0.0) -- (10.6/2, 1.0-0.0);
		\draw (9.6/2, 1.0+0.0) -- (10.6/2, 1.0+0.0);
		
		\draw (10.8/2,0) -- (10.8/2,1.4513265) -- (12.2/2,1.4513265) -- (12.2/2,0) --cycle;
		\filldraw[color=gray!80!black, opacity=0.9] (10.8/2,0) -- (10.8/2,1.4513265) -- (12.2/2,1.4513265) -- (12.2/2,0) --cycle;
		\draw (11.5/2,2/4) node[above]{\hlw{$1.45$}};
		\draw (11.5/2,1.4515908-0.0) -- (11.5/2,1.4515908+0.0);
		\draw (11/2, 1.4515908-0.0) -- (12/2, 1.4515908-0.0);
		\draw (11/2, 1.4515908+0.0) -- (12/2, 1.4515908+0.0);
		%0.0 for 8.390818541177929E-17
		
		\end{tikzpicture}
	\end{subfigure}
	\caption{Averaged welfare ratios obtained for $\Proj(\rho)$ with $n = m = 1350$ (left) and $n = 3000$, $m = 1350$ (right) in our simulated instances of the Singapore public housing allocation problem. Horizontal axis shows values of the utility model parameter $\rho$ (radius); vertical axis shows ratio (or upper bound on ratio) of optimal unconstrained welfare to welfare induced by some constrained mechanism --- hatched bar: $\PoD(u)$ as per Definition~\ref{def:pod}, light gray bar: upper bound from Theorem~\ref{thm_parampod}, dark gray bar: $\PoDL(u,\cdot)$ as per Definition~\ref{def:podl}, averaged over $100$ agent permutations. \label{fig:ProjApprov}}
\end{figure}

Finally, we study the price-based utility model $\Price(\sigma^2)$, varying $\sigma^2$ in $\{0,10,50\}$; the results obtained by averaging over 100 runs are given in Figure~\ref{figTests3}.
While the price of diversity is practically equal to one in all instances, the welfare loss observed with the lottery mechanism drastically increases with $\sigma^2$ (recall that agents from the same ethnicity group have identical preferences when $\sigma^2 = 0$): for instance, for $1350$ agents, it extracts $98\%$ of the optimal unconstrained welfare under $\Price(0)$ while it only extracts $35\%$ of this value under $\Price(50)$. 
%Moreover, for $\Price(50)$, %we found problem instances in which the the welfare loss induced by the lottery mechanism exceeds the corresponding theoretical upper bound on $\PoD(u)$.
 %we observe that the welfare loss induced by the lottery mechanism exceeds the theoretical upper bound on $\PoD(u)$ in some instances with $1350$ agents. %(but this effect appears to be mitigated by a larger number of agents in our experiments). 
These numerical tests show that utility models exist for which the lottery mechanism may perform poorly compared to the optimal constrained allocation mechanism, even in allocation problems with a very low price of diversity.

%Figure for price based utilities

\begin{figure}[h]
\centering
\hspace{-1.4mm}
\begin{subfigure}{.5\columnwidth}%1350 agents
\begin{tikzpicture}[scale=1] \normalsize
\coordinate (A) at (-1/2,0);
\coordinate (B) at (12.7/2,0);
\coordinate (C) at (-1/2,5);

\draw[thick,-latex] (A) -- (B);
\draw (12.5/2,-0.02) node[below]{$\sigma^2$};
\draw[thick,-latex] (A) -- (C);

\draw (-1/2,1) node{$\_$};
\draw (-1.5/2,1) node{$1$};
\draw (-1.5/2,4.2) node{\rotatebox{90}{welfare ratio}};

\draw (1.1/2,0) node{$|$};
\draw (1.1/2,-0.1) node[below]{$0$};
\draw (5.6/2,0) node{$|$};
\draw (5.6/2,-0.1) node[below]{$10$};
\draw (10/2,0) node{$|$};
\draw (10/2,-0.1) node[below]{$50$};

%random var 0
\draw  (-1/2,0) -- (-1/2,1.00) -- (0.4/2,1.00) -- (0.4/2,0) --cycle;
\filldraw[pattern=north east lines] (-1/2,0) -- (-1/2,1.00) -- (0.4/2,1.00) -- (0.4/2,0) --cycle;
\draw (-0.3/2,0) node[above]{\hlw{$1.00$}};
\draw (-0.3/2, 1.00-0.000010066365846028857) -- (-0.3/2,1.00+0.000010066365846028857);
\draw (-0.8/2, 1.00-0.000010066365846028857) -- (0.2/2, 1.00-0.000010066365846028857);
\draw (-0.8/2, 1.00+0.000010066365846028857) -- (0.2/2, 1.00+0.000010066365846028857);

\draw (0.4/2,0) -- (0.4/2,1.0249861) -- (1.8/2,1.0249861) -- (1.8/2,0)--cycle;
\filldraw[color=gray!80!black, opacity=0.3] (0.4/2,0) -- (0.4/2,1.0249861) -- (1.8/2,1.0249861) -- (1.8/2,0)--cycle;
\draw (1.1/2,1/4) node[above]{\hlw{$1.02$}};
\draw (1.1/2, 1.0249861-0.002655117164065021) -- (1.1/2, 1.0249861+0.002655117164065021);
\draw (0.6/2, 1.0249861-0.002655117164065021) -- (1.6/2, 1.0249861-0.002655117164065021);
\draw (0.6/2, 1.0249861+0.002655117164065021) -- (1.6/2, 1.0249861+0.002655117164065021);

\draw (1.8/2,0) -- (1.8/2,3.2287123) -- (3.2/2,3.2287123) -- (3.2/2,0)  --cycle;
\filldraw[color=gray!80!black, opacity=0.9] (1.8/2,0) -- (1.8/2,3.2287123) -- (3.2/2,3.2287123) -- (3.2/2,0)  --cycle;
\draw (2.5/2,2/4) node[above]{\hlw{$3.23$}};
\draw (2.5/2,3.2287123-0.0012919861748411875) -- (2.5/2,3.2287123+0.0012919861748411875);
\draw (2/2, 3.2287123-0.0012919861748411875) -- (3/2, 3.2287123-0.0012919861748411875);
\draw (2/2, 3.2287123+0.0012919861748411875) -- (3/2, 3.2287123+0.0012919861748411875);

%random var 10
\draw  (3.5/2,0) -- (3.5/2,1.000) -- (4.9/2,1.00) -- (4.9/2,0) --cycle;
\filldraw[pattern=north east lines] (3.5/2,0) -- (3.5/2,1.002) -- (4.9/2,1.00) -- (4.9/2,0) --cycle;
\draw (4.2/2,0) node[above]{\hlw{$1.00$}};
\draw (4.2/2,1.0-0.000003415081036494972) -- (4.2/2,1.000+0.000003415081036494972);
\draw (3.7/2, 1.00-0.000003415081036494972) -- (4.7/2, 1.00-0.000003415081036494972);
\draw (3.7/2, 1.00+0.000003415081036494972) -- (4.7/2, 1.0+0.000003415081036494972);

\draw  (4.9/2,0) -- (4.9/2,2.1536946) -- (6.3/2,2.1536946) -- (6.3/2,0)--cycle;
\filldraw[color=gray!80!black, opacity=0.3] (4.9/2,0) -- (4.9/2,2.1536946) -- (6.3/2,2.1536946) -- (6.3/2,0)--cycle;
\draw (5.6/2,1/4) node[above]{\hlw{$2.15$}};
\draw (5.6/2,2.1536946-0.02978231642766816) -- (5.6/2,2.1536946+0.02978231642766816);
\draw (5.1/2, 2.1536946-0.02978231642766816) -- (6.1/2, 2.1536946-0.02978231642766816);
\draw (5.1/2, 2.1536946+0.02978231642766816) -- (6.1/2, 2.1536946+0.02978231642766816);

\draw (6.3/2,0) -- (6.3/2,3.8367848) -- (7.7/2,3.8367848) -- (7.7/2,0)  --cycle;
\filldraw[color=gray!80!black, opacity=0.9] (6.3/2,0) -- (6.3/2,3.8367848) -- (7.7/2,3.8367848) -- (7.7/2,0)  --cycle;
\draw (7/2,2/4) node[above]{\hlw{$3.84$}};
\draw (7/2,3.8367848-0.13431284321414882) -- (7/2,3.8367848+0.13431284321414882);
\draw (6.5/2, 3.8367848-0.13431284321414882) -- (7.5/2,3.8367848-0.13431284321414882);
\draw (6.5/2, 3.8367848+0.13431284321414882) -- (7.5/2, 3.8367848+0.13431284321414882);

%random var 50
\draw (8/2,0) -- (8/2,1.00) -- (9.4/2,1.0) -- (9.4/2,0) --cycle;
\filldraw[pattern=north east lines] (8/2,0) -- (8/2,1.000) -- (9.4/2,1.00) -- (9.4/2,0) --cycle;
\draw (8.7/2,0) node[above]{\hlw{$1.00$}};
\draw (8.7/2,1.00-0.000001812873141482601) -- (8.7/2,1.00+0.000001812873141482601);
\draw (8.2/2, 1.000-0.000001812873141482601) -- (9.2/2, 1.0-0.000001812873141482601);
\draw (8.2/2, 1.000+0.000001812873141482601) -- (9.2/2, 1.00+0.000001812873141482601);

\draw (9.4/2,0) -- (9.4/2,2.8566687) -- (10.8/2,2.8566687) -- (10.8/2,0)--cycle;
\filldraw[color=gray!80!black, opacity=0.3] (9.4/2,0) -- (9.4/2,2.8566687) -- (10.8/2,2.8566687) -- (10.8/2,0)--cycle;
\draw (10.1/2,1/4) node[above]{\hlw{$2.86$}};
\draw (10.1/2,2.8566687-0.02915663322998673) -- (10.1/2,2.8566687+0.02915663322998673);
\draw (9.6/2, 2.8566687-0.02915663322998673) -- (10.6/2, 2.8566687-0.02915663322998673);
\draw (9.6/2, 2.8566687+0.02915663322998673) -- (10.6/2, 2.8566687+0.02915663322998673);

\draw (10.8/2,0) -- (10.8/2,3.5933943) -- (12.2/2,3.5933943) -- (12.2/2,0) --cycle;
\filldraw[color=gray!80!black, opacity=0.9] (10.8/2,0) -- (10.8/2,3.5933943) -- (12.2/2,3.5933943) -- (12.2/2,0) --cycle;
\draw (11.5/2,2/4) node[above]{\hlw{$3.59$}};
\draw (11.5/2,3.5933943-0.1160429889854055) -- (11.5/2,3.5933943+0.1160429889854055);
\draw (11/2, 3.5933943-0.1160429889854055) -- (12/2, 3.5933943-0.1160429889854055);
\draw (11/2, 3.5933943+0.1160429889854055) -- (12/2, 3.5933943+0.1160429889854055);

\end{tikzpicture}
\end{subfigure}
\begin{subfigure}{.5\columnwidth}%3000 agents
	\begin{tikzpicture}[scale=1] \normalsize
	\coordinate (A) at (-1/2,0);
	\coordinate (B) at (12.7/2,0);
	\coordinate (C) at (-1/2,5);
	
	\draw[thick,-latex] (A) -- (B);
	\draw (12.5/2,-0.02) node[below]{$\sigma^2$};
	\draw[thick,-latex] (A) -- (C);
	
	\draw (-1/2,1) node{$\_$};
	\draw (-1.5/2,1) node{$1$};
	\draw (-1.5/2,4.2) node{\rotatebox{90}{welfare ratio}};
	
	\draw (1.1/2,0) node{$|$};
	\draw (1.1/2,-0.1) node[below]{$0$};
	\draw (5.6/2,0) node{$|$};
	\draw (5.6/2,-0.1) node[below]{$10$};
	\draw (10/2,0) node{$|$};
	\draw (10/2,-0.1) node[below]{$50$};
	
	%random var 0
	\draw (-1/2,0) -- (-1/2,1.0) -- (0.4/2,1.00) -- (0.4/2,0) --cycle;
	\filldraw[pattern=north east lines] (-1/2,0) -- (-1/2,1.0) -- (0.4/2,1.0) -- (0.4/2,0) --cycle;
	\draw (-0.3/2,0) node[above]{\hlw{$1.00$}};
	\draw (-0.3/2, 1.0-0.00006749640745528032) -- (-0.3/2,1.0+0.00006749640745528032);
	\draw (-0.8/2, 1.0-0.00006749640745528032) -- (0.2/2, 1.0-0.00006749640745528032);
	\draw (-0.8/2, 1.0+0.00006749640745528032) -- (0.2/2, 1.00+0.00006749640745528032);
	
	\draw (0.4/2,0) -- (0.4/2,1.0261301) -- (1.8/2,1.0261301) -- (1.8/2,0)--cycle;
	\filldraw[color=gray!80!black, opacity=0.3] (0.4/2,0) -- (0.4/2,1.0261301) -- (1.8/2,1.0261301) -- (1.8/2,0)--cycle;
	\draw (1.1/2,1/4) node[above]{\hlw{$1.03$}};
	\draw (1.1/2, 1.0261301-0.003010456922195181) -- (1.1/2, 1.0261301+0.003010456922195181);
	\draw (0.6/2, 1.0261301-0.003010456922195181) -- (1.6/2, 1.0261301-0.003010456922195181);
	\draw (0.6/2, 1.0261301+0.0030104569221951813) -- (1.6/2, 1.0261301+0.003010456922195181);
	
	\draw (1.8/2,0) -- (1.8/2,3.2310586) -- (3.2/2,3.2310586) -- (3.2/2,0)  --cycle;
	\filldraw[color=gray!80!black, opacity=0.9] (1.8/2,0) -- (1.8/2,3.2310586) -- (3.2/2,3.2310586) -- (3.2/2,0)  --cycle;
	\draw (2.5/2,2/4) node[above]{\hlw{$3.23$}};
	\draw (2.5/2,3.2310586 -0.0012140550406617855) -- (2.5/2,3.2310586+0.0012140550406617855);
	\draw (2/2, 3.2310586-0.0012140550406617855) -- (3/2, 3.2310586-0.0012140550406617855);
	\draw (2/2, 3.2310586+0.0012140550406617855) -- (3/2, 3.2310586+0.0012140550406617855);

	%random var 10
	\draw (3.5/2,0) -- (3.5/2,1.00) -- (4.9/2,1.0) -- (4.9/2,0) --cycle;
	\filldraw[pattern=north east lines] (3.5/2,0) -- (3.5/2,1.0) -- (4.9/2,1.000) -- (4.9/2,0) --cycle;
	\draw (4.2/2,0) node[above]{\hlw{$1.00$}};
	\draw (4.2/2,1.00-0.000029920885694495303) -- (4.2/2,1.000+0.000029920885694495303);
	\draw (3.7/2, 1.0-0.000029920885694495303) -- (4.7/2, 1.0-0.000029920885694495303);
	\draw (3.7/2, 1.00+0.000029920885694495303) -- (4.7/2, 1.00+0.000029920885694495303);
	
	\draw (4.9/2,0) -- (4.9/2,2.2648854) -- (6.3/2,2.2648854) -- (6.3/2,0)--cycle;
	\filldraw[color=gray!80!black, opacity=0.3] (4.9/2,0) -- (4.9/2,2.2648854) -- (6.3/2,2.2648854) -- (6.3/2,0)--cycle;
	\draw (5.6/2,1/4) node[above]{\hlw{$2.26$}};
	\draw (5.6/2,2.2648854-0.03179460038863851) -- (5.6/2,2.2648854+0.03179460038863851);
	\draw (5.1/2, 2.2648854-0.03179460038863851) -- (6.1/2, 2.2648854-0.03179460038863851);
	\draw (5.1/2, 2.2648854+0.03179460038863851) -- (6.1/2, 2.2648854+0.03179460038863851);
	
	\draw (6.3/2,0) -- (6.3/2,3.7084708 ) -- (7.7/2,3.7084708 ) -- (7.7/2,0)  --cycle;
	\filldraw[color=gray!80!black, opacity=0.9] (6.3/2,0) -- (6.3/2,3.7084708 ) -- (7.7/2,3.7084708 ) -- (7.7/2,0)  --cycle;
	\draw (7/2,2/4) node[above]{\hlw{$3.71$}};
	\draw (7/2,3.7084708 -0.10108855611470575) -- (7/2,3.7084708 +0.10108855611470575);
	\draw (6.5/2, 3.7084708 -0.10108855611470575) -- (7.5/2,3.7084708 -0.10108855611470575);
	\draw (6.5/2, 3.7084708 +0.10108855611470575) -- (7.5/2, 3.7084708 +0.10108855611470575);

	%random var 50
	\draw (8/2,0) -- (8/2,1.00) -- (9.4/2,1.00) -- (9.4/2,0) --cycle;
	\filldraw[pattern=north east lines] (8/2,0) -- (8/2,1.0) -- (9.4/2,1.0) -- (9.4/2,0) --cycle;
	\draw (8.7/2,0) node[above]{\hlw{$1.00$}};
	\draw (8.7/2, 1.000-0.00000355633007674229) -- (8.7/2,1.000+0.00000355633007674229);
	\draw (8.2/2, 1.000-0.00000355633007674229) -- (9.2/2, 1.00-0.00000355633007674229);
	\draw (8.2/2, 1.00+0.00000355633007674229) -- (9.2/2, 1.000+0.00000355633007674229);
	
	\draw (9.4/2,0) -- (9.4/2,2.9919255) -- (10.8/2,2.9919255) -- (10.8/2,0)--cycle;
	\filldraw[color=gray!80!black, opacity=0.3] (9.4/2,0) -- (9.4/2,2.9919255) -- (10.8/2,2.9919255) -- (10.8/2,0)--cycle;
	\draw (10.1/2,1/4) node[above]{\hlw{$2.99$}};
	\draw (10.1/2, 2.9919255-0.03323042870932185) -- (10.1/2, 2.9919255+0.03323042870932185);
	\draw (9.6/2, 2.9919255-0.03323042870932185) -- (10.6/2, 2.9919255-0.03323042870932185);
	\draw (9.6/2, 2.9919255+0.03323042870932185) -- (10.6/2, 2.9919255+0.03323042870932185);
	
	\draw (10.8/2,0) -- (10.8/2,3.8903532) -- (12.2/2,3.8903532) -- (12.2/2,0)  --cycle;
	\filldraw[color=gray!80!black, opacity=0.9] (10.8/2,0) -- (10.8/2,3.8903532) -- (12.2/2,3.8903532) -- (12.2/2,0)  --cycle;
	\draw (11.5/2,2/4) node[above]{\hlw{$3.89$}};
	\draw (11.5/2,3.8903532-0.11756406246190147) -- (11.5/2,3.8903532+0.11756406246190147);
	\draw (11/2, 3.8903532-0.11756406246190147) -- (12/2, 3.8903532-0.11756406246190147);
	\draw (11/2, 3.8903532+0.11756406246190147) -- (12/2, 3.8903532+0.11756406246190147);
	
	\end{tikzpicture}
\end{subfigure}
\caption{Averaged welfare ratios obtained for $\Price(\sigma^2)$ with $n = m = 1350$ (left) and $n = 3000$, $m = 1350$ (right) in our simulated instances of the Singapore public housing allocation problem. Horizontal axis shows values of utility model parameter $\sigma^2$ (variance); vertical axis shows ratio (or upper bound on ratio) of optimal unconstrained welfare to welfare induced by some constrained mechanism --- hatched bar: $\PoD(u)$ as per Definition~\ref{def:pod}, light gray bar: upper bound from Theorem~\ref{thm_parampod}, dark gray bar: $\PoDL(u,\cdot)$ as per Definition~\ref{def:podl}, averaged over $100$ agent permutations. \label{figTests3}}
\end{figure}

\subsection{Chicago Public School Admissions}\label{sec:chicago_expts}

\paragraph{Data Collection} From the Chicago Public Schools website,\footnote{\url{http://cps.edu/Pages/home.aspx}} we collected data on the locations of magnet schools in Chicago; we focus on these schools because they use a lottery mechanism to select students. 
We also collected the total number of students enrolled in these schools in 2018, and divided this number by $9$ to obtain the approximate number of students that can be accepted to the first grade (there are nine grades in total). This leads us to instances with $l=37$ blocks (schools) and $m = 2261$ items in total (available spots). 
In this school admission problem, students are partitioned into $k=4$ types, viz. Tiers 1, 2, 3 and 4, depending on their residence (see Figure \ref{ChicagoMap}). In our experiments, we consider a pool of $n \in \{2261, 5000\}$ students whose type composition follows the real-world proportion: we have $|N_1| = 613$ ($\approx 27.1 \%$: Tier 1), $|N_2| = 622$ ($\approx 27.5 \%$: Tier 2), $|N_3| = 533$ ($\approx 23.6 \%$: Tier 3) and $|N_4| = 493$ ($\approx 21.8 \%$: Tier 4) for $n = 2261$ and we have $|N_1| = 1355$, $|N_2| = 1375$, $|N_3| = 1180$ and $|N_4| = 1090$ for $n=5000$. We use equal fractional quotas across schools and tiers, i.e. our type-block capacities are $\lambda_{pq} = 0.25 |M_q|$ $\forall (p,q) \in [k]\times [l]$. Recall also that we have one-shot lotteries (Section~\ref{sec:lottery}) in our simulations.
%the same type-block capacities $\lambda_{pq},(p,q) \in [k]\times [l]$ as in the Chicago public school admission problem, which are simply defined by $\lambda_{pq} = 0.25 |M_q|$.

\begin{figure}[!t]
\centering
\includegraphics[scale=0.4]{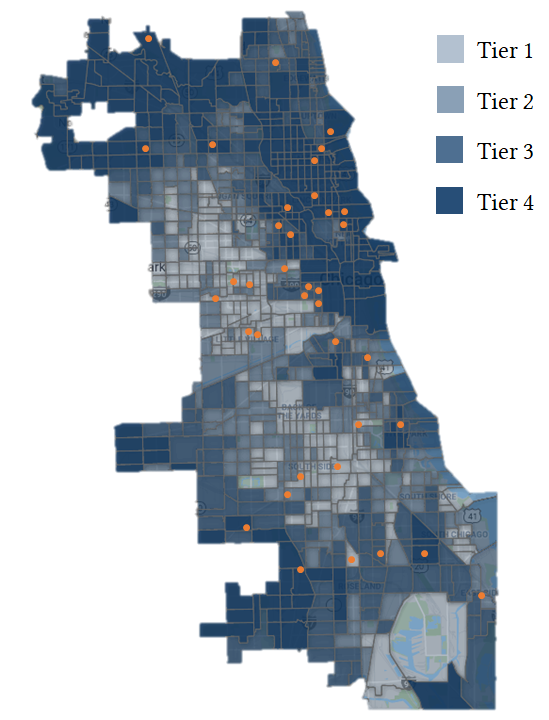}
\caption{Map of Chicago showing the tier statuses of census tracts (\url{http://cpstiers.opencityapps.org/}) and locations of magnet schools (orange dots) based on \url{http://cps.edu/ScriptLibrary/Map-SchoolLocator/index.html}.\label{ChicagoMap}}
\end{figure}

\paragraph{Utility model} Since we do not have access to students' utilities for schools, we simulate them as follows. We use the distance-based utility model $\Dist(\sigma^2)$ introduced in Section~\ref{sec:HDBexpts} with the following important modifications:
\begin{itemize}
	\item We choose the preferred location of a student uniformly at random from the collection of census tracts (polygons) belonging to her tier (see Figure \ref{ChicagoMap}); the position of every polygon is approximated by taking the average of the coordinates of its extreme points.
	\item We reset each student's utility to $0$ for any school ranked $21$st or lower in the preference ordering induced by her utilities computed as in the distance-based model of Section~\ref{sec:HDBexpts}, and then renormalize the utilities; this is because students are allowed to apply to at most 20 schools.
\end{itemize}

\paragraph{Evaluation} In our experiments, we vary both $\sigma^2$ in $\{0,10,50\}$ and $n$ in $\{2261,5000\}$, and for each setting, we compute the following quantities: the price of diversity $\PoD(u)$ as per Definition~\ref{def:pod} (hatched bar); the theoretical upper bound on $\PoD(u)$ as per Theorem~\ref{thm_parampod} (dark gray bar); and the price of the diverse lottery $\PoDL(u,\cdot)$ as per Definition~\ref{def:podl}, averaged over $100$ agent permutations (light gray bar).
Results obtained by averaging over 100 runs are given in Figure \ref{figTests4}. First, we observe that both the price of diversity and the loss of the lottery mechanism decrease as $\sigma^2$ increases; this is unsurprising since utilities are less type-dependent as $\sigma^2$ grows. However, we observe that the loss of the lottery mechanism is quite high in all instances. 
%This shows that we need a better approximation algorithm for the $\ASTC$ problem. 
Moreover, just as in the Singapore public housing allocation problem, the outcome of the lottery mechanism is negatively impacted by the number of students; therefore, we can conclude that this mechanism seems to be better suited to problems with an equal number of agents and items.

\begin{figure}[!t]
\centering
\begin{subfigure}{.5\columnwidth}
\begin{tikzpicture}[scale=1] \normalsize
\coordinate (A) at (-1/2,0);
\coordinate (B) at (12.7/2,0);
\coordinate (C) at (-1/2,6.5);

\draw[thick,-latex] (A) -- (B);
\draw (12.5/2,-0.02) node[below]{$\sigma^2$};
\draw[thick,-latex] (A) -- (C);

\draw (-1/2,1) node{$\_$};
\draw (-1.5/2,1) node{$1$};
\draw (-1.5/2,5.7) node{\rotatebox{90}{welfare ratio}};

\draw (1.1/2,0) node{$|$};
\draw (1.1/2,-0.1) node[below]{$0$};
\draw (5.6/2,0) node{$|$};
\draw (5.6/2,-0.1) node[below]{$10$};
\draw (10/2,0) node{$|$};
\draw (10/2,-0.1) node[below]{$50$};

%random var 1
\draw (-1/2,0) -- (-1/2,1.2118467) -- (0.4/2,1.2118467) -- (0.4/2,0) --cycle;
\filldraw[pattern=north east lines] (-1/2,0) -- (-1/2,1.2118467) -- (0.4/2,1.2118467) -- (0.4/2,0) --cycle;
\draw (-0.3/2,0) node[above]{\hlw{$1.21$}};
\draw (-0.3/2, 1.2118467-0.012177202313235398) -- (-0.3/2,1.2118467+0.012177202313235398);
\draw (-0.8/2, 1.2118467-0.012177202313235398) -- (0.2/2, 1.2118467-0.012177202313235398);
\draw (-0.8/2, 1.2118467+0.012177202313235398) -- (0.2/2, 1.2118467+0.012177202313235398);

\draw (0.4/2,0) -- (0.4/2,2.0566387) -- (1.8/2,2.0566387) -- (1.8/2,0)--cycle;
\filldraw[color=gray!80!black, opacity=0.3] (0.4/2,0) -- (0.4/2,2.0566387) -- (1.8/2,2.0566387) -- (1.8/2,0)--cycle;
\draw (1.1/2,1/4) node[above]{\hlw{$2.06$}};
\draw (1.1/2, 2.0566387-0.029619403903660566) -- (1.1/2, 2.0566387+0.029619403903660566);
\draw (0.6/2, 2.0566387-0.029619403903660566) -- (1.6/2, 2.0566387-0.029619403903660566);
\draw (0.6/2, 2.0566387+0.029619403903660566) -- (1.6/2, 2.0566387+0.029619403903660566);

\draw  (1.8/2,0) -- (1.8/2,5.503533) -- (3.2/2,5.503533) -- (3.2/2,0)  --cycle;
\filldraw[color=gray!80!black, opacity=0.9] (1.8/2,0) -- (1.8/2,5.503533) -- (3.2/2,5.503533) -- (3.2/2,0)  --cycle;
\draw (2.5/2,2/4) node[above]{\hlw{$5.50$}};
\draw (2.5/2,5.503533-0.26256064648175814) -- (2.5/2,5.503533+0.26256064648175814);
\draw (2/2, 5.503533-0.26256064648175814) -- (3/2, 5.503533-0.26256064648175814);
\draw (2/2, 5.503533+0.26256064648175814) -- (3/2, 5.503533+0.26256064648175814);

%random var 5
\draw  (3.5/2,0) -- (3.5/2,1.1326774) -- (4.9/2,1.1326774) -- (4.9/2,0) --cycle;
\filldraw[pattern=north east lines] (3.5/2,0) -- (3.5/2,1.1326774) -- (4.9/2,1.1326774) -- (4.9/2,0) --cycle;
\draw (4.2/2,0) node[above]{\hlw{$1.13$}};
\draw (4.2/2,1.1326774-0.007462567628421884) -- (4.2/2,1.1326774+0.007462567628421884);
\draw (3.7/2, 1.1326774-0.007462567628421884) -- (4.7/2, 1.1326774-0.007462567628421884);
\draw (3.7/2, 1.1326774+0.007462567628421884) -- (4.7/2, 1.1326774+0.007462567628421884);

\draw  (4.9/2,0) -- (4.9/2,1.7563305) -- (6.3/2,1.7563305) -- (6.3/2,0)--cycle;
\filldraw[color=gray!80!black, opacity=0.3] (4.9/2,0) -- (4.9/2,1.7563305) -- (6.3/2,1.7563305) -- (6.3/2,0)--cycle;
\draw (5.6/2,1/4) node[above]{\hlw{$1.76$}};
\draw (5.6/2,1.7563305-0.021859431351834397) -- (5.6/2,1.7563305+0.021859431351834397);
\draw (5.1/2, 1.7563305-0.021859431351834397) -- (6.1/2, 1.7563305-0.021859431351834397);
\draw (5.1/2, 1.7563305+0.021859431351834397) -- (6.1/2, 1.7563305+0.021859431351834397);

\draw  (6.3/2,0) -- (6.3/2,5.6126537) -- (7.7/2,5.6126537) -- (7.7/2,0)  --cycle;
\filldraw[color=gray!80!black, opacity=0.9] (6.3/2,0) -- (6.3/2,5.6126537) -- (7.7/2,5.6126537) -- (7.7/2,0)  --cycle;
\draw (7/2,2/4) node[above]{\hlw{$5.61$}};
\draw (7/2,5.6126537-0.21768159779801774) -- (7/2,5.6126537+0.21768159779801774);
\draw (6.5/2,5.6126537-0.21768159779801774) -- (7.5/2, 5.6126537-0.21768159779801774);
\draw (6.5/2,5.6126537+0.21768159779801774) -- (7.5/2, 5.6126537+0.21768159779801774);

%random var 10 - constraints
\draw  (8/2,0) -- (8/2,1.0384176) -- (9.4/2,1.0384176) -- (9.4/2,0) --cycle;
\filldraw[pattern=north east lines] (8/2,0) -- (8/2,1.0384176) -- (9.4/2,1.0384176) -- (9.4/2,0) --cycle;
\draw (8.7/2,0) node[above]{\hlw{$1.04$}};
\draw (8.7/2,1.0384176-0.0026611483744819983) -- (8.7/2,1.0384176+0.0026611483744819983);
\draw (8.2/2, 1.0384176-0.0026611483744819983) -- (9.2/2, 1.0384176-0.0026611483744819983);
\draw (8.2/2, 1.0384176+0.0026611483744819983) -- (9.2/2, 1.0384176+0.0026611483744819983);

%random var 10 - lottery
\draw (9.4/2,0) -- (9.4/2,1.4161681) -- (10.8/2,1.4161681) -- (10.8/2,0)--cycle;
\filldraw[color=gray!80!black, opacity=0.3] (9.4/2,0) -- (9.4/2,1.4161681) -- (10.8/2,1.4161681) -- (10.8/2,0)--cycle;
\draw (10.1/2,1/4) node[above]{\hlw{$1.42$}};
\draw (10.1/2,1.4161681-0.009918654783187215) -- (10.1/2,1.4161681+0.009918654783187215);
\draw (9.6/2, 1.4161681-0.009918654783187215) -- (10.6/2, 1.4161681-0.009918654783187215);
\draw (9.6/2, 1.4161681+0.009918654783187215) -- (10.6/2, 1.4161681+0.009918654783187215);

%random var 10 - bound
\draw (10.8/2,0) -- (10.8/2,4.8575587) -- (12.2/2,4.8575587) -- (12.2/2,0) --cycle;
\filldraw[color=gray!80!black, opacity=0.9] (10.8/2,0) -- (10.8/2,4.8575587) -- (12.2/2,4.8575587) -- (12.2/2,0) --cycle;
\draw (11.5/2,2/4) node[above]{\hlw{$4.86$}};
\draw (11.5/2,4.8575587-0.15814840728147608) -- (11.5/2,4.8575587+0.15814840728147608);
\draw (11/2, 4.8575587-0.15814840728147608) -- (12/2, 4.8575587-0.15814840728147608);
\draw (11/2, 4.8575587+0.15814840728147608) -- (12/2, 4.8575587+0.15814840728147608);

\end{tikzpicture}
%  \caption{$\Dist(\sigma^2)$}
  %\label{fig:1-sub1}
\end{subfigure}%
\begin{subfigure}{.5\columnwidth}

\hfill
\begin{tikzpicture}[scale=1] \normalsize
\coordinate (A) at (-1/2,0);
\coordinate (B) at (12.7/2,0);
\coordinate (C) at (-1/2,6.5);

\draw[thick,-latex] (A) -- (B);
\draw (12.5/2,-0.02) node[below]{$\sigma^2$};
\draw[thick,-latex] (A) -- (C);

\draw (-1/2,1) node{$\_$};
\draw (-1.5/2,1) node{$1$};
\draw (-1.5/2,5.7) node{\rotatebox{90}{welfare ratio}};

\draw (1.1/2,0) node{$|$};
\draw (1.1/2,-0.1) node[below]{$0$};
\draw (5.6/2,0) node{$|$};
\draw (5.6/2,-0.1) node[below]{$10$};
\draw (10/2,0) node{$|$};
\draw (10/2,-0.1) node[below]{$50$};

%ethnicity var 1
\draw (-1/2,0) -- (-1/2,1.2518184) -- (0.4/2,1.2518184) -- (0.4/2,0) --cycle;
\filldraw[pattern=north east lines] (-1/2,0) -- (-1/2,1.2518184) -- (0.4/2,1.2518184) -- (0.4/2,0) --cycle;
\draw (-0.3/2,0) node[above]{\hlw{$1.25$}};
\draw (-0.3/2,1.2518184-0.041640607392458374) -- (-0.3/2,1.2518184+0.041640607392458374);
\draw (-0.8/2, 1.2518184-0.041640607392458374) -- (0.2/2, 1.2518184-0.041640607392458374);
\draw (-0.8/2, 1.2518184+0.041640607392458374) -- (0.2/2, 1.2518184+0.041640607392458374);

\draw (0.4/2,0) -- (0.4/2,2.4225519) -- (1.8/2,2.4225519) -- (1.8/2,0)--cycle;
\filldraw[color=gray!80!black, opacity=0.3] (0.4/2,0) -- (0.4/2,2.4225519) -- (1.8/2,2.4225519) -- (1.8/2,0)--cycle;
\draw (1.1/2,1/4) node[above]{\hlw{$2.42$}};
\draw (1.1/2,2.4225519-0.3671208483056494) -- (1.1/2,2.4225519+0.3671208483056494);
\draw (0.6/2, 2.4225519-0.3671208483056494) -- (1.6/2, 2.4225519-0.3671208483056494);
\draw (0.6/2,2.4225519+0.3671208483056494) -- (1.6/2, 2.4225519+0.3671208483056494);

\draw (1.8/2,0) -- (1.8/2,5.7169294) -- (3.2/2,5.7169294) -- (3.2/2,0)  --cycle;
\filldraw[color=gray!80!black, opacity=0.9] (1.8/2,0) -- (1.8/2,5.7169294) -- (3.2/2,5.7169294) -- (3.2/2,0)  --cycle;
\draw (2.5/2,2/4) node[above]{\hlw{$5.72$}};
\draw (2.5/2,5.7169294 -0.35556812450077185) -- (2.5/2,5.7169294 +0.35556812450077185);
\draw (2/2, 5.7169294 -0.35556812450077185) -- (3/2, 5.7169294 -0.35556812450077185);
\draw (2/2,5.7169294 +0.35556812450077185) -- (3/2,5.7169294 +0.35556812450077185);

%ethn var 5 - constraints
\draw (3.5/2,0) -- (3.5/2,1.1677887) -- (4.9/2,1.1677887) -- (4.9/2,0) --cycle;
\filldraw[pattern=north east lines] (3.5/2,0) -- (3.5/2,1.1677887) -- (4.9/2,1.1677887) -- (4.9/2,0) --cycle;
\draw (4.2/2,0) node[above]{\hlw{$1.17$}};
\draw (4.2/2,1.1677887-0.03608559367976477) -- (4.2/2,1.1677887+0.03608559367976477);
\draw (3.7/2, 1.1677887-0.03608559367976477) -- (4.7/2, 1.1677887-0.03608559367976477);
\draw (3.7/2, 1.1677887+0.03608559367976477) -- (4.7/2, 1.1677887+0.03608559367976477);

\draw  (4.9/2,0) -- (4.9/2,2.0379956) -- (6.3/2,2.0379956) -- (6.3/2,0)--cycle;
\filldraw[color=gray!80!black, opacity=0.3] (4.9/2,0) -- (4.9/2,2.0379956) -- (6.3/2,2.0379956) -- (6.3/2,0)--cycle;
\draw (5.6/2,1/4) node[above]{\hlw{$2.04$}};
\draw (5.6/2,2.0379956-0.28259715638611543) -- (5.6/2,2.0379956+0.28259715638611543);
\draw (5.1/2, 2.0379956-0.28259715638611543) -- (6.1/2,2.0379956-0.28259715638611543);
\draw (5.1/2, 2.0379956+0.28259715638611543) -- (6.1/2,2.0379956+0.28259715638611543);

\draw (6.3/2,0) -- (6.3/2,5.9224505) -- (7.7/2,5.9224505) -- (7.7/2,0)  --cycle;
\filldraw[color=gray!80!black, opacity=0.9] (6.3/2,0) -- (6.3/2,5.9224505) -- (7.7/2,5.9224505) -- (7.7/2,0)  --cycle;
\draw (7/2,2/4) node[above]{\hlw{$5.92$}};
\draw (7/2,5.9224505-0.41942978839418443) -- (7/2,5.9224505+0.41942978839418443);
\draw (6.5/2,5.9224505-0.41942978839418443) -- (7.5/2,5.9224505-0.41942978839418443);
\draw (6.5/2,5.9224505+0.41942978839418443) -- (7.5/2,5.9224505+0.41942978839418443);

%ethn var 10 - constraints
\draw (8/2,0) -- (8/2,1.0431257) -- (9.4/2,1.0431257) -- (9.4/2,0) --cycle;
\filldraw[pattern=north east lines] (8/2,0) -- (8/2,1.0431257) -- (9.4/2,1.0431257) -- (9.4/2,0) --cycle;
\draw (8.7/2,0) node[above]{\hlw{$1.04$}};
\draw (8.7/2,1.0431257-0.005621286665752604) -- (8.7/2,1.0431257+0.005621286665752604);
\draw (8.2/2, 1.0431257-0.005621286665752604) -- (9.2/2, 1.0431257-0.005621286665752604);
\draw (8.2/2, 1.0431257+0.005621286665752604) -- (9.2/2, 1.0431257+0.005621286665752604);

%ethn var 10 - lottery
\draw (9.4/2,0) -- (9.4/2,1.5519581) -- (10.8/2,1.5519581) -- (10.8/2,0)--cycle;
\filldraw[color=gray!80!black, opacity=0.3] (9.4/2,0) -- (9.4/2,1.5519581) -- (10.8/2,1.5519581) -- (10.8/2,0)--cycle;
\draw (10.1/2,1/4) node[above]{\hlw{$1.55$}};
\draw (10.1/2,1.5519581-0.13621049867099336) -- (10.1/2,1.5519581+0.13621049867099336);
\draw (9.6/2, 1.5519581-0.13621049867099336) -- (10.6/2, 1.5519581-0.13621049867099336);
\draw (9.6/2, 1.5519581+0.13621049867099336) -- (10.6/2, 1.5519581+0.13621049867099336);

%ethn var 10 - bound
\draw (10.8/2,0) -- (10.8/2,5.0881653) -- (12.2/2,5.0881653) -- (12.2/2,0) --cycle;
\filldraw[color=gray!80!black, opacity=0.9] (10.8/2,0) -- (10.8/2,5.0881653) -- (12.2/2,5.0881653) -- (12.2/2,0) --cycle;
\draw (11.5/2,2/4) node[above]{\hlw{$5.09$}};
\draw (11.5/2,5.0881653-0.3004501757244344) -- (11.5/2,5.0881653+0.3004501757244344);
\draw (11/2,5.0881653-0.3004501757244344) -- (12/2, 5.0881653-0.3004501757244344);
\draw (11/2,5.0881653+0.3004501757244344) -- (12/2, 5.0881653+0.3004501757244344);
\end{tikzpicture}
%  \caption{$\Ethn(\sigma^2)$}
  %\label{fig:1-sub3}
\end{subfigure}%
\caption{Averaged welfare ratios obtained for $\Dist(\sigma^2)$ with $n = m = 2261$ (left) and $n = 5000, m = 2261$ (right) in our simulated instances of the Chicago public school admission problem. Horizontal axis shows values of utility model parameter $\sigma^2$ (variance); vertical axis shows ratio (or upper bound on ratio) of optimal unconstrained welfare to welfare induced by some constrained mechanism --- hatched bar: $\PoD(u)$ as per Definition~\ref{def:pod}, light gray bar: upper bound from Theorem~\ref{thm_parampod}, dark gray bar: $\PoDL(u,\cdot)$ as per Definition~\ref{def:podl}, averaged over $100$ agent permutations. \label{figTests4}}
\end{figure}

\section{Discussion and Future Work}
Our work constitutes a first step towards a better understanding of  the  effect of hard diversity constraints on social welfare. We can summarize our results as follows. One of our main contributions is to offer computational insights into the \ASTC problem: a general hardness result, sufficient conditions for tractability, and a $\tfrac12$-approximation algorithm. We derive two upper bounds on the price of diversity defined as the ratio of the optimal welfare achievable with and without type-block constraints: the first is in terms of block quotas only, independent of the utility model, hence under the planner's control; the second is parametrized in terms of inter-type disparity, which shows that when the disparity is low, the welfare loss is much closer to its ideal value of $1$ than the first bound would suggest. We analyze our model's behavior in simulations based on two real-world data sets.

 We will conclude with further remarks on the above results, a few  straightforward implications, and directions for further research.

\subsection{Further questions on complexity and approximation}
Some questions that remain open are:
\begin{itemize}
	\item Is a better approximation to \ASTC possible?
	\item Does there exist a more natural/elegant polynomial-time approximation algorithm for \ASTC that does not involve reduction to \BCM?
	\item Are there any other non-trivial special cases of \ASTC that admit polynomial-time solutions?
\end{itemize}
However, since we have shown that \ASTC is a special (still NP-complete) case of the \BCM problem, some easy generalizations of our results follow from the literature on \BCM.
For example, there is a polynomial-time approximation scheme (PTAS) for \BCM if one allows $(1+\eps)$-violations of the color constraints~\cite{grandoni2009iterative}; this immediately implies a PTAS for \ASTC where one allows $(1+\eps)$-violations of the type-block constraints. Intuitively, this means that we can efficiently achieve near-optimal constrained optimal welfare if we are willing to \emph{soften} our hard constraints within well-defined limits. 

\subsection{\emph{Soft} diversity concepts}
 There are alternative approaches towards diversity promotion in various branches of research that do not rely on \emph{hard bounds} (fixed quotas):
\begin{itemize}
	\item maximization of a carefully constructed monotone, submodular function that itself trades diversity off against ``solution quality'' (similar in spirit to the welfare concept we have used) in a matching/allocation/subset selection setting \cite{ahmed2017diverse,schumann2019diverse,dickerson2019balancing};
	\item soft diversity constraints in a matching setting (e.g. school seat allocation) with ordinal preferences \cite[and references therein]{ehlers2014school,kurata2017controlled}: quotas are used as ``flexible limits that regulate ... priorities" when assigning scarce items \cite{ehlers2014school}.
	\item sampling datasets to ensure an adequate representation of labeled categories (\emph{combinatorial diversity} or \emph{diversity indices}) and/or a broad coverage of the feature space (\emph{geometric diversity} via \emph{determinantal point processes}) for machine learning applications \cite[and references therein]{celis2016fair}.	 
\end{itemize}
An important future research direction would be to compare and contrast diversity quotas with these different approaches towards diverse solutions in terms of their impact on welfare, both theoretically and experimentally.

\subsection{Types based on ``overlapping'' attributes} 
The types considered in our paper essentially constitute a partition over the agents with respect to one \emph{attribute}: ethnicity in Singapore public housing and SES-based tiers in Chicago Public Schools. But in many situations, quotas/bounds are specified for multiple attributes. Consider for example, a hypothetical public housing allocation problem in a city that has the same three ethnic groups (and the respective quotas) as in Singapore; but the population is also divided into high-income and low-income groups and the city additionally requires that no block of flats should have more than $55\%$ of its occupants from either of these two groups. Flats, in addition to be being grouped into blocks, could also be partitioned based on categories (see Section~\ref{sec:HDBexpts}) and any number of capacity constraints could be specified combinatorially.

It is easy to see that \ASTC is a special case of this class of problems. So, the NP-hardness of \ASTC (Theorem~\ref{theoNP}) readily implies that of this more general class. However, it is not obvious whether this general version reduces to \BCM, hence approximability results of \ASTC may not generalize. It is also unclear how to extend the $\PoD$ analysis of Section~\ref{sec:diversity} -- if it can be done at all -- to the situation with multiple partitions each with its own set of arbitrary percentage caps. This is a challenge we wish to tackle in the future. For existing work on diversity requirements based on multiple overlapping attributes in a subset selection setting (not matching), please refer to \citet{bredereck2018multiwinner,lang2016multi}. 

%An obvious attempt is to take the coarsest common refinement of partitions induced by all attributes on either side (agents and items) and treat each part of the resulting single partition as an effective type (resp. effective block) of agents (resp. items); e.g. for the above example, the effective types would be ``high-income Chinese", ``low-income Chinese", ``high-income Malay", ``low-income Malay", ``high-income Indian/Others", and ``low-income Indian/Others". However, the (arbitrary) percentage caps for each partition may not translate fixed percentage caps for effective types and blocks. For example, consider the above partition of agents based on ethnicity and income and a partition of flats into blocks only; let the subset of agents of ethnicity $e \in [3]$ and income-group $g \in [2]$ be denoted by $N_{e,g}$. This affects not just the complexity results but also the welfare properties. utility-free upper bound on $\PoD$ applies?
 
\subsection{Utility models and experiments}
Simulating agent utilities is still a major challenge: ideally, one would elicit applicants' utilities directly via large-scale national surveys. 
The aim of our experiments was to see how much worse the diversity-constrained optimal allocation performs than our theoretical bounds (and, to some extent, how much better it is compared to its lottery version) depending on the utility model, and not to compare utility models based on the $\PoD$ as a performance measure. Hence, we formulated clean utility models that focus on individual factors that, we believe, bear on individuals' preferences in the problem domains considered. For example, our simulations with the $\Dist(\sigma^2)$ and $\Ethn(\sigma^2)$ models in Section~\ref{sec:HDBexpts} tested two `extreme' cases: one where there is no correlation between ethnicity and utility, and one where utility is artificially correlated to ethnicity. The truth is likely somewhere in between. Ethnic groups in Singapore most likely do have some correlation between their utility models; this can be due to socioeconomic factors (there is some correlation between ethnicity and socioeconomic status), the location of cultural or religious centers, or other unknown factors. Developing and investigating refined utility models is an interesting direction for future work.

\subsection{Lottery mechanisms with diversity quotas}

Allocation based on constrained optimization under known utilities can serve as a benchmark for any mechanism respecting the same set of diversity constraints, including the lottery mechanism in Algorithm~\ref{lottery}.
Obviously, the relative loss of the lottery mechanism cannot be smaller than the $\PoD$ for the same problem instance; but, in our experiments, the diverse lottery performs surprisingly well for several utility models (e.g. Figures~\ref{figTests1} and~\ref{fig:ProjApprov}), although it appears to be more sensitive to the utility model and its parameters than the constrained optimization benchmark (e.g. Figures~\ref{figTests3} and~\ref{figTests4}). As yet, we do not have satisfactory explanations for many of our observations. Offering theoretical guarantees on the performance of the lottery mechanism with diversity constraints (and in some sense, complementing the analysis by \citet{raffles17immorlica}) would provide better insights on our experimental results. In particular, we could look at the ratio of the expected utilitarian social welfare induced by a lottery without diversity constraints to that under diversity constraints for the same utilities.\footnote{This ratio can never exceed the expected price of the diverse lottery as defined in  Section~\ref{sec:lottery} (see Definition~\ref{def:podl} and the following paragraph), but is not guaranateed to be larger than the $\PoD$ for the same problem instance. A detailed analysis, both theoretical and experimental, of this concept is an important research topic in its own right, hence we have not reported the ratio of the expected welfare of the lottery mechanism without constraints to that with type-block constraints for our experiments in Section~\ref{sec:experiments}.} 

Moreover, recall that the lottery mechanism operates on the basis of partial information on the cardinal utilities to which the benchmark has access: specifically, for each agent in the random sequence, the algorithm only elicits the agent's top choice among the items in blocks for which the corresponding capacity constraint has not been satisfied, which also depends on the predecessors of the agent in the random sequence. This is even more restricted information than some recent work on matching/assignment \cite{anshelevich2016blind} and subset selection \cite[and references therein]{anshelevich2018approximating} has considered: these approaches assume access to ordinal information (each agent's ranking over alternatives) but the performance is measured in terms of the hidden cardinal utilities that generate these rankings -- the concept of  \emph{distortion} \cite{procaccia2006distortion} is often used in this line of work as a measure of fractional loss in (cardinal) solution quality due to ordinal information. It will be interesting to study (optimal) algorithms working with ranking data under diversity constraints, and the resulting combination of distortion and price of diversity -- intuitively, the welfare loss with full rankings cannot be worse than that of our lottery (with significantly less information), which already performs close to the constrained optimum (with the true cardinal inputs) in many of our experiments.\footnote{We thank an anonymous reviewer for pointing out this possible connection to the literature on distortion.}

A relevant practical question that also needs attention is why the constrained optimization approach is not used in our motivating real-world problems instead of allocations based on lotteries. There could be several problem-specific reasons for this. We offer \ASTC as a benchmark for the \emph{utilitarian} objective under one particular diversity-promoting policy (hard capacity bounds); but a central planner/allocator might have several objectives in mind at the same time. In a mechanism that relies on agents' subjective valuations of items (that need to be elicited in some form), one might also care about \emph{individual fairness} and \emph{discouraging strategic behavior}. In the absence of diversity constraints, a (uniform) lottery is identical to \emph{random serial dictatorship} which satisfies many such desiderata \cite{abdulkadirouglu1998random,bogomolnaia2001new}. As such, an extensive investigation of diversity constraints with properties other than the utilitarian social welfare is also imperative.  

\subsection{Closing remarks}
Our results describe an inevitable tradeoff between diversity and social welfare; however, we emphasize that this does not constitute a moral judgment on the authors' part. Economic considerations are certainly important, but they are by no means an exclusive or even a first-order consideration. That said, understanding the impact of diversity constraints on social welfare is key if one is to justify their implementation.

\begin{acks}
This research was funded by an MOE Grant (no. R-252-000-625-133) and an NRF Research Fellowship (no. R-252-000-750-733). The authors would like to thank the anonymous reviewers and attendees of both AAMAS 2018 and COMSOC 2018, where part of this work was presented, and the anonymous reviewers of TEAC, whose insightful comments on previous drafts greatly improved the paper. Thanks are also due to members of the Research and Planning Group, Housing and Development Board of Singapore, who met with the authors and gave valuable feedback on an early version of the paper.
\end{acks}

\bibliographystyle{ACM-Reference-Format}
\bibliography{abb,paper}

\end{document}